\documentclass{article}

     \PassOptionsToPackage{numbers, compress}{natbib}

\usepackage[final]{neurips_2022}

\usepackage[utf8]{inputenc} 
\usepackage[T1]{fontenc}    
\usepackage{hyperref}       
\usepackage{url}            
\usepackage{booktabs}       
\usepackage{amsfonts}       
\usepackage{nicefrac}       
\usepackage{microtype}      
\usepackage{xcolor}         
\usepackage{graphicx}
\usepackage{float}
\usepackage{caption}
\usepackage{subcaption}
\usepackage{rotating}
\usepackage{diagbox}
\usepackage{paralist}
\usepackage{mathtools}

\usepackage{amsmath}
\usepackage{amsthm}
\usepackage{amssymb}

\usepackage{pgf,tikz}
\usetikzlibrary{positioning}

\newtheorem{theorem}{Theorem}
\newtheorem{lemma}[theorem]{Lemma}

\newtheorem{cor}[theorem]{Corollary}
\newtheorem{assum}{Assumption}

\usepackage{listings}
\usepackage{xcolor}

\theoremstyle{definition}

\definecolor{codegreen}{rgb}{0,0.6,0}
\definecolor{codegray}{rgb}{0.5,0.5,0.5}
\definecolor{codepurple}{rgb}{0.58,0,0.82}
\definecolor{backcolour}{rgb}{0.95,0.95,0.92}

\lstdefinestyle{mystyle}{
    backgroundcolor=\color{backcolour},   
    commentstyle=\color{codegreen},
    keywordstyle=\color{magenta},
    numberstyle=\tiny\color{codegray},
    stringstyle=\color{codepurple},
    basicstyle=\ttfamily\footnotesize,
    breakatwhitespace=false,         
    breaklines=true,                 
    captionpos=b,                    
    keepspaces=true,                 
    numbers=left,                    
    numbersep=5pt,                  
    showspaces=false,                
    showstringspaces=false,
    showtabs=false,                  
    tabsize=2
}

\newcommand{\indep}{\perp \!\!\! \perp}

\title{Deep Learning Methods for Proximal Inference via Maximum Moment Restriction}

\author{
Benjamin Kompa\thanks{Denotes equal contribution} \\
Department of Biomedical Informatics\\
Harvard Medical School \\
 \texttt{benjamin\_kompa@hms.harvard.edu}
 \And
 David R. Bellamy\textsuperscript{*} \\
 Department of Epidemiology, CAUSALab\\
 Harvard School of Public Health\\
 \texttt{david\_bellamy@g.harvard.edu}\\
 \And
Thomas Kolokotrones \\
Department of Epidemiology\\
Harvard School of Public Health\\
\texttt{thomas\_kolokotrones@hms.harvard.edu}\\
\And
James M. Robins \\
Department of Epidemiology, CAUSALab\\
Harvard School of Public Health\\
\texttt{robins@hsph.harvard.edu}\\
\And
Andrew L. Beam \\
Department of Epidemiology, CAUSALab\\
Harvard School of Public Health\\
\texttt{andrew\_beam@hms.harvard.edu}\\
}

\begin{document}

\maketitle

\begin{abstract}
  The No Unmeasured Confounding Assumption is widely used to identify causal effects in observational studies. Recent work on proximal inference has provided alternative identification results that succeed even in the presence of unobserved confounders, provided that one has measured a sufficiently rich set of \emph{proxy variables}, satisfying specific structural conditions. However, proximal inference requires solving an ill-posed integral equation. Previous approaches have used a variety of machine learning techniques to estimate a solution to this integral equation, commonly referred to as the \emph{bridge function}. However, prior work has often been limited by relying on pre-specified kernel functions, which are not data adaptive and struggle to scale to large datasets. In this work, we introduce a flexible and scalable  method based on a deep neural network to estimate causal effects in the presence of unmeasured confounding using proximal inference. Our method achieves state of the art performance on two well-established proximal inference benchmarks. Finally, we provide theoretical consistency guarantees for our method. 
  
\end{abstract}

\section{Introduction}\label{sec:introduction}
\looseness=-1
Causal inference is concerned with estimating the effect of a treatment $A$ on an outcome $Y$ from either observational data or the results of a randomized experiment. An estimand of primary importance is the \emph{average causal effect} (ACE), which is the expected difference in $Y$ caused by changing the treatment from value $a$ to $a'$ for each unit in the study population, and is defined as a contrast between the expected value of the potential outcomes at the two levels of the treatment: $\mathbb{E}[Y^{a'}] - \mathbb{E}[Y^a]$. However, in observational settings, the ACE is rarely equal to the observed difference in conditional means, $\mathbb{E}[Y | A=a'] - \mathbb{E}[Y | A=a]$ due to confounding. In an attempt to eliminate the influence of confounding, investigators measure putative confounders $X$ and subsequently make adjustments for $X$ in their analyses.

Given $X$, common approaches, such as standardization and inverse probability weighting (\citet{hernan_robins_2021}), obtain valid estimates of the ACE given that the following assumptions hold: i) Positivity: $Pr[A = a|X = x] > 0$ for all $x$ in the population, ii) Consistency: $Y^a = Y$ for all individuals with $A=a$, iii) No unmeasured confounding which results in conditional exchangeability: $Y^a \indep A | X$., and iv) No model misspecification.



The assumptions are typically unverifiable for continuous data. While model misspecification is likely in all real-world scenarios, flexible models and doubly robust estimators have been developed to mitigate the effect of this assumption\cite{10.1093/aje/kwq439}. Therefore, the assumption of conditional exchangeability, or equivalently, the No Unmeasured Confounding Assumption (NUCA), is the defining characteristic of this broad set of approaches to causal effect estimation (\citet{hernan2006estimating}). However, in many settings, it is unrealistic to assume that we are able to measure a sufficient set of confounders for $A$ and $Y$ such that conditional exchangeability holds.

\emph{Proximal inference} is a recently introduced framework that allows for the identification of causal effects even in the presence of unmeasured confounders \citep{Miao2018-tr, Tchetgen_Tchetgen2020-nw}. Proximal inference requires categorizing the measured covariates into three groups: treatment-inducing proxy variables $Z$, outcome-inducing proxy variables $W$, and ``backdoor'' variables $X$ that affect both $A$ and $Y$ (i.e. typical confounders). See Figure \ref{fig:figure1} for an example of a directed acyclic graph (DAG) that admits identification under the assumptions of proximal inference. The proxy sets $W$ and $Z$ must contain sufficient information about the remaining unobserved confounders $U$, a condition that can be formalized by completeness assumptions. Under these and several other conditions, one can estimate average potential outcomes from data even in the presence of unmeasured confounding. Proximal inference has potential applications in medical settings, where a natural question is the effect of a treatment on an outcome in the presence of unmeasured confounding. Before applying proximal inference to real world problems, more validation is required before they can be used safely to inform medical decision-making.

Existing methods for proximal inference can be divided into two categories: two-stage regression procedures and methods that impose a maximum moment restriction (MMR). In two-stage regression procedures, the first stage aims to predict outcome-inducing proxy variables $W$ as a function of $A$, $X$, and $Z$. Then, the second stage regression estimates outcomes $Y$ as a function of the predicted $\hat{W}$ and the treatment $A$, and measured confounders $X$. \citet{Tchetgen_Tchetgen2020-nw} introduced the first estimation technique for proximal inference which was a two-stage procedure that used a model based on ordinary least squares regression. \citet{Mastouri2021-rn} extended this framework by replacing simple linear regression with kernel ridge regression. \citet{Xu2021-io} increased feature flexibility further by incorporating neural networks as feature maps instead of kernels.

\looseness=-1
In contrast, MMR methods are single-stage procedures to estimate average potential outcomes. \citet{Muandet2020-yv} introduced MMR for reproducing kernel Hilbert spaces (RKHS). MMR critically relies on the optimization of a V-statistic or U-statistic for learning a function needed to calculate the ACE. \citet{Zhang2020-qr} used an MMR method to obtain point identification of the ACE in the instrumental variable (IV) setting and incorporated neural networks into their method trained with the V-statistic as a loss function and optimized using stochastic gradient descent. \citet{Mastouri2021-rn} demonstrated that the MMR framework with kernel functions can be used for proximal inference as well as IV regression. 

In this work, we introduce a new method, \textit{Neural Maximum Moment Restriction} (NMMR) which is a flexible neural network approach that is trained to minimize a loss function derived from either a U-statistic or V-statistic to satisfy MMR in the proximal setting. The method introduced in this work makes several novel contributions to the proximal inference literature: 
\begin{compactitem}
\looseness=-1
\item We introduce a new, single stage method based on neural networks for estimating potential outcomes and the ACE in the presence of unmeasured confounding. 
\looseness=-1
\item We provide new theoretical consistency guarantees for our method. 
\looseness=-1
\item We demonstrate state-of-the-art (SOTA) performance on two well-established proximal inference benchmark tasks.
\looseness=-1
\item We show for the first time how to incorporate domain-specific inductive biases using a convolutional model on a proximal inference task that uses images.
\looseness=-1
\item We provide the first unbiased estimate of the MMR risk function using the U-statistic rather than V-statistic in the proximal setting.

\end{compactitem}

\begin{figure}[htbp]
    \centering
    \includegraphics[width=\textwidth]{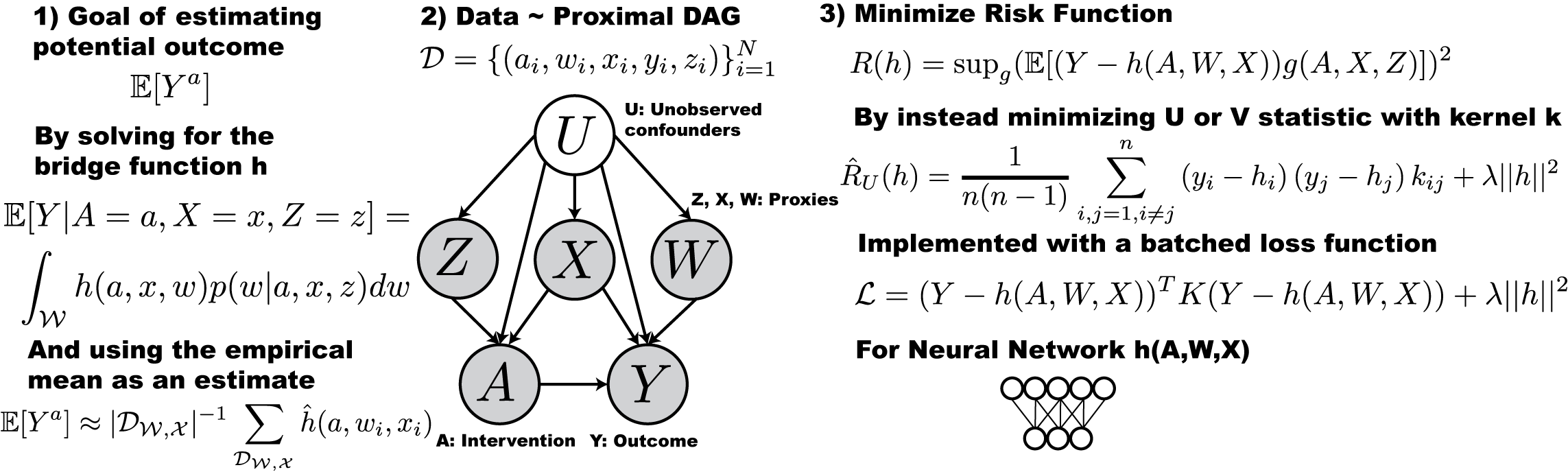}
    \caption{A summary of NMMR. Our method estimates the bridge function $h$, which can be used to compute the average potential outcome $\mathbb{E}[Y^a]$. We rely on structural assumptions for the causal DAG generating the data. NMMR uses a U or V statistic to train a neural network that solves a risk function that reflects a maximum moment restriction function. }
    \label{fig:figure1}
\end{figure}

\section{Background: Proximal Inference}
In this manuscript, capital, caligraphic, and lowercase letters (e.g. $A$, $\mathcal{A}$, and $a$) denote random variables and their corresponding ranges and realizations, respectively. Estimates of random variables and functions will be indicated by hats, e.g. $\hat{Y}$ and $\hat{h}$ are estimates of $Y$ and $h$, respectively.

Our goal is to estimate, for each level of treatment, $a$, the expected potential outcome $\mathbb{E}[Y^{a}]$. Without loss of generality, we refer to $A$ as a treatment, though it could refer to any (possibly continuous) intervention. Proximal inference allows us to do this, in the presence of unobserved confounders, $U$, provided that we have a sufficiently rich set of proxies, $(W, Z)$, that obey certain structural assumptions.  We may also include observed confounders, $X$. We also require the following: 

\begin{assum} \label{a:indep}

Given $(A, U, W, X, Y, Z)$, $Y \indep Z | A, U, X$ and $W \indep (A, Z) | U, X$.

\end{assum}

Figure \ref{fig:figure1} provides an example of a DAG that satisfies these assumptions. The following completeness conditions formalize the notion that the proxies are ``sufficiently rich'':

\begin{assum} \label{a:comp U}

For all $f \in L^2$ and all $a \in \mathcal{A}, x \in \mathcal{X}$, $\mathbb{E}[ f(U) | A=a, X=x, Z=z ] = 0$ for all $z \in \mathcal{Z}$ if and only if $f(U) = 0$ almost surely.

\end{assum}

\begin{assum} \label{a:comp Z}

For all $f \in L^2$ and all $a \in \mathcal{A}, x \in \mathcal{X}$, $\mathbb{E}[ f(Z) | A=a, W=w, X=x ] = 0$ for all $w \in \mathcal{W}$ if and only if $f(Z) = 0$ almost surely.

\end{assum}
\looseness=-1
    We will use two other assumptions at various points in the paper.  The first guarantees the uniqueness of the bridge function, while the second ensures the risk function does not have false zeros.

\begin{assum} \label{a:comp W}

	$\mathbb{E} \left[ f \left( A, W, X \right) \middle| A, X, Z \right] = 0$ $\mathrm{P}_{A, X, Z}$-almost surely if and only if $f \left( A, W, X \right) = 0$ $\mathrm{P}_{A, W, X}$-almost surely.

\end{assum}

\begin{assum} \label{a:ISPD}

	$k : \left( \mathcal{A} \times \mathcal{X} \times \mathcal{Z} \right)^2 \to \mathbb{R}$ is continuous, bounded, and Integrally Strictly Positive Definite (ISPD), so that $\int f \left( \xi \right) k \left( \xi, \xi' \right) f \left( \xi' \right) d \xi d \xi' > 0$ if and only if $f \neq 0$ $\mathrm{P}_{A, Z, X}$-almost surely.

\end{assum}

Assumptions \ref{a:indep}-\ref{a:comp Z} together with several regularity assumptions (see assumptions (v)-(vii) in \citep{Miao2018-tr}) ensure that there exists a function $h$ such that:
\begin{equation}\label{eq:integral_equation}
    \mathbb{E}[Y|A=a, X=x, Z=z] = \int_\mathcal{W} h(a,w,x)p(w|a,x,z)dw
\end{equation}
Equation \ref{eq:integral_equation} is a Fredholm integral equation of the first kind; its solution, $h$, is often called the ``bridge function.'' Theorem 1 of \citet{Miao2018-tr} shows that the expected potential outcomes are given by:
\begin{equation}\label{eq:potential_outcome}
    \mathbb{E}[Y^{a}] = \int_{\mathcal{W},\mathcal{X}} h(a,w,x)p(w,x)dw dx = \mathbb{E}_{W,X}[h(a, W, X)]
\end{equation}

We can obtain unbiased estimates of the expected potential outcomes, $\mathbb{E}[Y^a]$, by splitting the sample, using the first part of the data to estimate the bridge function $h$, by some $\hat{h}$, and using the second part of the data to compute the empirical mean of $\hat{h}$ with $a$ fixed to the value of interest, $\hat{\mathbb{E}}[Y^{a}]=\frac{1}{M}\sum_{i=1}^M \hat{h}(a, w_i, x_i)$. From these, we can obtain other quantities of interest, like the ACE.

In what follows, all norms will be $L^2$ with respect to the relevant probability measure, unless otherwise noted. If necessary, we will explicitly denote the norm of $f$ with respect to a probability measure $\mathcal{P}_{U, V, \dots}$ by $\| f \|_{\mathcal{P}_{U, V, \dots}}$ or $\| f \|_{2, \mathcal{P}_{U, V, \dots}}$.

\section{Related Work} \label{sec:related_work}

\citet{Kuroki2014-ob} first established identification of a causal effect in the setting of unobserved confounders by leveraging noisy proxy variables, $W$, to ``recover'' the distribution of $U$, potentially using external datasets to estimate $p(w|u)$. Tchetgen-Tchetgen and colleagues extended these results to allow for identification without recovery of $U$ in \citet{Miao2018-tr} and \citet{Tchetgen_Tchetgen2020-nw}, also providing a 2-Stage Least Squares (2SLS) method to identify and estimate causal effects under the assumption that the bridge function is linear. \citet{Cui2020-di} introduced a bridge function for Inverse Probability Weighting (IPW), which enabled IPW and Doubly Robust (DR) proximal estimators, for which they presented influence functions under binary treatment.  This was extended in \citet{Ghassami-2022} and further explored by \citet{Kallus2021-hi}, who provided alternative identification assumptions as well as results for general treatments.  The latter two works also consider the use of adversarial methods for estimation, which were previously utilized by Lewis and Syrgkanis \cite{Lewis2018-et} and \citet{Bennett2019-xb} in the Instrumental Variable (IV) setting and \citet{Dikkala2020-sz} in conditional moment models.

Other early investigators include \citet{Deaner2018-ik}, who developed machine learning techniques for proximal inference introducing a method based on a two-stage penalized sieve distance minimization. Several later works similarly employed two-stage regressions with increasingly flexible basis functions to estimate potential outcomes. \citet{Mastouri2021-rn} developed a two-stage kernel ridge regression (Kernel Proxy Variables ``KPV'') to estimate the bridge function $h$, allowing more flexibility than the linear basis of \citet{Tchetgen_Tchetgen2020-nw}. \citet{Xu2021-io} further improved upon this with an adaptive basis derived from neural networks. Their two stage regression method, Deep Feature Proxy Variables (DFPV), established the previous SOTA performance on the proximal benchmark tasks that we consider in our work. Singh and colleagues also considered two stage kernel models in the IV setting \cite{Singh2019-nc} and RKHS techniques for proximal inference \cite{Singh2020-pd}.

An alternative approach based on maximum moment restriction (MMR) uses single-stage estimators of the bridge function. MMR-based methods were established in \citet{Muandet2020-yv} as a way to enforce conditional moment restrictions \citep{Newey2003-ia}. \citet{Zhang2020-qr} introduced the MMR framework to the IV setting, which can be considered a subset of proximal inference without outcome-inducing proxies $W$ and with additional exclusion restrictions \citep{Tchetgen_Tchetgen2020-nw}. There are now several machine learning methods that can be applied in the IV setting \citep{Hartford2017-wm, Lewis2018-et, Singh2019-nc, Dikkala2020-sz, Kato2021-zo} 

Of note, \citet{Zhang2020-qr} introduced MMR-IV, which is related to work by \citet{Lewis2018-et} and \citet{Dikkala2020-sz}. MMR-IV involves optimizing a family of risk functions based on U or V-statistics \cite{Serfling2009-gk}. However, \citet{Zhang2020-qr} only considered IV, rather than proximal, inference and only optimized neural networks by a loss that corresponds to the V-statistic. The V-statistic provides a biased estimate \citep{Serfling2009-gk} of its corresponding risk function, such as $R(h)$ in Equation \ref{eq:r_k}.

Finally, \citet{Mastouri2021-rn} introduced an MMR-based method for proximal inference called Proximal Maximum Moment Restriction (PMMR). PMMR extends the MMR framework to the proximal setting through the use of kernel functions and also optimizes Equation \ref{eq:r_k} via a V-statistic. For a comparison of our model to PMMR and MMR-IV, see Table \ref{tab:related_works}.

\section{Our Method: Neural Maximum Moment Restriction (NMMR)} \label{sec:NMMR}
In this work we propose \emph{Neural Maximum Moment Restriction} (NMMR): a method to estimate expected potential outcomes $\mathbb{E}[Y^a]$ in the presence of unmeasured confounding. We use deep neural networks due to their flexibility, scalability, and adaptability to diverse data types (e.g. using convolutions for images). By rewriting maximum moment restrictions \citep{Muandet2020-yv} as U and V-statistics \citep{Mastouri2021-rn, Zhang2020-qr}, we show how a single-stage neural network procedure can be used to estimate the bridge function.

Following \citet{Muandet2020-yv} and \citet{Zhang2020-qr}, we rewrite the integral equation (\ref{eq:integral_equation}) as a \emph{conditional moment restriction}: $\mathbb{E}[ Y - h(A,W,X) | A, X, Z] = 0$. Then, for any measurable $g : \mathcal{A} \times \mathcal{X} \times \mathcal{Z} \to \mathbb{R}$, $\mathbb{E}[(Y - h(A,W,X))g(A,X,Z)|A, X, Z] = 0$ , and, thus, $\mathbb{E}[(Y - h(A,W,X))g(A,X,Z)] = 0$, so we can use unconditional expectations instead of conditional ones.  This is the basis of the MMR framework of \citet{Muandet2020-yv}. Since this yields an infinite number of moment restrictions we employ a minimax strategy to estimate $h$ by minimizing the risk $R(h)$ for the worst-case value of $g$:
\begin{align} \label{eq:r_k}
    R(h) &= \sup_{\|g\| \leq 1} (\mathbb{E}[(Y-h(A,W,X))g(A,X,Z)])^2
\end{align}
Following \citet{Zhang2020-qr}'s work in the IV setting, \citet{Mastouri2021-rn} (Lemma 2) showed that, if $g$ is an element of an RKHS, $R(h)$ can be rewritten in the form
\begin{align*}
    R_k(h) &= \mathbb{E}[
                            (Y - h(A,W,X))
                            (Y' - h(A',W',X'))
                            k( (A,X,Z), (A',X',Z') )
                        ]
\end{align*}
where $(A', W', X', Y', Z')$ are independent copies of the random variables $(A, W, X, Y, Z)$ and $ k: (\mathcal{A}\times\mathcal{Z}\times\mathcal{X})^2 \to \mathbb{R}$ is a continuous, bounded, and Integrally Strictly Positive Definite (ISPD) kernel. Then, if $h$ satisfies $R_k(h)=0$,  $\mathbb{E}[Y  -h(A,W,X)|A, X, Z]=0$ $\mathrm{P}_{A,X,Z}$-almost surely. Thus, if we can find a neural network $h$ that satisfies $R_k(h)=0$, we will have obtained a $\mathrm{P}_{A, X, Z}$-almost sure solution to Equation \ref{eq:integral_equation} and can compute any expected potential outcome by using Equation \ref{eq:potential_outcome}.

The empirical risk $\hat{R}_{k, n}$ given data $\mathcal{D}=\{(a_i, w_i, x_i, y_i, z_i)\}_{i=1}^N$ can be written as either a U or V-statistic, respectively \citep{Serfling2009-gk}:
\begin{align*}
    \hat{R}_{k, U, n}(h)
        &= \frac{1}{n (n - 1)} \sum_{ i, j = 1, i\neq j }^n { \left( y_i - h_i \right) \left( y_j - h_j \right) k_{ij} } \\
    \hat{R}_{k, V, n}(h)
        &= \frac{1}{n^2} \sum_{ i, j = 1 }^n { \left( y_i - h_i \right) \left( y_j - h_j \right) k_{ij} }
\end{align*}
where $h_i = h \left( a_i, w_i, x_i \right)$ and $k_{ij} = k \left( \left( a_i, z_i, x_i \right), \left( a_j, z_j, x_j \right) \right)$. $\hat{R}_{k, U, n}(h)$ is the minimum variance unbiased estimator of $R_k(h)$ \citep{Serfling2009-gk}, while $\hat{R}_{k, V, n}(h)$ is a biased estimator of $R_k(h)$. In order to prevent overfitting, we add an additional penalty to our risk function $\Lambda : \mathcal{H} \times \Theta_\mathcal{H} \to \mathbb{R}_+$, which is a function of $h$ as well as, possibly, its parameters, $\theta_h$ (e.g. network weights). Specifically, we take $\Lambda$ to be an $L^2$ penalty on the weights so $\Lambda \left[h, \theta_h \right] = \sum_i \theta_{h, i}^2$.  We then denote the penalized risk functions by $\hat{R}_{k, U, \lambda, n}(h) = \hat{R}_{k, U, n}(h) + \lambda \Lambda \left[h, \theta_h \right]$ and $\hat{R}_{k, V, \lambda, n}(h) = \hat{R}_{k, V, n}(h) + \lambda \Lambda \left[h, \theta_h \right]$, respectively.  In practice, $\hat{R}_{k, U, \lambda, n}(h)$ is slightly biased, but, in simulations, is much less biased than even the unpenalized $\hat{R}_{k, V, n}(h)$. Previous work either did not consider the U-statistic \citep{Mastouri2021-rn}, or did not utilize the U-statistic \citep{Zhang2020-qr}. In our work, we introduce two variants of our method, NMMR-U and NMMR-V,  where the former is optimized with a U-statistic and the latter a V-statistic. We train the neural networks in both variants with the regularized loss function:
\begin{align*}
    \mathcal{L} &= ( Y- h(A,W,X))^t K(Y - h(A,W,X)) + \lambda \Lambda \left[ h, \theta_h \right]
\end{align*}
where $(Y-h(A,W,X))$ is a vector of residuals from the neural network's predictions and  $K$ is a kernel matrix with entries $k_{ij}$.  We choose $k$ to be an RBF kernel (see Appendix \ref{appendix:hp_model_arch}). If $\mathcal{L}$ represents a V-statistic, we include the main diagonal elements of $K$, while if $\mathcal{L}$ represents a U-statistic, we set the main diagonal to be 0.  Once we've obtained an optimal neural network $\hat{h}$, we can compute an estimate of the expected potential outcome with data from a held-out dataset with $M$ data points
\begin{align*}
    D_{\mathcal{W}, \mathcal{X}}
        = \{ (w_i, x_i) \}_{i=1}^M,
    \quad \hat{\mathbb{E}[Y^{a}]}
        = \frac{1}{M} \sum_{i=1}^M \hat{h}(a, w_i, x_i)
\end{align*}
\looseness=-1
In contrast to PMMR \citep{Mastouri2021-rn}, which uses kernels as feature maps for proxy and treatment variables, NMMR uses adaptive feature maps from neural networks. NMMR is similar to MMR-IV \citep{Zhang2020-qr}, but MMR-IV is restricted to the instrumental variable (IV) setting rather than the proximal inference setting. Table \ref{tab:related_works} places NMMR in context with existing methods for proximal inference and IV regression.
\begin{table}[H]
\centering
\caption{Comparison of the most related methods to NMMR.}
\resizebox{.9\textwidth}{!}{%
\begin{tabular}{@{}lcccc@{}}
\toprule
\multicolumn{1}{c}{Method}   & Setting  & \# of Stages & Hypothesis Class & Optimization Objective \\ \midrule
KPV \citep{Mastouri2021-rn}  & Proximal & 2           & Kernels          & 2-stage least squares  \\
DFPV \citep{Xu2021-io}       & Proximal & 2           & Neural Networks  & 2-stage least squares  \\
MMR-IV \citep{Zhang2020-qr}  & IV       & 1           & Neural Networks  & V-statistic            \\
PMMR \citep{Mastouri2021-rn} & Proximal & 1           & Kernels          & V-statistic            \\
\textbf{NMMR-V (ours)}       & Proximal & 1           & Neural Networks  & V-statistic            \\
\textbf{NMMR-U (ours)}       & Proximal & 1           & Neural Networks  & U-statistic            \\ \bottomrule
\end{tabular}%
}
\label{tab:related_works}
\end{table}

\section{Consistency of NMMR}\label{sec:consistency}

   	In this section we provide a probabilistic bound on the distance of the estimated bridge function, $\hat{h}_{k, \lambda, n}$, from the true bridge function, $h^*$, in terms of the Radamacher complexity $\mathcal{R}_n(\mathcal{F})$ of a class of functions $\mathcal{F}$ derived from elements of the hypothesis space $\mathcal{H}$ and the fixed kernel, $k$. Note that $\hat{R}_{k, \lambda, n}(h)= \hat{R}_{k, n}(h)+ \lambda \Lambda \left[ h, \theta_h \right]$ (see Section \ref{sec:NMMR}). We use this bound to demonstrate that, under mild conditions, $\hat{h}_{k, \lambda, n}$ converges in probability to $h^*$, and that, under an additional completeness assumption, $h^*$ is unique $\mathrm{P}_{A, W, X}$-almost surely. This provides a consistent estimate of $\mathbb{E}[Y^a]$.

\begin{theorem} \label{AXZ Conv}

	Let $\tilde{h}_k$ minimize $R_k(h)$ and $\hat{h}_{k, U, \lambda, n}$ minimize $\hat{R}_{k, U, \lambda, n}(h)$ for $h \in \mathcal{H}$, $k : ( \mathcal{A} \times \mathcal{X} \times \mathcal{Z} )^2 \to [ - M_k, M_k ]$, $\Lambda : \mathcal{H} \times \Theta_h \to [ 0, M_\lambda ]$, and let $h^* : \mathcal{A} \times \mathcal{W} \times \mathcal{X} \to \mathbb{R}$ satisfy $\mathbb{E} \left[ Y - h^*(A, W, X) \middle| A, X, Z \right] = 0$ $\mathrm{P}_{A, X, Z}$-almost surely, where
\begin{align*}
	R_k(h)
		&= \mathbb{E} \left[
						\left( Y - h \left( A, W, X \right) \right)
						\left( Y' - h \left( A', W', X' \right) \right)
						k \left( \left( A, X, Z \right), \left( A', X', Z' \right) \right)
					\right] \\
	\hat{R}_{k, U, \lambda, n}(h)
		&= \frac{1}{ n (n - 1) } \sum_{i, j = 1, i \neq j}^n
				\left[
					\left( y_i - h \left( a_i, w_i, x_i \right) \right)
					\left( y_j - h \left( a_j, w_j, x_j \right) \right)
				\right. \\
				&\hspace{12em} \times
				\left.
					k \left( \left( a_i, x_i, z_i \right), \left( a_j, x_j, z_j \right) \right)
				\right]
					+ \lambda \Lambda[ h, \theta_h ]
\end{align*}
	Also let,
\begin{align*}
	d_k^2 \left( h, h' \right)
		&= \mathbb{E} \left[
						\left( h \left( A, W, X \right) - h' \left( A, W, X \right) \right)
						\left( h \left( A', W', X' \right) - h' \left( A', W', X' \right) \right)
						\right.\\
						&\hspace{6em} \times \left.
						k \left( \left( A, X, Z \right), \left( A', X', Z' \right) \right)
					\right]
\end{align*}
	Then, $d_k^2 \left( h^*, h \right) = R_k(h)$ and, with probability at least $1 - \delta$,
\begin{align*}
	d_k^2 \left( h^*, \hat{h}_{k, U, \lambda, n} \right)
		&\leq d_k^2 \left( h^*, \tilde{h}_k \right)
			+ \lambda M_\lambda
			+ 8 M \mathbb{E}_{A, X, Z} \left(
					\mathcal{R}_{n-1} \left( \mathcal{F}'_{A, X, Z} \right)
						+ \mathcal{R}_n \left( \mathcal{F}'_{A, X, Z} \right)
				\right) \\
			&\qquad+ 16 M^2 M_k \left( \frac{2}{n} \log \frac{2}{\delta} \right)^\frac{1}{2}
			+ 10 \left( 2 \log 2 \right)^\frac{1}{2} M^2 M_k n^{ - \frac{1}{2} } \\
		&\leq d_k^2 \left( h^*, \tilde{h}_k \right)
			+ \lambda M_\lambda
			+ 8 M \left(
					\mathcal{R}_{n-1} \left( \mathcal{F}' \right)
						+ \mathcal{R}_n \left( \mathcal{F}' \right)
				\right) \\
			&\qquad+ 16 M^2 M_k \left( \frac{2}{n} \log \frac{2}{\delta} \right)^\frac{1}{2}
			+ 10 \left( 2 \log 2 \right)^\frac{1}{2} M^2 M_k n^{ - \frac{1}{2} }
\end{align*}
	Further, if Assumption \ref{a:ISPD} holds, so $k$ is ISPD, then $d_k$ is a metric on $L^2_{\mathcal{AXZ}}$ and, if the right hand side of the inequality goes to zero as $n$ goes to infinity, \\
	$d_k \left(
			\mathbb{E} \left[ h^* \middle| A, X, Z \right]
				- \mathbb{E} \left[ \hat{h}_{k, \lambda, n} \middle| A, X, Z \right]
		\right)
		\xrightarrow{\mathrm{P}} 0$
	so
	$\mathbb{E} \left[ \hat{h}_{k, \lambda, n} \middle| A, X, Z \right]
		\xrightarrow{\mathrm{P}}
		\mathbb{E} \left[ h^* \middle| A, X, Z \right]$
	in $d_k$.  Also,
	$\left\|
			\mathbb{E} \left[ h^* \middle| A, X, Z \right]
				- \mathbb{E} \left[ \hat{h}_{k, \lambda, n} \middle| A, X, Z \right]
		\right\|_{ \mathrm{P}_{A, X, Z} }
		\xrightarrow{\mathrm{P}} 0$
	so
	$\mathbb{E} \left[ \hat{h}_{k, \lambda, n} \middle| A, X, Z \right]
		\xrightarrow{\mathrm{P}}
		\mathbb{E} \left[ h^* \middle| A, X, Z \right]$
	in $L^2 \left( \mathrm{P}_\mathcal{A, X, Z} \right)-norm$.
{\footnotesize
\begin{align*}
	&\mathcal{F}'_{a, x, z}
		= \left\{ f_{a, x, z} \ \middle| \
			\exists_{h \in \mathcal{H}}
			\forall_{ a' \in \mathcal{A}, x' \in \mathcal{X}, z' \in \mathcal{Z} }
				f_{a, x, z} \left( a', w', x', z' \right)
					= h \left( a', w', x' \right) k \left( \left( a', x', z' \right), \left( a, x, z \right) \right)
		\right\} \\
	&\mathcal{F}'
		= \left\{ f \ \middle| \
			\exists_{ h \in \mathcal{H}, a \in \mathcal{A}, x \in \mathcal{X}, z \in \mathcal{Z} }
			\forall_{ a' \in \mathcal{A}, x' \in \mathcal{X}, z' \in \mathcal{Z} }
				f \left( a', w', x', z' \right)
					= h \left( a', w', x' \right) k \left( \left( a', x', z' \right), \left( a, x, z \right) \right)
		\right\}
\end{align*}
}
\end{theorem}

    Corollary \ref{V Stat} provides a similar result for V-statistic estimators of $R(h)$, meaning we can choose to use either U or V-Statistics and have similar guarantees.  In Theorem \ref{AXZ Conv} if the quadratic form converges at a particular rate, say $n^{ - \frac{1}{2} }$, $\mathbb{E} \left[ \hat{h} \middle| A, X, Z \right] \xrightarrow{\mathrm{P}} \mathbb{E} \left[ h^* \middle| A, X, Z \right]$, under the metric induced by the kernel, $d_k$, at half the rate, in this case $n^{ - \frac{1}{4} }$. This is similar to \citet{Kallus2021-hi}'s findings in the unstabilized case.

\begin{theorem} \label{AWX Conv}

	Under Assumption \ref{a:comp W}, $h^*$ is the unique solution to the integral equation $\mathrm{P}_{A, W, X}$-almost surely.
	Further, if
	$\mathbb{E} \left[ \hat{h}_{k, \lambda, n} \middle| A, X, Z \right]
		\xrightarrow{\mathrm{P}}
		\mathbb{E} \left[ h^* \middle| A, X, Z \right]$,
	$\hat{h}_n \xrightarrow{\mathrm{P}} h^*$.
	
\end{theorem}

See Appendix \ref{appendix:proofs} for proofs of Theorems \ref{AXZ Conv} and \ref{AWX Conv}. Taken together, these results tell us that, as long as our optimization algorithm is successful in estimating $\hat{h}_{k, \lambda, n}$, it will asymptotically approach the true bridge function, $h^*$.  In order for this to occur, the right hand side of the inequalities in Theorem \ref{AXZ Conv} must go to zero, which requires not only that the Rademacher terms vanish, but also that $\tilde{h}_k$ must approach $h^*$ arbitrarily closely as $n$ increases.  In practice, this means increasing the complexity of the neural network, but doing so slowly enough the Rademacher complexity terms still decrease with sample size.  Following \citet{Xu2021-io}, we note that recent results from \citet{Neyshabur2018-on} suggest that the Rademacher complexity of a fixed network scales like $n^{ - \frac{1}{2} }$ (similar to many other popular hypothesis classes) and that, although we cannot compute the scaling of the Rademacher terms directly due to the presence of the kernel function, we expect that they will decline with sample size and that, as the neural network becomes more complex, their scaling will more closely resemble terms derived from a pure neural network.  Finally, we require that the regularization parameter decrease as sample size grows, which will, again, depend on the balance between increasing sample size, which tends to decrease the need for regularization, and increasing complexity, which tends to increase its importance.  Thus, by choosing an appropriate growth rate for the network complexity, we expect the aforementioned terms to vanish as $n$ increases to infinity, and, with them, the entire right hand side, making $\hat{h}_{k, \lambda, n}$ a consistent (likely $\sqrt{n}$) estimator of $h^*$.

We can also compare the convergence of the estimated bridge function to that of its projection onto $L^2_\mathcal{AXZ}$.  Prior literature has focused on a measure of ``ill-posedness'' $\tau = \sup_{h \in \mathcal{H}} \| h - h^* \|_2 \left\| \mathrm{E} \left[ h - h^* \middle| A, Z, X \right] \right\|_2^{-1}$.  If $\tau$ is finite, then the rate of convergence of the estimated bridge function will be at worst $\tau$ times that of its projection (it will be slower by a factor of $\tau$).  This will be the case whether we measure convergence using $\left\| \mathrm{E} \left[ h - h^* \middle| A, Z, X \right] \right\|_2$ or the metric induced by $k$.

\section{Experiments}\label{sec:experiments}
\subsection{Overview of Baseline Models}
\looseness=-1
We compare the performance of NMMR-U and NMMR-V to that of several previous approaches, which we describe briefly here. The baselines can be divided into two categories: structural and naive. Structural approaches leverage causal information about the data generating process. They include Kernel Proxy Variables (KPV) \citep{Mastouri2021-rn}, Proximal
Maximum Moment Restriction (PMMR) \citep{Mastouri2021-rn}, Deep Feature Proxy Variables (DFPV) \citep{xu2020learning}, Causal Effect Variational Autoencoder\citep{louizos2017causal} (CEVAE), and the two-stage least squares model (2SLS) from \citet{Miao2018-tr}. For a review of KPV, PMMR, and DFPV, see Section \ref{sec:related_work}. CEVAE is an autoencoder approach derived by \citet{Xu2021-io} from \citet{louizos2017causal}. 2SLS is a two-stage least squares method which assumes that the bridge function $h$ is linear \citep{Tchetgen_Tchetgen2020-nw}. 

\looseness=-1
The naive approaches serve as baselines and do not use causal information, instead directly regressing $A$ and $W$ on the outcome $Y$. These methods include a naive neural network (Naive net), ordinary least squares regression (LS), and ordinary least squares with quadratic features (LS-QF). Naive Net is a neural network that has undergone the same architecture search as NMMR (described further in Appendix \ref{appendix:hp_model_arch}) that is trained to predict $Y$ directly from $A$ and $W$ by minimizing observational MSE, $\frac{1}{n} \sum_{i=1}^n (y - \hat{y})^2$. Least Squares (LS) is the standard linear regression model that predicts $Y$ using a linear combination of $A$ and $W$. Least Squares with Quadratic Features (LS-QF) is the same as LS but with additional quadratic terms $A^2, W^2, AW$. 

\looseness=-1
We evaluate NMMR-U, NMMR-V and baseline methods on two synthetic benchmark tasks from \citet{Xu2021-io}. The first is a simulation of how ticket prices affect the number of tickets sold in the presence of a latent confounder: demand for travel (the Demand experiment). The second is an experiment where the goal is to recover a property of an image that is influenced by an unobserved confounder (the dSprite experiment). The Demand experiment is a low-dimensional estimation problem, whereas dSprite is high-dimensional as $A$ and $W$ are 64x64=4096-dimensional. dSprite leverages image-specific models, which are rarely used in the causal inference literature. Neither task uses $X$. For the Demand experiment we evaluate all the methods mentioned above, whereas for the dSprite experiment, we omit 2SLS, LS, and LS-QF because of their lack of scalability to high-dimensional settings.

\looseness=-1
Experiments were conducted in PyTorch 1.9.0 (Python 3.9.7), using an A100 40GB or TitanX 12GB GPU and CUDA version 11.2. They can be run in minutes for simpler models (LS, LS-QF, 2SLS) and in several hours for the larger experiments and more complex models (DFPV, NMMR). The code to reproduce our experiments can be accessed on GitHub.\footnote{https://github.com/beamlab-hsph/Neural-Moment-Matching-Regression}



\subsection{Demand Experiment}
\begin{figure}[htbp]
    \centering
    \includegraphics[width=\textwidth]{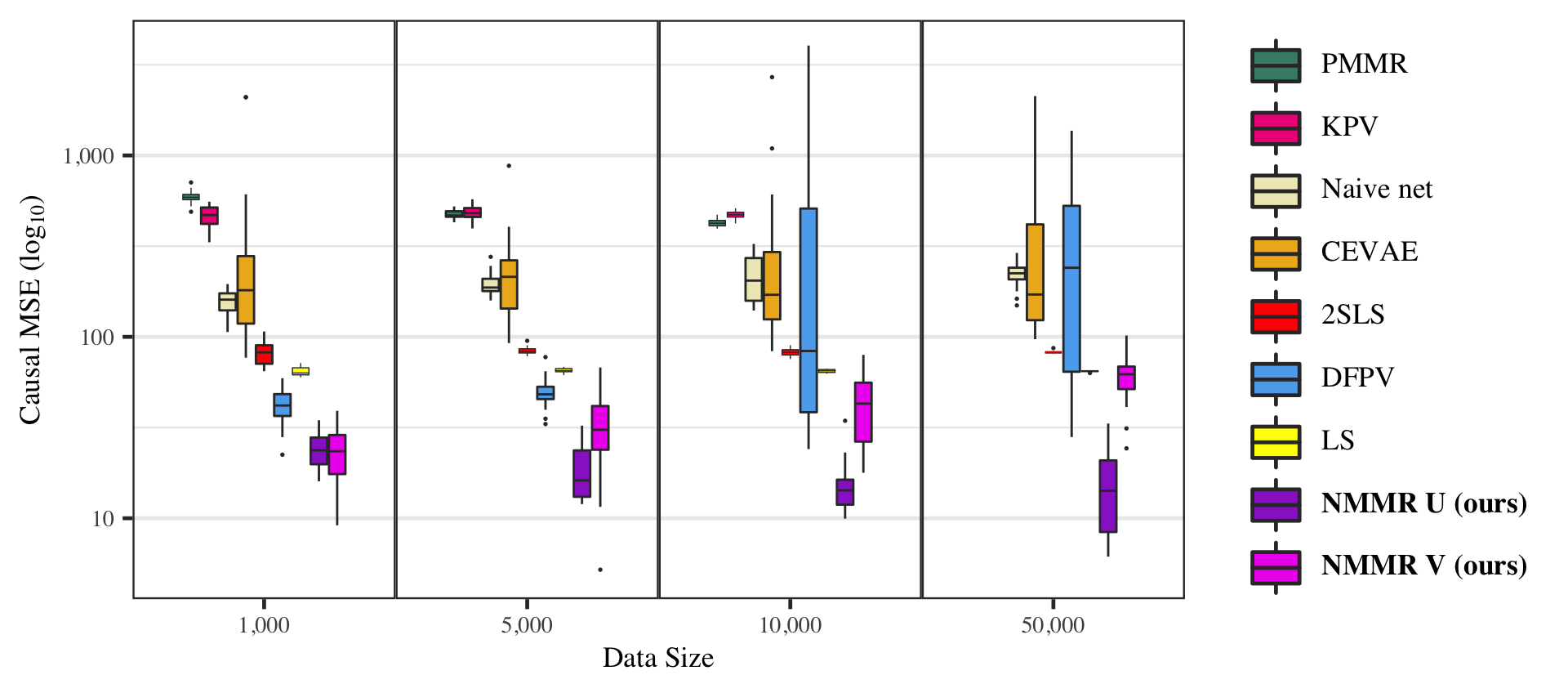}
    \caption{\textbf{NMMR-U and NMMR-V achieve state of the art performance across all sample sizes}. Causal MSE (c-MSE) of NMMR and baseline methods in the Demand experiment. Each method was replicated 20 times and evaluated on the same 10 test values of $\mathbb{E}[Y^a]$ each replicate. Each individual box plot represents 20 values of c-MSE. See Table \ref{tab:demand_table} for the statistics of each boxplots}
    \label{fig:demand_boxplot}
\end{figure}

\citet{Hartford2017-wm} introduced a data generating process for studying instrumental variable regression, and \citet{Xu2021-io} adapted it to the proximal setting. The goal is to estimate the effect of airline ticket price $A$ on sales $Y$, which is confounded by demand $U$ (e.g. seasonal fluctuations). We use the cost of fuel, $Z = (Z_1, Z_2)$, as a treatment-inducing proxy and number of views at a ticket reservation website, $W$, as an outcome-inducing proxy (Figure \ref{fig:demand_dag}). Additional simulation details and the structural equations underlying the causal DAG can be found in Appendix \ref{appendix:demand_data_gen}.

Each method was trained on simulated datasets with sample sizes of 1000, 5000, 10,000, and 50,000. To assess the performance of each method, we evaluated $a$ at 10 equally-spaced intervals between 10 and 30. We compared each method's estimated potential outcomes, $\hat{E}[Y^a]$, against estimates of the truth, $E[Y^a]$, obtained from Monte Carlo simulations (10,000 replicates) of the data generating process for each $a$. The evaluation metric is the causal mean squared error (c-MSE) across the 10 evaluation points of $a$: $\frac{1}{10} \sum_{i=1}^{10} (\mathbb{E}[Y^{a_i}] - \hat{\mathbb{E}}[Y^{a_i}])^2$. For MMR-based methods, predictions are computed using a heldout dataset, $\mathcal{D}_{\mathcal{W}}$ with 1,000 draws from $W$ so $\hat{\mathbb{E}}[Y^{a_i}]=|D_{\mathcal{W}}|^{-1}\sum_{j}^{|D_{\mathcal{W}}|} \hat{h}(a_i, w_j)$, i.e. a sample average of the estimated bridge function over $W$. We performed 20 replicates for each method on each sample size, where a single replicate yields one c-MSE value. Figure \ref{fig:demand_boxplot} summarizes the c-MSE distribution for each method across the four sample sizes. NMMR-U has the lowest c-MSE across all sample sizes, with NMMR-V a close second. DFPV encounters difficulties with the larger sample sizes of 10,000 and 50,000, potentially due to convergence issues with its feature maps. Similarly, PMMR and KPV could not scale to $n=50,000$.

For a more in-depth view of the potential outcome curve estimated by each method, we provide replicate-wise potential outcome prediction curves for each of the 4 sample sizes in Figures \ref{fig:demand_predcurve1k}-\ref{fig:demand_predcurve50k}. Least Squares estimates relatively unbiased prediction curves due to the nature of the data generating process and has very low variance. LS-QF matches some of the curvature, although its c-MSE distribution (not shown) is not better than LS. Kernel-based methods, KPV and PMMR, are highly biased. DFPV is less biased, but still suffers from a lack of flexibility. Both NMMR variants demonstrate the benefit of added flexibility and have lower variance, resulting in a lower c-MSE.

Finally, we also varied the variance of the Gaussian noise terms in the structural equations for $Z$ and $W$, in order to examine how each method performs with varying quality proxies for $U$ (see Appendix \ref{appendix:noise}). Figure \ref{fig:demand_noise_boxplot}, shows that NMMR-V is more robust to proxy noise than NMMR-U. This could be because U-statistics yield unbiased, but higher variance, estimators than V-statistics, so, when proxies are less reliable, the estimated risk function $R_k(h)$ is less stable. Kernel-based methods (KPV and PMMR) perform increasingly well with noisier proxies, which is likely related to the fact that they are less data-adaptive. Figures \ref{fig:demand_noise_predcurve_cevae} through \ref{fig:demand_noise_predcurve_nmmr_v} show replication-wise prediction curves across all 72 noise levels, with one grid plot per method.

\subsection{dSprite Experiment}
\begin{figure}[htbp]
    \centering
    \includegraphics[width=\textwidth]{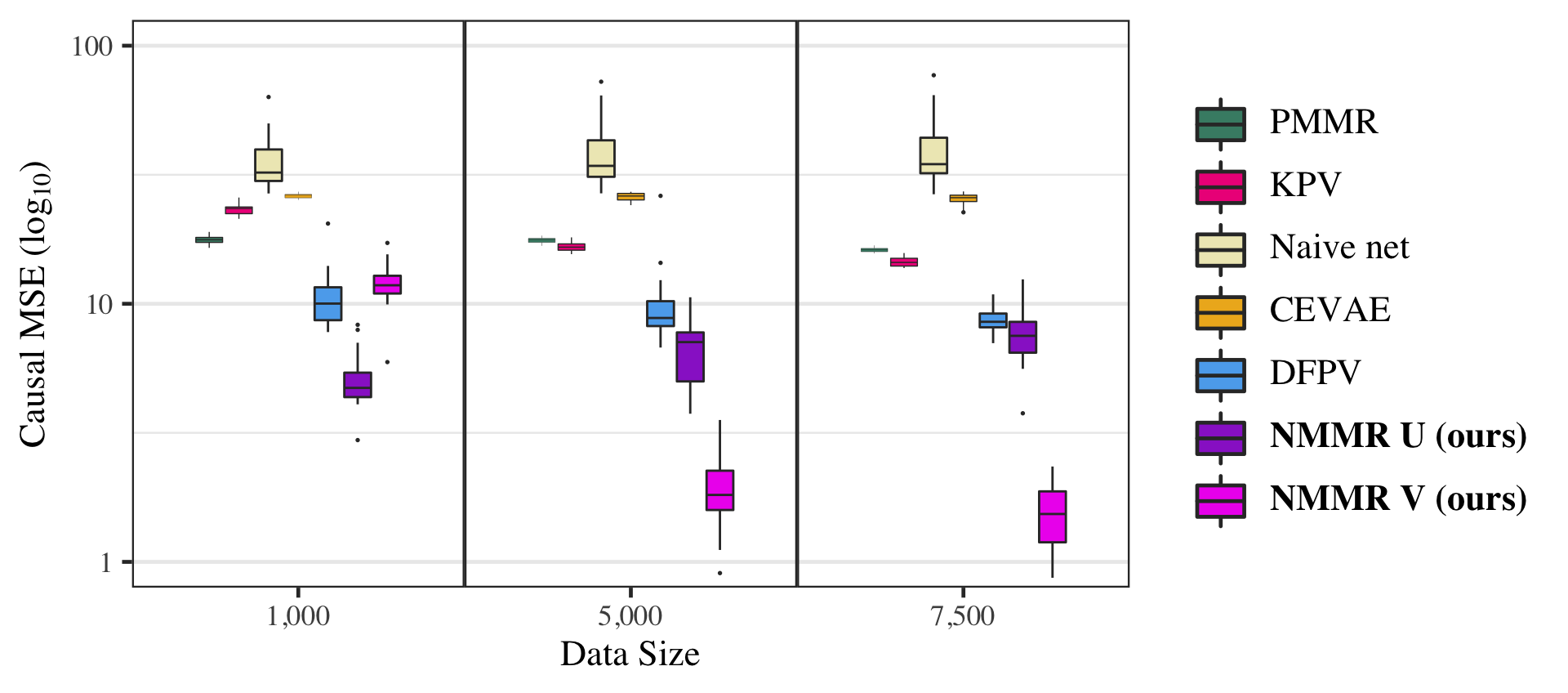}
    \caption{Causal MSE (c-MSE) of NMMR and baseline methods in the dSprite experiment. Each method was replicated 20 times and evaluated on the same 588 test images $A$ each replicate. Each individual box plot represents 20 values of c-MSE. See Table \ref{tab:dsprite_table} for the statistics of each boxplots}
    \label{fig:dsprite_boxplot}
\end{figure}
\looseness=-1
The second benchmark uses the dSprite dataset from \citet{dsprites17}, which was initially adapted to instrumental variable regression in \citet{xu2020learning}, and repurposed for proximal inference in \citet{Xu2021-io}. This image dataset consists of 2D shapes procedurally generated from 6 independent parameters: color, shape, scale, rotation, posX, and posY. All possible combinations of these parameters are present exactly once, generating 737,280 total images. In this experiment, we fix shape = heart, color = white, resulting in 245,760 images, each of which contains 64x64=4096 pixels. The causal DAG for this problem is shown in Figure \ref{fig:dsprite_dag}. The structural equations and detailed data generating mechanism underlying the causal DAG can be found in Appendix \ref{section:dsprite_datagen}.

\looseness=-1
In the DAG, $Fig(\cdot)$ represents the act of retrieving the image from the dSprite dataset with the given arguments. $A$ and $W$ are vectors representing noised images of a heart shape, where the heart has a size (\emph{scale}), orientation (\emph{rotation}), horizontal position (\emph{posX}), and vertical position (\emph{posY}). For an exemplar image $A$ and $W$, see Figure \ref{fig:dsprite_examples}. The benchmark computes
\begin{align*}
    \mathbb{E}[Y^a] &= \frac{ \frac{1}{10} \left\| vec(a)^t B \right\|_2^2 - 5000}{1000}
\end{align*}
where $B$ is a $4096\times 10$ matrix of $\mathcal{U}(0, 1)$ weights from \citet{Xu2021-io}. The observed outcome is computed as
\begin{align*}
    Y &= \frac{ \frac{1}{10} \left\| vec(A)^t B \right\|_2^2 - 5000}{1000}
                \times \frac{(31 \times U - 15.5)^2}{85.25}
            + \epsilon,
            \quad \epsilon \sim \mathcal{N}(0, 0.5)
\end{align*}
$U$ is a discrete uniform random variable with
\begin{align*}
    \mathbb{E} \left[ \frac{(31 \times U - 15.5)^2}{85.25} \right] = 1
\end{align*}
that dictates the vertical position of the shape in $A$, as well as the value of $Y$, making $U$ a confounder of $A, Y$. We hypothesized that a convolutional neural network would be exceptionally strong at recovering this information about $U$ from the images $A$ and $W$.



\looseness=-1

\looseness=-1
Similar to the Demand experiment, we trained each method on simulated datasets with sizes 1,000, 5,000, and 7,500, followed by an evaluation on the same test set as \citet{Xu2021-io}. This test set contains 588 images $A$ that span the range of scale, rotation, posX and posY values (see Appendix \ref{appendix:dsprite_test_set}) and the 588 corresponding values of $\mathbb{E}[Y^a]$. The evaluation metric is again c-MSE:
\begin{align*}
    \frac{1}{588} \sum_{i=1}^{588}
        \left(
                \mathbb{E}[Y^{a_i}]
                - \hat{\mathbb{E}}[Y^{a_i}]
        \right)^2
\end{align*}
We performed 20 replicates for each method on each sample size. Figure \ref{fig:dsprite_boxplot} shows that NMMR-U or NMMR-V is consistently lowest in c-MSE, with NMMR-V showing substantial improvement with increasing sample size. Due to the high dimensionality of the images $A$ and $W$, we could not evaluate Least Squares, LS-QF or 2SLS on this experiment. KPV and PMMR do not improve much with increasing sample size. The Naive net, which uses the same underlying convolutional neural network architecture as NMMR but is trained using observational MSE, performs second-to-worst, with a much larger c-MSE than NMMR-U or NMMR-V. This reinforces the need to use causal knowledge in scenarios where it is available. 

\section{Conclusion} \label{sec:conclusion}
\looseness=-1
In this work we have presented a novel method to estimate potential outcomes in the presence of unmeasured confounding using deep neural networks. Though our method is promising, it has several limitations. For very high dimensional data, calculating the kernel matrix $K$ in the loss function can be computationally intensive (see Appendix \ref{sec:batched_loss}). Additionally, mapping real world scenarios to DAGs that satisfy Assumption 1 is non-trivial and technically unverifiable (e.g. we cannot be truly sure that $W$ has no impact on $A$), though unverifiable assumptions are inherent to causal inference.

Further, the present work focuses on methods that estimate only the outcome bridge function, rather than also estimating the IPW bridge function, which would permit us to construct a doubly robust estimator, as is done in \citet{Cui2020-di} and \citet{Kallus2021-hi}.  However, our method extends naturally to this setting and we expect to explore such estimators in future work.

\looseness=-1
In summary, we provide a new single stage estimator and show how it can be trained on a U-statistic based loss in addition to existing approaches based on V-statistics. We further prove theoretical convergence properties of our method. On established proximal inference benchmarks, our method achieves state of the art performance in estimating causal quantities. Finally, since our approach is a single-stage neural network, it potentially unlocks new domains for causal inference where deep learning has had success, such as imaging.

\bibliography{neurips.bib}
\section*{Checklist}


\begin{enumerate}

\item For all authors...
\begin{enumerate}
  \item Do the main claims made in the abstract and introduction accurately reflect the paper's contributions and scope?
    \answerYes{The claims made in the abstract and introduction are reflected in Sections \ref{sec:NMMR}, \ref{sec:consistency}, and \ref{sec:experiments} }
  \item Did you describe the limitations of your work?
    \answerYes{See Section \ref{sec:conclusion}}
  \item Did you discuss any potential negative societal impacts of your work?
    \answerYes{See Section \ref{sec:introduction} where we discuss potential medical impacts.}
  \item Have you read the ethics review guidelines and ensured that your paper conforms to them?
    \answerYes{}
\end{enumerate}

\item If you are including theoretical results...
\begin{enumerate}
  \item Did you state the full set of assumptions of all theoretical results?
    \answerYes{See Section \ref{sec:consistency} and Appendix \ref{appendix:proofs}}
        \item Did you include complete proofs of all theoretical results?
    \answerYes{See Appendix \ref{appendix:proofs}}
\end{enumerate}

\item If you ran experiments...
\begin{enumerate}
  \item Did you include the code, data, and instructions needed to reproduce the main experimental results (either in the supplemental material or as a URL)?
    \answerYes{Available here \url{https://github.com/beamlab-hsph/Neural-Moment-Matching-Regression}}
  \item Did you specify all the training details (e.g., data splits, hyperparameters, how they were chosen)?
    \answerYes{See Appendix \ref{appendix:hp_model_arch} and \ref{appendix:experiments} as well as our code}
        \item Did you report error bars (e.g., with respect to the random seed after running experiments multiple times)?
    \answerYes{We reported results across 20 random seeds in all our Figures and Tables.}
        \item Did you include the total amount of compute and the type of resources used (e.g., type of GPUs, internal cluster, or cloud provider)?
    \answerYes{See details in Section \ref{sec:experiments}}
\end{enumerate}

\item If you are using existing assets (e.g., code, data, models) or curating/releasing new assets...
\begin{enumerate}
  \item If your work uses existing assets, did you cite the creators?
    \answerYes{We described the fork of assets from \citet{Xu2021-io} in Section \ref{sec:experiments}}
  \item Did you mention the license of the assets?
    \answerYes{See \ref{sec:experiments}, MIT license}
  \item Did you include any new assets either in the supplemental material or as a URL?
    \answerYes{Available here \url{https://github.com/beamlab-hsph/Neural-Moment-Matching-Regression}}
  \item Did you discuss whether and how consent was obtained from people whose data you're using/curating?
    \answerNA{Synthetic data}
  \item Did you discuss whether the data you are using/curating contains personally identifiable information or offensive content?
    \answerNA{Synthetic data}
\end{enumerate}

\item If you used crowdsourcing or conducted research with human subjects...
\begin{enumerate}
  \item Did you include the full text of instructions given to participants and screenshots, if applicable?
    \answerNA{}
  \item Did you describe any potential participant risks, with links to Institutional Review Board (IRB) approvals, if applicable?
    \answerNA{}
  \item Did you include the estimated hourly wage paid to participants and the total amount spent on participant compensation?
    \answerNA{}
\end{enumerate}

\end{enumerate}


\newpage
\appendix
\renewcommand{\thefigure}{S\arabic{figure}}
\setcounter{figure}{0}
\renewcommand{\thetable}{S\arabic{table}}
\setcounter{table}{0}

\section{Proofs} \label{appendix:proofs}

\begin{lemma} \label{Exp to Emp}

	Let $X$ be a random variable taking values in $\mathcal{X}$ and let $\mathcal{F}$ be a family of measurable functions with
	$f \in \mathcal{F}: \mathcal{X}^2 \to [ - M, M ]$,
	with each $f \in \mathcal{F}$ optionally, and possibly not uniquely, (partially) parameterized by $\theta_f \in \Theta_\mathcal{F}$,
	$\Lambda : \mathcal{F} \times \Theta_\mathcal{F} \to \left[ 0, M_\Lambda \right]$,
	and
	$\lambda \geq 0$.
	Then, for any $\delta > 0$, with probability at least $1 - \delta$, any IID sample $S = \left\{ x_i \right\}_{i = 1}^n$ drawn from $\mathrm{P}_X$ satisfies
	
	\begin{align*}
		\mathbb{E} f
			= \mathbb{E}_{X, X'} f \left( X, X' \right)
				&\leq \frac{1}{ n (n - 1) } \sum_{i, j = 1, i \neq j}^n f \left( x_i, x_j \right)
					+ \lambda \Lambda \left[ f, \theta_f \right]
					 \\
				&\qquad+ 2 \left(
							\mathcal{R}_{n-1} \left( \mathcal{F}_1 \right)
								+ \mathcal{R}_n \left( \mathcal{F}_2 \right)
							\right)
					+ 2 M \left( \frac{2}{n} \log \frac{1}{\delta} \right)^\frac{1}{2}
			\\
			\frac{1}{ n (n - 1) } \sum_{i, j = 1, i \neq j}^n f \left( x_i, x_j \right)
					+ \lambda \Lambda \left[ f, \theta_f \right]
				&\leq \mathbb{E} f
					+ \lambda M_\Lambda
					+ 2 \left(
							\mathcal{R}_{n-1} \left( \mathcal{F}_1 \right)
								+ \mathcal{R}_n \left( \mathcal{F}_2 \right)
					\right) \\
					&\quad+ 2 M \left( \frac{2}{n} \log \frac{1}{\delta} \right)^\frac{1}{2}
			\\ \\
		\mathbb{E} f
			= \mathbb{E}_{X, X'} f \left( X, X' \right)
				&\leq \frac{1}{ n (n - 1) } \sum_{i, j = 1, i \neq j}^n f \left( x_i, x_j \right)
					+ \lambda \Lambda \left[ f, \theta_f \right] \\
					&\qquad + 2 \left(
									\hat{\mathcal{R}}_{n - 1, S} \left( \mathcal{F}_1 \right)
										+ \mathcal{R}_S \left( \mathcal{F}_2 \right)
								\right)
				+ 6 M \left( \frac{2}{n} \log \frac{2}{\delta} \right)^\frac{1}{2}	\\
		\frac{1}{ n (n - 1) } \sum_{i, j = 1, i \neq j}^n f \left( x_i, x_j \right)
				+ \lambda \Lambda \left[ f, \theta_f \right]
			&\leq \mathbb{E} f
					+ \lambda M_\Lambda
					+ 2 \left(
							\hat{\mathcal{R}}_{n - 1, S} \left( \mathcal{F}_1 \right)
								+ \mathcal{R}_S \left( \mathcal{F}_2 \right)
						\right) \\
					&\quad+ 6 M \left( \frac{2}{n} \log \frac{2}{\delta} \right)^\frac{1}{2}	\end{align*}
	where $\mathcal{R}_S$ is the empirical Rademacher Complexity, given by:
	
	\begin{align*}
		\mathcal{R}_S \left( \mathcal{F} \right)
			= \mathbb{E}_\epsilon \sup_{ f \in \mathcal{F} } \frac{1}{n}
				\sum_{i = 1}^n \epsilon_i
				f \left( x_i \right)
	\end{align*}
where $\epsilon$ is a Rademacher random vector taking values uniformly in $\{ -1, 1 \}^n$, $S_{-i} = \left\{ x_j \right\}_{j = 1, j \neq i}^n$,
	$\hat{\mathcal{R}}_{n-1, S} \left( \mathcal{F} \right) = n^{-1} \sum_{i = 1}^n \mathcal{R}_{S_{-i}} \left( \mathcal{F} \right)$,
	$\mathcal{F}_1 = \left\{ g \ \middle| \ 
						\exists_{f \in \mathcal{F}, x' \in \mathcal{X}}
						\forall_{x \in \mathcal{X}}
							\left[ g(x) = f \left( x', x \right) \right]
					\right\}$, \\
	$\mathcal{F}_2 = \left\{ g \ \middle| \ 
						\exists_{f \in \mathcal{F}, x' \in \mathcal{X}}
						\forall_{x \in \mathcal{X}}
							\left[ g(x) = f \left( x, x' \right) \right]
					\right\}$,
	and
	$\mathcal{R}_n$ is the (expected) Rademacher complexity for a sample of size $n$, $\mathcal{R}_n = \mathbb{E}_S \mathcal{R}_S$, where the expectation is over all samples, $S$, of size $n$.
	
	Finally, if the elements of $\mathcal{F}$ are symmetric, so that $\forall_{f \in \mathcal{F}, x, x' \in \mathcal{X} } \left( f \left( x, x' \right) = f \left( x', x \right) \right)$, $\mathcal{F}_1 = \mathcal{F}_2$.

\end{lemma}

	In particular, $\theta_f$ might be the weights associated with the neural network $f$.  Note that, as is the case for neural networks, $\theta_f$ may not be uniquely determined by $f$, so that multiple $\theta$s may be associated with the same $f$.  We may also take $\Theta_\mathcal{F} = \emptyset$, so that $f$ is regarded purely as an element of $\mathcal{F}$.

\begin{proof}

	Let
	$\hat{E}_S f = \frac{1}{ n ( n - 1 ) } \sum_{i, j = 1, j \neq i} f \left( x_i, x_j \right)$,
	$\hat{E}_{S, \lambda} f = \frac{1}{ n ( n - 1 ) } \sum_{i, j = 1, j \neq i} f \left( x_i, x_j \right)
						+ \lambda \Lambda \left[ f, \theta_f \right]$,
	$\tilde{E}_S f = \frac{1}{n} \sum_{i = 1}^n E_X f \left( x_i, X \right)$,
	$\phi(S) = \sup_{f \in \mathcal{F}} \left( \hat{\mathbb{E}}_S f - \mathbb{E} f \right)$,
	$\phi^{(1)}(S) = \sup_{f \in \mathcal{F}} \left( \hat{E}_S f - \tilde{E}_S f \right)$,
	and
	$\phi^{(2)}(S) = \sup_{f \in \mathcal{F}} \left( \tilde{\mathbb{E}}_S f - \mathbb{E} f \right)$.
	Note that $\mathbb{E}_S \hat{\mathbb{E}}_S f
				= \mathbb{E}_{X, X'} f \left( X, X' \right)
				= \mathbb{E} f
				= \mathbb{E}_{X, X'} f \left( X, X' \right)
				= \mathbb{E}_S \tilde{\mathbb{E}}_S f$.
	Also, let $S' = \left\{ x_i ' \right\}_{i = 1}^n$ and let $S_i = \left\{ x_{i, j} \right\}_{j = 1}^n$ be obtained from $S$ by replacing $x_i$ by $x_i'$, so that $x_{i, j} = x_j$ for $j \neq i$ and $x_{i, i} = x_i'$.  In order to apply McDiarmid's Inequality, we must find bounds, $c_i$ such that $\left| \phi \left( S_i \right) - \phi(S) \right| \leq c_i$ for $i = 1, \dots, n$.

\begin{align*}
	\left| \phi \left( S_i \right) - \phi(S) \right|
		&= \left|
				\sup_{f \in \mathcal{F} } \left( \hat{\mathbb{E}}_{S_i} f - \mathbb{E} f \right)
					- \sup_{f \in \mathcal{F} } \left( \hat{\mathbb{E}}_S f - \mathbb{E} f \right)
			 \right| \\
		&\leq \sup_{f \in \mathcal{F} }
			\left|
				 \left( \hat{\mathbb{E}}_{S_i} f - \mathbb{E} f \right)
					- \left( \hat{\mathbb{E}}_S f - \mathbb{E} f \right)
			 \right| \\
		&= \sup_{f \in \mathcal{F} } \left| \hat{\mathbb{E}}_{S_i} f - \hat{\mathbb{E}}_S f \right| \\
		&= \sup_{f \in \mathcal{F} }
			\left|
				\frac{1}{ n ( n - 1 ) } \sum_{j, k = 1, k \neq j}^n f \left( x_{i, j}, x_{i, k} \right)
					- \frac{1}{ n ( n - 1 ) } \sum_{j, k = 1, k \neq j}^n f \left( x_j, x_k \right)
			\right| \\
		&\leq \frac{1}{ n ( n - 1 ) } \sup_{f \in \mathcal{F} }
				\sum_{j, k = 1, k \neq j}^n
					\left|
						f \left( x_{i, j}, x_{i, k} \right)
							- f \left( x_j, x_k \right)
					\right| \\
		&= \frac{1}{ n ( n - 1 ) }
				\left(
					\sum_{j = 1, j \neq k, k = i}^n \sup_{f \in \mathcal{F} }
						\left|
							f \left( x_{i, j}, x_{i, i} \right)
								- f \left( x_j, x_i \right)
						\right|
				\right. \\
				&\hspace{8em}+ \left.
					\sum_{k = 1, k \neq j, j = i}^n \sup_{f \in \mathcal{F} }
						\left|
							f \left( x_{i, i}, x_{i, k} \right)
								- f \left( x_i, x_k \right)
						\right|
				\right) \\
		&= \frac{1}{ n ( n - 1 ) }
				\left(
					\sum_{j = 1, j \neq i}^n \sup_{f \in \mathcal{F} }
						\left|
							f \left( x_j, x_i' \right)
								- f \left( x_j, x_i \right)
						\right|
					\right. \\
					&\hspace{8em}+ \left.
					\sum_{k = 1, k \neq i}^n \sup_{f \in \mathcal{F} }
						\left|
							f \left( x_i', x_k \right)
								- f \left( x_i, x_k \right)
						\right|
				\right) \\
		&= \frac{1}{ n ( n - 1 ) }
				\left(
					\sum_{j = 1, j \neq i}^n \sup_{f \in \mathcal{F} } 2 M
					+ \sum_{k = 1, k \neq i}^n \sup_{f \in \mathcal{F} } 2 M
				\right)
		\leq \frac{2}{ n ( n - 1 ) } \cdot ( n - 1 ) \cdot 2 M \\
		&= 4 M n^{-1} \\
\end{align*}
	so we can choose $c_i = c = 4 M n^{-1}$.  The exponent in McDiarmid's Inequality is then
	$ - 2 \epsilon^2 \left( \sum_{i = 1}^n c_i^2 \right)^{-1}
		= - 2 \epsilon^2 \left( n \cdot \left( 4 M n^{-1} \right)^2 \right)^{-1}
		= - 2 \epsilon^2 \left( 16 M^2 n^{-1}\right)^{-1}
		= - \frac{1}{8} n M^{-2} \epsilon^2$.
		Setting $\frac{\delta}{2} = e^{ - \frac{1}{8} n M^{-2} \epsilon^2 }$
		gives
		$\epsilon = \left( - \frac{ 8 M^2 }{n} \log \frac{\delta}{2} \right)^\frac{1}{2}
			= M \left( \frac{8}{n} \log \frac{2}{\delta} \right)^\frac{1}{2}
			= 2 M \left( \frac{2}{n} \log \frac{2}{\delta} \right)^\frac{1}{2}$.
		Then, McDiarmid's Inequality yields,

\begin{align*}
	\mathrm{P} \left[ \phi(S) - \mathbb{E}_S \phi(S) \geq \epsilon \right] \leq \frac{\delta}{2},
		\quad
	\mathrm{P} \left[ \mathbb{E}_S \phi(S) - \phi(S) \geq \epsilon \right] \leq \frac{\delta}{2}
\end{align*}
	so that, with probability $1 - \frac{\delta}{2}$,
	$\phi(S) \leq \mathbb{E}_S \phi(S) + 2 M \left( \frac{2}{n} \log \frac{2}{\delta} \right)^\frac{1}{2}$.  	We now need to compute $\mathbb{E}_S \phi(S)$.  However, this is difficult to do directly, so we instead compute it separately for $\phi^{(1)}$ and $\phi^{(2)}$.  Let $\epsilon$ be a Rademacher random vector taking values uniformly in $\{ -1, 1 \}^n$. Then,

\begin{align*}
	\mathbb{E}_S \phi^{(1)}(S)
		&= \mathbb{E}_S \sup_{f \in \mathcal{F}} \left( \hat{E}_S f - \tilde{E}_S f \right) \\
		&= \mathbb{E}_S \sup_{f \in \mathcal{F}}
			\left(
				\frac{1}{ n ( n - 1 ) } \sum_{i, j = 1, j \neq i}^n f \left( x_i, x_j \right)
				- \frac{1}{n} \sum_{i = 1}^n E_X f \left( x_i, X \right)
			\right) \\
		&= \mathbb{E}_S \sup_{f \in \mathcal{F}}
			\left(
				\frac{1}{ n ( n - 1 ) } \sum_{i, j = 1, j \neq i}^n f \left( x_i, x_j \right)
				- \frac{1}{n} \sum_{i = 1}^n E_{S'} \frac{1}{n - 1}
					\sum_{j = 1, j \neq i}^n f \left( x_i, x'_j \right)
			\right) \\
		&= \mathbb{E}_S \sup_{f \in \mathcal{F}} E_{S'} \frac{1}{ n ( n - 1 ) }
			\sum_{i, j = 1, j \neq i}^n
				\left(
					f \left( x_i, x_j \right)
					- f \left( x_i, x'_j \right)
				\right) \\
		&\leq \mathbb{E}_{S, S'} \sup_{f \in \mathcal{F}} \frac{1}{ n ( n - 1 ) }
								\sum_{i, j = 1, j \neq i}^n
									\left(
										f \left( x_i, x_j \right)
										- f \left( x_i, x'_j \right)
									\right) \\
		&\leq \frac{1}{n} \sum_{i = 1}^n \mathbb{E}_{S, S'}
						\sup_{f \in \mathcal{F}} \frac{1}{ n - 1 }
							\sum_{j = 1, j \neq i}^n
								\left(
									f \left( x_i, x_j \right)
									- f \left( x_i, x'_j \right)
								\right) \\
		&= \frac{1}{n} \sum_{i = 1}^n \mathbb{E}_{x_i, S_{-i}, S'_{-i}}
					\sup_{f \in \mathcal{F}} \frac{1}{ n - 1 }
						\sum_{j = 1, j \neq i}^n
							\left(
								f \left( x_i, x_j \right)
								- f \left( x_i, x'_j \right)
							\right) \\
		&= \frac{1}{n} \sum_{i = 1}^n \mathbb{E}_\epsilon \mathbb{E}_{x_i, S_{-i}, S'_{-i}}
					\sup_{f \in \mathcal{F}} \frac{1}{ n - 1 }
						\sum_{j = 1, j \neq i}^n
							\epsilon_j
							\left(
								f \left( x_i, x_j \right)
								- f \left( x_i, x'_j \right)
							\right) \\
		&\leq \frac{1}{n} \sum_{i = 1}^n
			\left[
				\mathbb{E}_{x_i, S_{-i}, \epsilon}
					\sup_{f \in \mathcal{F}} \frac{1}{ n - 1 }
						\sum_{j = 1, j \neq i}^n
							\epsilon_j f \left( x_i, x_j \right)
			\right. \\
			&\hspace{6em}+ \left.
			    \mathbb{E}_{x_i, S'_{-i}, \epsilon}
					\sup_{f \in \mathcal{F}} \frac{1}{ n - 1 }
						\sum_{j = 1, j \neq i}^n
							- \epsilon_j f \left( x_i, x'_j \right)
			\right] \\
		&= \frac{1}{n} \sum_{i = 1}^n \mathbb{E}_{x_i}
			\left[
				\mathbb{E}_{S_{-i}, \epsilon}
					\sup_{f \in \mathcal{F}} \frac{1}{ n - 1 }
						\sum_{j = 1, j \neq i}^n
							\epsilon_j f \left( x_i, x_j \right)
			\right. \\
			&\hspace{6em}+ \left.
			    \mathbb{E}_{S_{-i}, \epsilon}
					\sup_{f \in \mathcal{F}} \frac{1}{ n - 1 }
						\sum_{j = 1, j \neq i}^n
							\epsilon_j f \left( x_i, x_j \right)
			\right] \\
		&= \frac{2}{n} \sum_{i = 1}^n \mathbb{E}_{x_i}
				\mathbb{E}_{S_{-i}, \epsilon}
					\sup_{f \in \mathcal{F}} \frac{1}{ n - 1 }
						\sum_{j = 1, j \neq i}^n
							\epsilon_j f \left( x_i, x_j \right) \\
		&= \frac{2}{n} \sum_{i = 1}^n \mathbb{E}_{x_n}
				\mathbb{E}_{S_{-n}, \epsilon}
					\sup_{f \in \mathcal{F}} \frac{1}{ n - 1 }
						\sum_{j = 1}^{n - 1}
							\epsilon_j f \left( x_n, x_j \right) \\
		&= 2 \mathbb{E}_{x_n}
				\mathbb{E}_{S_{-n}, \epsilon}
					\sup_{f \in \mathcal{F}} \frac{1}{ n - 1 }
						\sum_{j = 1}^{n - 1}
							\epsilon_j f \left( x_n, x_j \right)
		= 2 \mathbb{E}_X
				\mathbb{E}_{S_{-n}, \epsilon}
					\sup_{f \in \mathcal{F}} \frac{1}{ n - 1 }
						\sum_{i = 1}^{n - 1}
							\epsilon_i f \left( X, x_i \right) \\
		&\leq 2 \mathbb{E}_{S_{-n}, \epsilon}
				\sup_{ f \in \mathcal{F}, x \in \mathcal{X} } \frac{1}{ n - 1 }
						\sum_{i = 1}^{n - 1}
							\epsilon_i f \left( x, x_i \right)
		= 2 \mathbb{E}_{S_{-n}} \mathbb{E}_\epsilon
				\sup_{ f \in \mathcal{F}, x \in \mathcal{X} } \frac{1}{ n - 1 }
						\sum_{i = 1}^{n - 1}
							\epsilon_i f \left( x, x_i \right) \\
		&= 2 \mathcal{R}_{n - 1} \left( \mathcal{F}_1 \right)
\end{align*}
	where, in the seventh line, we note that the inner sum depends only on $S = \left\{ x_i \right\} \cup S_{-i}$ and $S'_{-i}$, but not $x'_i$, in the eighth line, we introduce Rademacher variables because reversing the order of the difference is equivalent to swapping elements between $S_{-i}$ and $S'_{-i}$ and, since the expectation is over all possible pairs of samples, its value is unchanged, in the tenth line, we note that negation simply interchanges pairs of Rademacher vectors, so the expectation is unchanged, and, in the final line,
	$\mathcal{F}_1 = \left\{ g \ \middle| \ 
							\exists_{f \in \mathcal{F}, x' \in \mathcal{X}}
							\forall_{x \in \mathcal{X}}
								\left[ g(x) = f \left( x', x \right) \right]
						\right\}$.

\begin{align*}
	\mathbb{E}_S \phi^{(2)}(S)
		&= \mathbb{E}_S \sup_{f \in \mathcal{F}}
			\left(
				\tilde{\mathbb{E}}_S f
					- \mathbb{E} f
				\right)
		= \mathbb{E}_S \sup_{f \in \mathcal{F}}
			\left(
				\tilde{\mathbb{E}}_S f
					- \mathbb{E}_{S'} \tilde{\mathbb{E}}_{S'} f
			\right)
		= \mathbb{E}_S \sup_{f \in \mathcal{F}} \mathbb{E}_{S'}
			\left(
				\tilde{\mathbb{E}}_S f
					- \tilde{\mathbb{E}}_{S'} f
			\right) \\
		&\leq \mathbb{E}_{S, S'} \sup_{f \in \mathcal{F}}
			\left(
				\tilde{\mathbb{E}}_S f
					- \tilde{\mathbb{E}}_{S'} f
			\right) \\
		&= \mathbb{E}_{S, S'} \sup_{f \in \mathcal{F}}
			\left(
				\frac{1}{n} \sum_{i = 1}^n E_X f \left( x_i, X \right)
					- \frac{1}{n} \sum_{i = 1}^n E_X f \left( x_i', X \right)
			\right) \\
		&= \mathbb{E}_{S, S'} \sup_{f \in \mathcal{F}}
			\left(
				\frac{1}{n} \sum_{i = 1}^n
					\left(
						E_X f \left( x_i, X \right)
							- E_X f \left( x_i', X \right)
					\right)
			\right) \\
		&= \mathbb{E}_\epsilon \mathbb{E}_{S, S'} \sup_{f \in \mathcal{F}}
			\left(
				\frac{1}{n} \sum_{i = 1}^n
					\epsilon_i
					\left(
						E_X f \left( x_i, X \right)
							- E_X f \left( x_i', X \right)
					\right)
			\right) \\
		&\leq \mathbb{E}_{S, \epsilon} \sup_{f \in \mathcal{F}}
				\frac{1}{n} \sum_{i = 1}^n
					\epsilon_i E_X f \left( x_i, X \right)
			+ \mathbb{E}_{S', \epsilon} \sup_{f \in \mathcal{F}}
				\frac{1}{n} \sum_{i = 1}^n
					- \epsilon_i E_X f \left( x_i', X \right) \\
		&= \mathbb{E}_{S, \epsilon} \sup_{f \in \mathcal{F}}
				\frac{1}{n} \sum_{i = 1}^n
					\epsilon_i E_X f \left( x_i, X \right)
				+ \mathbb{E}_{S, \epsilon} \sup_{f \in \mathcal{F}}
				\frac{1}{n} \sum_{i = 1}^n
					\epsilon_i E_X f \left( x_i, X \right) \\
		&= 2 \mathbb{E}_{S, \epsilon} \sup_{f \in \mathcal{F}}
				\frac{1}{n} \sum_{i = 1}^n
					\epsilon_i E_X f \left( x_i, X \right)
		\leq 2 \mathbb{E}_{S, \epsilon} \sup_{ f \in \mathcal{F} , x \in \mathcal{X} }
				\frac{1}{n} \sum_{i = 1}^n
					\epsilon_i f \left( x_i, x \right) \\
		&= 2 \mathcal{R}_n \left( \mathcal{F}_2 \right)
\end{align*}
	where, in the fifth line, we introduce Rademacher variables because changing the order of the difference is equivalent to swapping elements between $S$ and $S'$, and, since the expectation is over all possible pairs of samples, its value is unchanged, in the seventh line, we note that negation simply interchanges pairs of Rademacher vectors leaving the expectation unchanged, and, in the final line,
	$\mathcal{F}_2 = \left\{ g \ \middle| \ 
						\exists_{f \in \mathcal{F}, x' \in \mathcal{X}}
						\forall_{x \in \mathcal{X}}
							\left[ g(x) = f \left( x, x' \right) \right]
					\right\}$.
	Then,
\begin{align*}
	\mathbb{E}_S \phi(S)
		&= \mathbb{E}_S \sup_{f \in \mathcal{F}}
			\left(
				\hat{\mathbb{E}}_S f
					- \mathbb{E} f
			\right)
		= \mathbb{E}_S \sup_{f \in \mathcal{F}}
			\left(
				\hat{\mathbb{E}}_S f
				- \tilde{\mathbb{E}}_S f
				+ \tilde{\mathbb{E}}_S f
					- \mathbb{E} f
			\right) \\
		&\leq \mathbb{E}_S \sup_{f \in \mathcal{F}}
					\left(
						\hat{\mathbb{E}}_S f
							- \tilde{\mathbb{E}}_S f
					\right)
				+ \mathbb{E}_S \sup_{f \in \mathcal{F}}
					\left(
						\tilde{\mathbb{E}}_S f
							- \mathbb{E} f
					\right)
		= \mathbb{E}_S \phi^{(1)}(S) + \mathbb{E}_S \phi^{(2)}(S) \\
		&= 2 \mathcal{R}_{n - 1} \left( \mathcal{F}_1 \right)
			+ 2 \mathcal{R}_n \left( \mathcal{F}_2 \right)
\end{align*}
	
	Finally, combining the above results tells us that, with probability $1 - \frac{\delta}{2}$,
	$\phi(S)
		\leq 2 \left(
					\mathcal{R}_{n-1} \left( \mathcal{F}_1 \right)
						+ \mathcal{R}_n \left( \mathcal{F}_2 \right)
			\right)
			+ 2 M \left( \frac{2}{n} \log \frac{2}{\delta} \right)^\frac{1}{2}$
	so
		$\hat{\mathbb{E}}_S f
			\leq \mathbb{E} f
				+ 2 \left(
						\mathcal{R}_{n-1} \left( \mathcal{F}_1 \right)
							+ \mathcal{R}_n \left( \mathcal{F}_2 \right)
				\right)
				+ 2 M \left( \frac{2}{n} \log \frac{2}{\delta} \right)^\frac{1}{2}$.
	Replacing $\phi(S)$ by $\phi'(S) = \sup_{f \in \mathcal{F} } \left( \mathbb{E} f - \hat{\mathbb{E}}_S f \right)$
	in the above proof yields
		$\mathbb{E} f
			\leq \hat{\mathbb{E}}_S f
					+ 2 \left(
							\mathcal{R}_{n-1} \left( \mathcal{F}_1 \right)
								+ \mathcal{R}_n \left( \mathcal{F}_2 \right)
						\right)
					+ 2 M \left( \frac{2}{n} \log \frac{2}{\delta} \right)^\frac{1}{2}$.
	Since $\lambda \geq 0$, $0 \leq \Lambda \leq M_\Lambda$, $0 \leq \lambda \Lambda \left[ f, \theta_f \right] \leq \lambda M_\Lambda$, so we also have,
\begin{align*}
	\mathbb{E} f
			&\leq \hat{\mathbb{E}}_S f
					+ \lambda \Lambda \left[ f, \theta_f \right]
					+ 2 \left(
							\mathcal{R}_{n-1} \left( \mathcal{F}_1 \right)
								+ \mathcal{R}_n \left( \mathcal{F}_2 \right)
						\right)
					+ 2 M \left( \frac{2}{n} \log \frac{2}{\delta} \right)^\frac{1}{2} \\
			&= \hat{\mathbb{E}}_{S, \lambda} f
					+ 2 \left(
							\mathcal{R}_{n-1} \left( \mathcal{F}_1 \right)
								+ \mathcal{R}_n \left( \mathcal{F}_2 \right)
						\right)
					+ 2 M \left( \frac{2}{n} \log \frac{2}{\delta} \right)^\frac{1}{2} \\
	\\
	\hat{\mathbb{E}}_{S, \lambda} f
		= \hat{\mathbb{E}}_S f + \lambda \Lambda \left[ f, \theta_f \right]
			&\leq \mathbb{E} f
				+ \lambda M_\Lambda
				+ 2 \left(
						\mathcal{R}_{n-1} \left( \mathcal{F}_1 \right)
							+ \mathcal{R}_n \left( \mathcal{F}_2 \right)
				\right)
				+ 2 M \left( \frac{2}{n} \log \frac{2}{\delta} \right)^\frac{1}{2} \\
\end{align*}

	Using $2 \delta$ in place of $\delta$, we see that, with probability at least $1 - \delta$,
		$\mathbb{E} f
			\leq \hat{\mathbb{E}}_{S, \lambda} f
				+ 2 \left(
						\mathcal{R}_{n-1} \left( \mathcal{F}_1 \right)
							+ \mathcal{R}_n \left( \mathcal{F}_2 \right)
					\right)
				+ 2 M \left( \frac{2}{n} \log \frac{1}{\delta} \right)^\frac{1}{2}$
	and
		$\hat{\mathbb{E}}_{S, \lambda} f
			\leq \mathbb{E} f
				+ \lambda M_\Lambda
				+ 2 \left(
						\mathcal{R}_{n-1} \left( \mathcal{F}_1 \right)
							+ \mathcal{R}_n \left( \mathcal{F}_2 \right)
					\right)
				+ 2 M \left( \frac{2}{n} \log \frac{1}{\delta} \right)^\frac{1}{2}$,
	yielding the first two inequalities.
	
	In order to obtain results in terms of the empirical Rademacher complexity, $\mathcal{R}_S$, instead of the (expected) Rademacher complexity, $\mathcal{R}_n$, we need to apply McDiarmid's Inequality a second time.  Let
	$\mathcal{G}$ be a family of measurable functions with
	$g \in \mathcal{G}: \mathcal{X} \to [ - M, M ]$,
	Let
	$\hat{\mathcal{R}}_{n-1, S} \left( \mathcal{G} \right) = n^{-1} \sum_{i = 1}^n \mathcal{R}_{S_{-i}} \left( \mathcal{G} \right)$,
	so that
	$\mathbb{E}_S \hat{\mathcal{R}}_{n-1, S} \left( \mathcal{G} \right) = \mathcal{R}_{n-1} \left( \mathcal{G} \right)$.
	Then,
\begin{align*}
	&\hspace{-2em}
    \left|
		\hat{\mathcal{R}}_{n - 1, S_i} \left( \mathcal{G} \right)
		- \hat{\mathcal{R}}_{n - 1, S} \left( \mathcal{G} \right)
	\right| \\
		&= \left|
				n^{-1} \sum_{k = 1}^n
				\mathbb{E}_\epsilon \sup_{g \in \mathcal{G} } \frac{1}{n-1} \sum_{j = 1, j \neq k}^n
						\epsilon_j g \left( x_{i, j} \right)
				- n^{-1} \sum_{k = 1}^n
					\mathbb{E}_\epsilon \sup_{g \in \mathcal{G} } \frac{1}{n-1} \sum_{j = 1, j \neq k}^n
						\epsilon_j g \left( x_j \right)
			\right| \\
		&= \frac{1}{ n ( n - 1 ) }
			\left|
				\sum_{k = 1}^n
				\mathbb{E}_\epsilon \left(
					 \sup_{g \in \mathcal{G} } \sum_{j = 1, j \neq k}^n
						\epsilon_j g \left( x_{i, j} \right)
					- \sup_{g \in \mathcal{G} } \sum_{j = 1, j \neq k}^n
						\epsilon_j g \left( x_j \right)
				\right)
			\right| \\
		&\leq \frac{1}{ n ( n - 1 ) }
				\sum_{k = 1}^n
				\mathbb{E}_\epsilon
					\sup_{g \in \mathcal{G} } \sum_{j = 1, j \neq k}^n
						\left|
							\epsilon_j
							\left( g \left( x_{i, j} \right) - g \left( x_j \right) \right)
						\right| \\
		&= \frac{1}{ n ( n - 1 ) }
				\sum_{k = 1}^n
				\mathbb{E}_\epsilon
					\sup_{g \in \mathcal{G} } \left| g \left( x_i' \right) - g \left( x_i \right) \right|
					I( k \neq i ) \\
		&\leq \frac{1}{ n ( n - 1 ) }
				\sum_{k = 1}^n
				I( k \neq i )
				\mathbb{E}_\epsilon 2 M
		= \frac{1}{ n ( n - 1 ) }
				\sum_{k = 1, k \neq i}^n 2 M \\
		& = 2 M n^{-1} \\
\end{align*}
	where, in the fifth line, we note that the term inside the absolute value is potentially nonzero if and only if it involves $x_i'$, which occurs in each sum over $j$ exactly once, except in the case in which $k \neq i$, in which case it does not occur at all.

\begin{align*}
	\left|
		\mathcal{R}_{S_i} \left( \mathcal{G} \right)
		- \mathcal{R}_S \left( \mathcal{G} \right)
	\right|
		&= \left|
				\mathbb{E}_\epsilon \sup_{g \in \mathcal{G} } n^{-1} \sum_{j = 1}^n
						\epsilon_j g \left( x_{i, j} \right)
				- \mathbb{E}_\epsilon \sup_{g \in \mathcal{G} } n^{-1} \sum_{j = 1}^n
						\epsilon_j g \left( x_j \right)
			\right| \\
		&\leq n^{-1} \mathbb{E}_\epsilon \sup_{g \in \mathcal{G} }
				\sum_{j = 1}^n
					\left|
						\epsilon_j \left(
									g \left( x_{i, j} \right) - g \left( x_j \right)
								\right)
					\right|
		= n^{-1} \mathbb{E}_\epsilon \sup_{g \in \mathcal{G} }
			\left| g \left( x_i' \right) - g \left( x_i \right) \right| \\
		&\leq n^{-1} \mathbb{E}_\epsilon 2 M \\
		&= 2 M n^{-1}
\end{align*}

	Combining the above results gives,

\begin{align*}
	& \hspace{-2 em}
		\left|
			\left(
					\hat{\mathcal{R}}_{n - 1, S_i} \left( \mathcal{F}_1 \right)
						+ \mathcal{R}_{S_i} \left( \mathcal{F}_2 \right)
			\right)
			- \left(
					\hat{\mathcal{R}}_{n - 1, S} \left( \mathcal{F}_1 \right)
						+ \mathcal{R}_S \left( \mathcal{F}_2 \right)
			\right)
		\right| \\
		&= \left|
			\left(
					\hat{\mathcal{R}}_{n - 1, S_i} \left( \mathcal{F}_1 \right)
						- \hat{\mathcal{R}}_{n - 1, S} \left( \mathcal{F}_1 \right)
			\right)
			+ \left(
					\mathcal{R}_{S_i} \left( \mathcal{F}_2 \right)
						- \mathcal{R}_S \left( \mathcal{F}_2 \right)
			\right)
		\right| \\
		&\leq \left|
				\hat{\mathcal{R}}_{n - 1, S_i} \left( \mathcal{F}_1 \right)
					- \hat{\mathcal{R}}_{n - 1, S} \left( \mathcal{F}_1 \right)
			\right|
			+ \left|
				\mathcal{R}_{S_i} \left( \mathcal{F}_2 \right)
					- \mathcal{R}_S \left( \mathcal{F}_2 \right)
			\right|
		= 2 M n^{-1} + 2 M n^{-1} \\
		&= 4 M n^{-1}
\end{align*}
	This is the same value we obtained previously, so that the corresponding $\epsilon = 2 M \left( \frac{2}{n} \log \frac{2}{\delta} \right)^\frac{1}{2}$.  Then, McDiarmid's Inequality gives,

\begin{align*}
	\mathrm{P} \left[
					\left(
							\hat{\mathcal{R}}_{n - 1, S} \left( \mathcal{F}_1 \right)
								+ \mathcal{R}_S \left( \mathcal{F}_2 \right)
					\right)
					- \left(
							\mathcal{R}_{n - 1} \left( \mathcal{F}_1 \right)
								+ \mathcal{R}_n \left( \mathcal{F}_2 \right)
					\right)
					\geq \epsilon
			\right]
				&\leq \frac{\delta}{2} \\
	\mathrm{P} \left[
					\left(
							\mathcal{R}_{n - 1} \left( \mathcal{F}_1 \right)
								+ \mathcal{R}_n \left( \mathcal{F}_2 \right)
					\right)
					- \left(
							\hat{\mathcal{R}}_{n - 1, S} \left( \mathcal{F}_1 \right)
								+ \mathcal{R}_S \left( \mathcal{F}_2 \right)
					\right)
					\geq \epsilon
			\right]
				&\leq \frac{\delta}{2} \\
\end{align*}
so that, with probability
	$1 - \frac{\delta}{2}$,
	$\mathcal{R}_{n - 1} \left( \mathcal{F}_1 \right)
			+ \mathcal{R}_n \left( \mathcal{F}_2 \right)
		\leq \hat{\mathcal{R}}_{n - 1, S} \left( \mathcal{F}_1 \right)
			+ \mathcal{R}_S \left( \mathcal{F}_2 \right)
			+ 2 M \left( \frac{2}{n} \log \frac{2}{\delta} \right)^\frac{1}{2}$.
	Combining this with the previous results shows that, with probability at least $1 - \delta$,
	$\mathbb{E} f
		\leq \hat{\mathbb{E}}_{S, \lambda} f
			+ 2 \left(
					\hat{\mathcal{R}}_{n - 1, S} \left( \mathcal{F}_1 \right)
						+ \mathcal{R}_S \left( \mathcal{F}_2 \right)
				\right)
			+ \left( 2 M + 2 \cdot 2 M \right) \left( \frac{2}{n} \log \frac{2}{\delta} \right)^\frac{1}{2}
		= \hat{\mathbb{E}}_{S, \lambda} f
			+ 2 \left(
					\hat{\mathcal{R}}_{n - 1, S} \left( \mathcal{F}_1 \right)
						+ \mathcal{R}_S \left( \mathcal{F}_2 \right)
				\right)
			+ 6 M \left( \frac{2}{n} \log \frac{2}{\delta} \right)^\frac{1}{2}$
	and
	$\hat{\mathbb{E}}_{S, \lambda} f
		\leq \mathbb{E} f
			+ \lambda M_\Lambda
			+ 2 \left(
					\hat{\mathcal{R}}_{n - 1, S} \left( \mathcal{F}_1 \right)
						+ \mathcal{R}_S \left( \mathcal{F}_2 \right)
				\right)
			+ 6 M \left( \frac{2}{n} \log \frac{2}{\delta} \right)^\frac{1}{2}$,
	yielding the second pair of inequalities.
	
	Finally, if the elements of $\mathcal{F}$ are symmetric, so that
		$\forall_{ f \in \mathcal{F}, x, x' \in \mathcal{X} } f \left( x, x' \right) = f \left( x', x \right)$,
		if $g \in \mathcal{F}_1$ then
		$\exists_{ f \in \mathcal{F}, x' \in \mathcal{X} } \forall_{ x \in \mathcal{X} }
			\left( g(x) = f \left( x' , x \right) = f \left( x, x' \right) \right)$,
		so $g \in \mathcal{F}_2$ as well and $\mathcal{F}_1 \subseteq \mathcal{F}_2$.
		Likewise, if $g \in \mathcal{F}_2$ then
		$\exists_{ f \in \mathcal{F}, x' \in \mathcal{X} } \forall_{ x \in \mathcal{X} }
					\left( g(x) = f \left( x , x' \right) = f \left( x', x \right) \right)$,
		so $g \in \mathcal{F}_1$ as well and $\mathcal{F}_2 \subseteq \mathcal{F}_1$.
		Thus, $\mathcal{F}_1 = \mathcal{F}_2$.

\end{proof}

\begin{cor} \label{Imp Ineq}

	The inequalities in Lemma \ref{Exp to Emp} can be strengthened to the following:
	
	\begin{align*}
		\mathbb{E} f
			\leq &\frac{1}{ n (n - 1) } \sum_{i, j = 1, i \neq j}^n f \left( x_i, x_j \right)
					+ \lambda \Lambda \left[ f, \theta_f \right] \\
					&\qquad + 2 \mathbb{E}_X
						\left(
							\mathcal{R}_{n-1} \left( \mathcal{F}_{1, X} \right)
								+ \mathcal{R}_n \left( \mathcal{F}_{2, X} \right)
						\right)
					+ 2 M \left( \frac{2}{n} \log \frac{1}{\delta} \right)^\frac{1}{2}
			\\
			&\frac{1}{ n (n - 1) } \sum_{i, j = 1, i \neq j}^n f \left( x_i, x_j \right)
				+ \lambda \Lambda \left[ f, \theta_f \right]
				\leq \mathbb{E} f
					+ \lambda M_\Lambda \\
					&\qquad+ 2 \mathbb{E}_X
						\left(
							\mathcal{R}_{n-1} \left( \mathcal{F}_{1, X} \right)
								+ \mathcal{R}_n \left( \mathcal{F}_{2, X} \right)
						\right)
					+ 2 M \left( \frac{2}{n} \log \frac{1}{\delta} \right)^\frac{1}{2}
			\\ \\
		\mathbb{E} f
			\leq &\frac{1}{ n (n - 1) } \sum_{i, j = 1, i \neq j}^n f \left( x_i, x_j \right)
					+ \lambda \Lambda \left[ f, \theta_f \right] \\
				&\qquad+ 2 \mathbb{E}_X
						\left(
							\hat{\mathcal{R}}_{n-1, S} \left( \mathcal{F}_{1, X} \right)
								+ \mathcal{R}_S \left( \mathcal{F}_{2, X} \right)
						\right)
					+ 6 M \left( \frac{2}{n} \log \frac{2}{\delta} \right)^\frac{1}{2}
			\\
			&\frac{1}{ n (n - 1) } \sum_{i, j = 1, i \neq j}^n f \left( x_i, x_j \right)
				+ \lambda \Lambda \left[ f, \theta_f \right]
				\leq \mathbb{E} f
					+ \lambda M_\Lambda \\
					&\qquad+ 2 \mathbb{E}_X
						\left(
							\hat{\mathcal{R}}_{n-1, S} \left( \mathcal{F}_{1, X} \right)
								+ \mathcal{R}_S \left( \mathcal{F}_{2, X} \right)
						\right)
					+ 6 M  \left( \frac{2}{n} \log \frac{2}{\delta} \right)^\frac{1}{2}
\end{align*}
	$\mathcal{F}_{1, x} = \left\{ g_x \ \middle| \ 
							\exists_{ f \in \mathcal{F} }
							\forall_{ x' \in \mathcal{X} }
								\left[ g_x \left( x' \right) = f \left( x, x' \right) \right]
						\right\}$,
	$\mathcal{F}_{2, x} = \left\{ g_x \ \middle| \ 
							\exists_{ f \in \mathcal{F} }
							\forall_{ x' \in \mathcal{X} }
								\left[ g_x \left( x' \right) = f \left( x', x \right) \right]
						\right\}$
	
	If the elements of $\mathcal{F}$ are symmetric, then $\mathcal{F}_{1, x} = \mathcal{F}_{2, x}$.

\end{cor}

\begin{proof}

	This follows directly from the proof of Lemma \ref{Exp to Emp} in which we show that \\
	$\mathbb{E}_S \phi^{(1)}(S)
		\leq 2 \mathbb{E}_X \mathbb{E}_{S_{-n}, \epsilon}
			\sup_{f \in \mathcal{F}} \frac{1}{ n - 1 } \sum_{i = 1}^{n - 1}
				\epsilon_i f \left( X, x_i \right)
		= 2 \mathbb{E}_X \mathcal{R}_{n - 1} \left( \mathcal{F}_{1, X} \right)$ \\
	and
	$\mathbb{E}_S \phi^{(2)}(S)
		\leq 2 \mathbb{E}_{S, \epsilon} \sup_{f \in \mathcal{F}}
				\frac{1}{n} \sum_{i = 1}^n
					\epsilon_i E_X f \left( x_i, X \right)
		\leq 2 \mathbb{E}_X \mathbb{E}_{S, \epsilon} \sup_{f \in \mathcal{F}}
				\frac{1}{n} \sum_{i = 1}^n
					\epsilon_i f \left( x_i, X \right) \\
		= 2 \mathbb{E}_X \mathcal{R}_n \left( \mathcal{F}_{2, x} \right)$.
	Using these sharper bounds in the expressions (obtained from McDiarmid's inequality) in Lemma 3 (and using $2 \delta$ in place of $\delta$) yields the first pair of equations.
		
	In order to obtain the second pair of expressions, we again need to apply McDiarmid's inequality.  Using results from the proof of lemma \ref{Exp to Emp},
\begin{align*}
	& \hspace{-6 em}
		\left|
			\mathbb{E}_X
				\left(
						\hat{\mathcal{R}}_{n - 1, S_i} \left( \mathcal{F}_{1, X} \right)
							+ \mathcal{R}_{S_i} \left( \mathcal{F}_{2, X} \right)
				\right)
			- \mathbb{E}_X
				\left(
						\hat{\mathcal{R}}_{n - 1, S} \left( \mathcal{F}_{1, X} \right)
							+ \mathcal{R}_S \left( \mathcal{F}_{2, X} \right)
				\right)
		\right| \\
		&= \left|
				\mathbb{E}_X
					\left(
						\hat{\mathcal{R}}_{n - 1, S_i} \left( \mathcal{F}_{1, X} \right)
							- \hat{\mathcal{R}}_{n - 1, S} \left( \mathcal{F}_{1, X} \right)
					\right)
			+ \mathbb{E}_X
				\left(
					\mathcal{R}_{S_i} \left( \mathcal{F}_{2, X} \right)
						- \mathcal{R}_S \left( \mathcal{F}_{2, X} \right)
				\right)
		\right| \\
		&\leq \mathbb{E}_X
				\left|
					\hat{\mathcal{R}}_{n - 1, S_i} \left( \mathcal{F}_{1, X} \right)
						- \hat{\mathcal{R}}_{n - 1, S} \left( \mathcal{F}_{1, X} \right)
				\right|
			+ \mathbb{E}_X
				\left|
					\mathcal{R}_{S_i} \left( \mathcal{F}_{2, X} \right)
						- \mathcal{R}_S \left( \mathcal{F}_{2, X} \right)
				\right| \\
		&= \mathbb{E}_X 2 M n^{-1} + \mathbb{E}_X 2 M n^{-1}
		= 2 M n^{-1} + 2 M n^{-1} \\
		&= 4 M n^{-1}
\end{align*}
	Since this is the same value of $c_i$ we obtained in lemma \ref{Exp to Emp}, McDiarmid's Inequality holds for $\epsilon = 2 M \left( \frac{2}{n} \log \frac{2}{\delta} \right)^\frac{1}{2}$, so that, with probability at least $1 - \frac{\delta}{2}$,
	$\mathbb{E}_X \left(
			\mathcal{R}_{n - 1} \left( \mathcal{F}_{1, X} \right)
			+ \mathcal{R}_n \left( \mathcal{F}_{2, X} \right)
	\right)
		\leq \mathbb{E}_X \left(
				\hat{\mathcal{R}}_{n - 1, S} \left( \mathcal{F}_{1, X} \right)
					+ \mathcal{R}_S \left( \mathcal{F}_{2, X} \right)
	\right)
		+ 2 M \left( \frac{2}{n} \log \frac{2}{\delta} \right)^\frac{1}{2}$.
	Then, the second pair of expressions is obtained by replacing $\delta$ by $\frac{\delta}{2}$ in the first pair of expressions and combining the result with the above inequality.  Since each of these expressions hold with probability at least $1 - \frac{\delta}{2}$, their combination will hold with probability at least $1 - 2 \cdot \frac{\delta}{2} = 1 - \delta$, as claimed.
	
	If the elements of $\mathcal{F}$ is symmetric, so that
		$\forall_{ f \in \mathcal{F}, x, x' \in \mathcal{X} } f \left( x, x' \right) = f \left( x', x \right)$,
		if $g \in \mathcal{F}_{1, x}$ then
		$\exists_{ f \in \mathcal{F} } \forall_{ x' \in \mathcal{X} }
					\left[ g_x \left( x' \right) = f \left( x, x' \right) = f \left( x', x \right) \right]$,
		so $g \in \mathcal{F}_{2, x}$ as well and $\mathcal{F}_{1, x} \subseteq \mathcal{F}_{2, x}$.
		Likewise, if $ g \in \mathcal{F}_{2, x}$
		$ \exists_{ f \in \mathcal{F} } \forall_{ x' \in \mathcal{X} }
					\left[ g_x \left( x' \right) = f \left( x', x \right) = f \left( x, x' \right) \right]$,
		so $g \in \mathcal{F}_{1, x}$ as well and $\mathcal{F}_{2, x} \subseteq \mathcal{F}_{1, x}$.
		Thus, $\mathcal{F}_{1, x} = \mathcal{F}_{2, x}$.
		
\end{proof}

\begin{lemma} \label{R-Comp}
	
	Let $h \in \mathcal{H} : \mathcal{A} \times \mathcal{W} \times \mathcal{X} \to [ - M, M ]$
	such that if $h \in \mathcal{H}$, $-h \in \mathcal{H}$,
	$\mathcal{Y} \subseteq [ - M, M ]$,
	$k: \left( \mathcal{A} \times \mathcal{Z} \times \mathcal{X} \right)^2 \to \left[ - M_k, M_k \right]$,
	$\forall_{ a, a' \in \mathcal{A}, x, x' \in \mathcal{X}, z, z' \in \mathcal{Z} }
	\left(
		k \left( \left( a, x, z \right), \left( a', x', z' \right) \right)
			= k \left( \left( a', x', z' \right), \left( a, x, z \right) \right)
	\right)$,
	and $\Xi = ( A, W, X, Y, Z )$.
	Additionally, let
$f_{h, k} \left( \left( a, w, x, y, z \right), \left( a', w', x', y', z' \right) \right)
		= \left( y - h \left( a, w, x \right) \right) \\
			\times \left( y' - h \left( a', w', x' \right) \right)
			k \left( \left( a, x, z \right), \left( a', x', z' \right) \right)$
	  Then,

\begin{align*}
	\mathbb{E}_\Xi \mathcal{R}_n \left( \mathcal{F}_{1, \Xi} \right)
		&= \mathbb{E}_\Xi \mathcal{R}_n \left( \mathcal{F}_{2, \Xi} \right) \\
		&\leq 2 M \mathbb{E}_{A, X, Z, S, \epsilon} \sup_{ h \in \mathcal{H} }
					\frac{1}{n} \sum_{i = 1}^n
						\epsilon_i h \left( a_i, w_i, x_i \right)
						k \left( \left( a_i, x_i, z_i \right), \left( A, X, Z \right) \right) \\
			&\qquad+ \left( 2 \log 2 \right)^\frac{1}{2} M^2 M_k n^{ - \frac{1}{2} } \\
		&= 2 M \mathbb{E}_{A, X, Z} \mathcal{R}_n \left( \mathcal{F}'_{A, X, Z} \right)
				+ \left( 2 \log 2 \right)^\frac{1}{2} M^2 M_k n^{ - \frac{1}{2} } \\
		&\leq 2 M \mathbb{E}_{S, \epsilon}
				\sup_{ h \in \mathcal{H}, a \in \mathcal{A}, x \in \mathcal{X}, z \in \mathcal{Z} }
					\frac{1}{n} \sum_{i = 1}^n
						\epsilon_i h \left( a_i, w_i, x_i \right)
						k \left( \left( a_i, x_i, z_i \right), \left( a, x, z \right) \right) \\
			&\qquad+ \left( 2 \log 2 \right)^\frac{1}{2} M^2 M_k n^{ - \frac{1}{2} } \\
		&= 2 M \mathcal{R}_n \left( \mathcal{F}' \right)
				+ \left( 2 \log 2 \right)^\frac{1}{2} M^2 M_k n^{ - \frac{1}{2} } \\
\end{align*}
{\footnotesize
\begin{align*}
	&\mathcal{F}'_{a, x, z}
		= \left\{ f_{a, x, z} \ \middle| \
			\exists_{ h \in \mathcal{H} }
			\forall_{ a' \in \mathcal{A}, x' \in \mathcal{X}, z' \in \mathcal{Z} }
				f_{a, x, z} \left( a', w', x', z' \right)
					= h \left( a', w', x' \right) k \left( \left( a', x', z' \right), \left( a, x, z \right) \right)
		\right\} \\
	&\mathcal{F}'
		= \left\{ f \ \middle| \
			\exists_{ h \in \mathcal{H}, a \in \mathcal{A}, x \in \mathcal{X}, z \in \mathcal{Z} }
			\forall_{ a' \in \mathcal{A}, x' \in \mathcal{X}, z' \in \mathcal{Z} }
				f \left( a', w', x', z' \right)
					= h \left( a', w', x' \right) k \left( \left( a', x', z' \right), \left( a, x, z \right) \right)
		\right\}
\end{align*}
}

\end{lemma}

\begin{proof}
	
	Let
	$\mathcal{F}
		= \left\{ g \ \middle| \ \exists_{ h \in \mathcal{H} } \left( g =  f_{h, k} \right) \right\}$.
	Since, for all $h \in \mathcal{H}$, $f_{h, k}$ is manifestly symmetric, we have $\mathcal{F}_1 = \mathcal{F}_2$ and $\mathcal{F}_{1, \xi} = \mathcal{F}_{2, \xi}$, where these classes are defined in lemma \ref{Exp to Emp} and corollary \ref{Imp Ineq}, respectively.  We now compute
	$\mathbb{E}_\Xi \mathcal{R}_n \left( \mathcal{F}_{1, \Xi} \right)
		= \mathbb{E}_\Xi \mathcal{R}_n \left( \mathcal{F}_{2, \Xi} \right)$.
	
\begin{align*}
	\mathbb{E}_\Xi & \mathcal{R}_n \left( \mathcal{F}_{1, \Xi} \right)
		= \mathbb{E}_\Xi \mathcal{R}_n \left( \mathcal{F}_{2, \Xi} \right)
		= \mathbb{E}_\Xi \mathbb{E}_{S, \epsilon} \sup_{ f \in \mathcal{F} }
				\frac{1}{n} \sum_{i = 1}^n
					\epsilon_i f \left( x_i, \Xi \right) \\
		&= \mathbb{E}_{A, W, X, Y, Z, S, \epsilon} \sup_{ f \in \mathcal{F} }
				\frac{1}{n} \sum_{i = 1}^n \epsilon_i
					f \left( \left( a_i, w_i, x_i, y_i, z_i \right), \left( A, W, X, Y, Z \right) \right) \\
		&= \mathbb{E}_{A, W, X, Y, Z, S, \epsilon} \sup_{ h \in \mathcal{H} }
				\frac{1}{n} \sum_{i = 1}^n \epsilon_i
					\left( y_i - h \left( a_i, w_i, x_i \right) \right)
						\left( Y - h \left( A, W, X \right) \right) \\
				&\hspace{12em} \times k \left( \left( a_i, x_i, z_i \right), \left( A, X, Z \right) \right) \\
		&= \mathbb{E}_{A, W, X, Y, Z, S, \epsilon} \sup_{ h \in \mathcal{H} }
				\left( Y - h \left( A, W, X \right) \right)
				\cdot \frac{1}{n} \sum_{i = 1}^n \epsilon_i
						\left( y_i - h \left( a_i, w_i, x_i \right) \right) \\
				&\hspace{20em} \times k \left( \left( a_i, x_i, z_i \right), \left( A, X, Z \right) \right) \\
		&\leq \mathbb{E}_{A, W, X, Y, Z, S, \epsilon} \sup_{ h \in \mathcal{H} }
				Y \cdot \frac{1}{n} \sum_{i = 1}^n
						\epsilon_i y_i
						k \left( \left( a_i, x_i, z_i \right), \left( A, X, Z \right) \right) \\
			&\quad+ \mathbb{E}_{A, W, X, Y, Z, S, \epsilon} \sup_{ h \in \mathcal{H} }
					- Y \cdot \frac{1}{n} \sum_{i = 1}^n
							\epsilon_i h \left( a_i, w_i, x_i \right)
							k \left( \left( a_i, x_i, z_i \right), \left( A, X, Z \right) \right) \\
			&\quad+ \mathbb{E}_{A, W, X, Y, Z, S, \epsilon} \sup_{ h \in \mathcal{H} }
					- h \left( A, W, X \right)
						\cdot \frac{1}{n} \sum_{i = 1}^n
								\epsilon_i y_i
								k \left( \left( a_i, x_i, z_i \right), \left( A, X, Z \right) \right) \\
			&\quad+ \mathbb{E}_{A, W, X, Y, Z, S, \epsilon} \sup_{ h \in \mathcal{H} }
					h \left( A, W, X \right)
						\cdot \frac{1}{n} \sum_{i = 1}^n
							\epsilon_i h \left( a_i, w_i, x_i \right)
							k \left( \left( a_i, x_i, z_i \right), \left( A, X, Z \right) \right) \\
\end{align*}

	We analyze each of these four terms separately.

\begin{align*}
	\mathbb{E}_{A, W, X, Y, Z, S, \epsilon} & \sup_{ h \in \mathcal{H} }
		Y \cdot \frac{1}{n} \sum_{i = 1}^n
				\epsilon_i y_i
				k \left( \left( a_i, x_i, z_i \right), \left( A, X, Z \right) \right) \\
		&= \mathbb{E}_{A, X, Y, Z, S, \epsilon}
			Y \cdot \frac{1}{n} \sum_{i = 1}^n
				\epsilon_i y_i
				k \left( \left( a_i, x_i, z_i \right), \left( A, X, Z \right) \right) \\
		&= \mathbb{E}_{A, X, Y, Z, S}
			Y \cdot \frac{1}{n} \sum_{i = 1}^n
				\mathbb{E}_\epsilon \epsilon_i y_i
				k \left( \left( a_i, x_i, z_i \right), \left( A, X, Z \right) \right) \\
		&= \mathbb{E}_{A, X, Y, Z, S}
			Y \cdot \frac{1}{n} \sum_{i = 1}^n
				0 \cdot y_i
				k \left( \left( a_i, x_i, z_i \right), \left( A, X, Z \right) \right)
		= \mathbb{E}_Y Y \cdot 0 \\
		&= 0
\end{align*}

\begin{align*}
	\mathbb{E}_{A, W, X, Y, Z, S, \epsilon} & \sup_{ h \in \mathcal{H} }
		- Y \cdot \frac{1}{n} \sum_{i = 1}^n
				\epsilon_i h \left( a_i, w_i, x_i \right)
				k \left( \left( a_i, x_i, z_i \right), \left( A, X, Z \right) \right) \\
	&\leq \mathbb{E}_{A, X, Y, Z, S, \epsilon} \sup_{ h \in \mathcal{H} }
		\left| Y \right|
		\left|
			\frac{1}{n} \sum_{i = 1}^n
				\epsilon_i h \left( a_i, w_i, x_i \right)
				k \left( \left( a_i, x_i, z_i \right), \left( A, X, Z \right) \right)
		\right| \\
	&\leq \mathbb{E}_{A, X, Z, S, \epsilon} \sup_{ h \in \mathcal{H} }
		M \left|
				\frac{1}{n} \sum_{i = 1}^n
					\epsilon_i h \left( a_i, w_i, x_i \right)
					k \left( \left( a_i, x_i, z_i \right), \left( A, X, Z \right) \right)
		\right| \\
	&= M \mathbb{E}_{A, X, Z, S, \epsilon} \sup_{ h \in \mathcal{H} }
			\frac{1}{n} \sum_{i = 1}^n
				\epsilon_i h \left( a_i, w_i, x_i \right)
				k \left( \left( a_i, x_i, z_i \right), \left( A, X, Z \right) \right)
\end{align*}
	where the final equality is due to the fact that $h \in \mathcal{H}$ if and only if $-h \in \mathcal{H}$, so that the supremum of the sum will be equal to the supremum of its absolute value.

\begin{align*}
	&\hspace{-2em}
	    \mathbb{E}_{A, W, X, Y, Z, S, \epsilon} \sup_{ h \in \mathcal{H} }
			h \left( A, W, X \right)
				\cdot \frac{1}{n} \sum_{i = 1}^n
					\epsilon_i h \left( a_i, w_i, x_i \right)
					k \left( \left( a_i, x_i, z_i \right), \left( A, X, Z \right) \right) \\
	&\leq \mathbb{E}_{A, W, X, Z, S, \epsilon} \sup_{ h, h' \in \mathcal{H} }
			h' \left( A, W, X \right)
				\cdot \frac{1}{n} \sum_{i = 1}^n
					\epsilon_i h \left( a_i, w_i, x_i \right)
					k \left( \left( a_i, x_i, z_i \right), \left( A, X, Z \right) \right) \\
	&\leq \mathbb{E}_{A, W, X, Z, S, \epsilon} \sup_{ h, h' \in \mathcal{H} }
			\left| h' \left( A, W, X \right) \right|
			\left|
				\frac{1}{n} \sum_{i = 1}^n
					\epsilon_i h \left( a_i, w_i, x_i \right)
					k \left( \left( a_i, x_i, z_i \right), \left( A, X, Z \right) \right)
			\right| \\
	&\leq \mathbb{E}_{A, X, Z, S, \epsilon} \sup_{ h \in \mathcal{H} }
			M \left|
				\frac{1}{n} \sum_{i = 1}^n
					\epsilon_i h \left( a_i, w_i, x_i \right)
					k \left( \left( a_i, x_i, z_i \right), \left( A, X, Z \right) \right)
			\right| \\
	&= M \mathbb{E}_{A, X, Z, S, \epsilon} \sup_{ h \in \mathcal{H} }
			\frac{1}{n} \sum_{i = 1}^n
					\epsilon_i h \left( a_i, w_i, x_i \right)
					k \left( \left( a_i, x_i, z_i \right), \left( A, X, Z \right) \right)
\end{align*}
	where the final equality follows, as above, because $h \in \mathcal{H}$ if and only if $-h \in \mathcal{H}$.

\begin{align*}
\mathbb{E}_{A, W, X, Y, Z, S, \epsilon} & \sup_{ h \in \mathcal{H} }
					- h \left( A, W, X \right)
						\cdot \frac{1}{n} \sum_{i = 1}^n
								\epsilon_i y_i
								k \left( \left( a_i, x_i, z_i \right), \left( A, X, Z \right) \right) \\
	&\leq \mathbb{E}_{A, W, X, Z, S, \epsilon} \sup_{ h \in \mathcal{H} }
					\left| h \left( A, W, X \right) \right|
					\left| \frac{1}{n} \sum_{i = 1}^n
								\epsilon_i y_i
								k \left( \left( a_i, x_i, z_i \right), \left( A, X, Z \right) \right)
					\right| \\
	&\leq \mathbb{E}_{A, X, Z, S, \epsilon}
					M \left| \frac{1}{n} \sum_{i = 1}^n
								\epsilon_i y_i
								k \left( \left( a_i, x_i, z_i \right), \left( A, X, Z \right) \right)
					\right| \\
	&= M \mathbb{E}_{A, X, Z, S, \epsilon} \sup_{ h \in \{ -1, 1 \} }
					h \cdot \frac{1}{n} \sum_{i = 1}^n
								\epsilon_i y_i
								k \left( \left( a_i, x_i, z_i \right), \left( A, X, Z \right) \right) \\
	&= M \mathbb{E}_{A, X, Z, S} \mathbb{E}_\epsilon \sup_{ h \in \{ -1, 1 \} }
					\frac{1}{n} \sum_{i = 1}^n
								\epsilon_i h y_i
								k \left( \left( a_i, x_i, z_i \right), \left( A, X, Z \right) \right) \\
	&\leq M \cdot \mathbb{E}_{A, X, Z, S} M M_k \left( 2 \log 2 \right)^\frac{1}{2} n^{ - \frac{1}{2} }
	= M \cdot M M_k \left( 2 \log 2 \right)^\frac{1}{2} n^{ - \frac{1}{2} } \\
	&= \left( 2 \log 2 \right)^\frac{1}{2} M^2 M_k n^{ - \frac{1}{2} }
\end{align*}
	where the final inequality follows from Massart's Finite Lemma using $\left| y k \right| \leq M M_k$.  Combining these results gives,

\begin{align*}
	\mathbb{E}_\Xi \mathcal{R}_n \left( \mathcal{F}_{1, \Xi} \right)
		&= \mathbb{E}_\Xi \mathcal{R}_n \left( \mathcal{F}_{2, \Xi} \right) \\
		&\leq 0 + \left( 2 \log 2 \right)^\frac{1}{2} M^2 M_k n^{ - \frac{1}{2} } \\
			&\quad+ M \mathbb{E}_{A, X, Z, S, \epsilon} \sup_{ h \in \mathcal{H} }
					\frac{1}{n} \sum_{i = 1}^n
						\epsilon_i h \left( a_i, w_i, x_i \right)
						k \left( \left( a_i, x_i, z_i \right), \left( A, X, Z \right) \right) \\
			&\quad+ M \mathbb{E}_{A, X, Z, S, \epsilon} \sup_{ h \in \mathcal{H} }
					\frac{1}{n} \sum_{i = 1}^n
						\epsilon_i h \left( a_i, w_i, x_i \right)
						k \left( \left( a_i, x_i, z_i \right), \left( A, X, Z \right) \right) \\
		&= 2 M \mathbb{E}_{A, X, Z, S, \epsilon} \sup_{ h \in \mathcal{H} }
					\frac{1}{n} \sum_{i = 1}^n
						\epsilon_i h \left( a_i, w_i, x_i \right)
						k \left( \left( a_i, x_i, z_i \right), \left( A, X, Z \right) \right) \\
			&\qquad+ \left( 2 \log 2 \right)^\frac{1}{2} M^2 M_k n^{ - \frac{1}{2} } \\
		&= 2 M \mathbb{E}_{A, X, Z} \mathbb{E}_{S, \epsilon} \sup_{ h \in \mathcal{H} }
					\frac{1}{n} \sum_{i = 1}^n
						\epsilon_i h \left( a_i, w_i, x_i \right)
						k \left( \left( a_i, x_i, z_i \right), \left( A, X, Z \right) \right) \\
			&\qquad+ \left( 2 \log 2 \right)^\frac{1}{2} M^2 M_k n^{ - \frac{1}{2} } \\
		&= 2 M \mathbb{E}_{A, X, Z} \mathbb{E}_{S, \epsilon}
				\sup_{ f_{A, X, Z} \in \mathcal{F}'_{A, X, Z} }
					\frac{1}{n} \sum_{i = 1}^n
						\epsilon_i f_{A, X, Z} \left( a_i, w_i, x_i, z_i \right) \\
			&\qquad+ \left( 2 \log 2 \right)^\frac{1}{2} M^2 M_k n^{ - \frac{1}{2} } \\
		&= 2 M \mathbb{E}_{A, X, Z} \mathcal{R}_n \left( \mathcal{F}'_{A, X, Z} \right)
				+ \left( 2 \log 2 \right)^\frac{1}{2} M^2 M_k n^{ - \frac{1}{2} } \\
		&\leq 2 M \mathbb{E}_{S, \epsilon}
				\sup_{ h \in \mathcal{H}, a \in \mathcal{A}, x \in \mathcal{X}, z \in \mathcal{Z} }
					\frac{1}{n} \sum_{i = 1}^n
						\epsilon_i h \left( a_i, w_i, x_i \right)
						k \left( \left( a_i, x_i, z_i \right), \left( a, x, z \right) \right) \\
			&\qquad+ \left( 2 \log 2 \right)^\frac{1}{2} M^2 M_k n^{ - \frac{1}{2} } \\
		&= 2 M \mathbb{E}_{S, \epsilon}
				\sup_{ f \in \mathcal{F}' }
					\frac{1}{n} \sum_{i = 1}^n
						\epsilon_i f \left( a_i, w_i, x_i, z_i \right)
			+ \left( 2 \log 2 \right)^\frac{1}{2} M^2 M_k n^{ - \frac{1}{2} } \\
		&= 2 M \mathcal{R}_n \left( \mathcal{F}' \right)
				+ \left( 2 \log 2 \right)^\frac{1}{2} M^2 M_k n^{ - \frac{1}{2} } \\
\end{align*}
{\footnotesize
\begin{align*}
	&\mathcal{F}'_{a, x, z}
		= \left\{ f_{a, x, z} \ \middle| \
			\exists_{h \in \mathcal{H}}
			\forall_{ a' \in \mathcal{A}, x' \in \mathcal{X}, z' \in \mathcal{Z} }
				f_{a, x, z} \left( a', w', x', z' \right)
					= h \left( a', w', x' \right) k \left( \left( a', x', z' \right), \left( a, x, z \right) \right)
		\right\} \\
	&\mathcal{F}'
		= \left\{ f \ \middle| \
			\exists_{ h \in \mathcal{H}, a \in \mathcal{A}, x \in \mathcal{X}, z \in \mathcal{Z} }
			\forall_{ a' \in \mathcal{A}, x' \in \mathcal{X}, z' \in \mathcal{Z} }
				f \left( a', w', x', z' \right)
					= h \left( a', w', x' \right) k \left( \left( a', x', z' \right), \left( a, x, z \right) \right)
		\right\}
\end{align*}
}

\end{proof}

\begin{lemma} \label{Inner Product}

	Let $\mathcal{X}$ be a measurable space, $\mu$ be a $\sigma$-finite measure on $\mathcal{X}$, $\mathcal{F}_0$ be a collection of $\mu$-measurable functions, with $f \in \mathcal{F}_0 : \mathcal{X} \to \mathbb{R}$, $k : \mathcal{X}^2 \to \mathbb{R}$ be symmetric and measurable with respect to the product measure ($\mu \times \mu$), and $\mathcal{F}$ be the quotient space of $\mathcal{F}_0$ in which functions are identified if they are equal $\mu$-almost everywhere.  For $f, g \in \mathcal{F}_0$, define the bilinear form $\langle f, g \rangle_k = \int f(x) k(x, y) g(y) d \mu(x, y)$, where $\mu$ is the product measure.  Then, $\langle \rangle_k$ is an inner product on $\mathcal{F}$ if and only if $k$ is an Integrally Strictly Positive Definite (ISPD) kernel, so that, for all $f \in \mathcal{F}_0$ such that $f \neq 0$ $\mu$-almost everywhere, $\int f(x) k(x, y) f(y) d \mu(x, y) > 0$.  Further, if $k$ is ISPD, then it defines a metric on $\mathcal{F}$ by $d_k(f, g) = \| f - g \|_k = \langle f - g, f -g \rangle_k^\frac{1}{2}$.

\end{lemma}

\begin{proof}

	For $f, g \in \mathcal{F}_0$,

\begin{align*}
	\langle c f + g, h \rangle_k
		&= \int ( c f + g )(x) k(x, y) h(y) d \mu(x, y) \\
		&= c \int f(x) k(x, y) h(y) d \mu(x, y) + \int g(x) k(x, y) h(y) d \mu(x, y) \\
		&= c \langle f, h \rangle_k + \langle g, h \rangle_k \\
\end{align*}
\begin{align*}
	\langle f, g \rangle_k
		&= \int f(x) k(x, y) g(y) d \mu(x, y)
		= \int g(y) k(y, x) f(x) d \mu(x, y) \\
		&= \int g(x) k(x, y) f(y) d \mu(x, y) \\
		&= \langle g, f \rangle_k \\
\end{align*}
	so $\langle \rangle_k$ is a bilinear form, as claimed.

	To see that $\langle \rangle_k$ is well defined on $\mathcal{F}$.  Note, that, if $f = f'$ a.e., then

\begin{align*}
	\left| \langle f, g \rangle_k - \langle f', g \rangle_k \right|
		&= \left| \int f(x) k(x, y) g(y) d \mu(x, y) - \int f'(x) k(x, y) g(y) d \mu(x, y) \right| \\
		&\leq \int \left| f(x) - f'(x) \right| \left| k(x, y) \right| \left| g(y) \right| d \mu(x, y) \\
		&= \int \left| \left( f - f' \right)(x) \right| \left| k(x, y) \right| \left| g(y) \right| d \mu(x, y) \\
		&= \int \int \left| \left( f - f' \right)(x) \right| \left| k(x, y) \right| \left| g(y) \right| d \mu(x) d \mu(y)
		= \int 0 d \mu(y) \\
		&= 0
\end{align*}
	where, in the fourth line, we use Tonelli's Theorem and the fact that $f - f' = 0$ a.e., so $\left| f' - f \right| |k| |g| = 0$ $\mu_x$-a.e. and, thus, the inner integral is $0$.  Thus, if $f = f'$ a.e., $\langle f, g \rangle_k = \langle f', g \rangle_k$.  By the symmetry of the bilinear form, if $g = g'$ a.e. $\langle f', g \rangle_k = \langle f', g' \rangle_k$, so that, if $f = f', g = g'$ a.e. $\langle f, g \rangle_k = \langle f', g' \rangle_k$, so $\langle \rangle_k$ is well defined on $\mathcal{F}$.
	
	If $k$ is also ISPD, then, for $f \neq 0$ $\mu$-almost everywhere, $\langle f, f \rangle_k = \int f(x) k(x, y) f(y) d \mu(x, y) > 0$, so that, combined with the above results, $\langle \rangle_k$ is an inner product on $\mathcal{F}$.  Conversely, if $\langle \rangle_k$ is an inner product, then, for $f \neq 0$ $\mu$-almost everywhere, $\langle f, f \rangle_k > 0$, so $k$ is ISPD, by definition.
	
	Since $\langle \rangle_k$ is an inner product, it defines a norm $\| \|_k$ on $\mathcal{F}$ by $\| f \|_k = \langle f, f \rangle_k^\frac{1}{2}$.  Let $d_k (f, g) = \| f - g \|_k$.  Since $\| \|_k$ is a norm, $d_k( f, g ) = \| f - g \|_k = |-1| \| g - f \|_k = \| g - f \|_k = d_k( g, f )$ and, if $f \neq g$ a.e., then, $d_k( f, g ) = \| f - g \|_k > 0$, while $d_k( f, f ) = \| f - f \|_k = \| 0 \|_k = 0$.  Finally, $d_k(f, h) = \| f - h \|_k = \| f - g + g - h \|_k \leq \| f - g \|_k + \| g - h \|_k = d_k( f, g ) + d_k( g, h )$ by the subadditivity of the norm, so that the triangle inequality holds and $d_k$ is a metric on $\mathcal{F}$, as claimed.

\end{proof}

{
\renewcommand{\thetheorem}{\ref{AXZ Conv}}

\begin{theorem}

	Let $\tilde{h}_k$ minimize $R_k(h)$ and $\hat{h}_{k, U, \lambda, n}$ minimize $\hat{R}_{k, U, \lambda, n}(h)$ for $h \in \mathcal{H}$, $k : ( \mathcal{A} \times \mathcal{X} \times \mathcal{Z} )^2 \to [ - M_k, M_k ]$, $\Lambda : \mathcal{H} \times \Theta_h \to [ - 0, M_\lambda ]$, and let $h^* : \mathcal{A} \times \mathcal{W} \times \mathcal{X} \to \mathbb{R}$ satisfy $\mathbb{E} \left[ Y - h^*(A, W, X) \middle| A, X, Z \right] = 0$ $\mathrm{P}_{A, X, Z}$-almost surely, where
\begin{align*}
	R_k(h)
		&= \mathbb{E} \left[
						\left( Y - h \left( A, W, X \right) \right)
						\left( Y' - h \left( A', W', X' \right) \right)
						k \left( \left( A, X, Z \right), \left( A', X', Z' \right) \right)
					\right] \\
	\hat{R}_{k, U, \lambda, n}(h)
		&= \frac{1}{ n (n - 1) } \sum_{i, j = 1, i \neq j}^n
				\left[
					\left( y_i - h \left( a_i, w_i, x_i \right) \right)
					\left( y_j - h \left( a_j, w_j, x_j \right) \right)
				\right. \\
				&\hspace{12em} \times
				\left.
					k \left( \left( a_i, x_i, z_i \right), \left( a_j, x_j, z_j \right) \right)
				\right]
					+ \lambda \Lambda[ h, \theta_h ]
\end{align*}
	Also let,
\begin{align*}
	d_k^2 \left( h, h' \right)
		&= \mathbb{E} \left[
						\left( h \left( A, W, X \right) - h' \left( A, W, X \right) \right)
						\left( h \left( A', W', X' \right) - h' \left( A', W', X' \right) \right)
						\right.\\
						&\hspace{6em} \times \left.
						k \left( \left( A, X, Z \right), \left( A', X', Z' \right) \right)
					\right]
\end{align*}
	Then, $d_k^2 \left( h^*, h \right) = R_k(h)$ and, with probability at least $1 - \delta$,
\begin{align*}
	d_k^2 \left( h^*, \hat{h}_{k, U, \lambda, n} \right)
		&\leq d_k^2 \left( h^*, \tilde{h}_k \right)
			+ \lambda M_\lambda
			+ 8 M \mathbb{E}_{A, X, Z} \left(
					\mathcal{R}_{n-1} \left( \mathcal{F}'_{A, X, Z} \right)
						+ \mathcal{R}_n \left( \mathcal{F}'_{A, X, Z} \right)
				\right) \\
			&\qquad+ 16 M^2 M_k \left( \frac{2}{n} \log \frac{2}{\delta} \right)^\frac{1}{2}
			+ 10 \left( 2 \log 2 \right)^\frac{1}{2} M^2 M_k n^{ - \frac{1}{2} } \\
		&\leq d_k^2 \left( h^*, \tilde{h}_k \right)
			+ \lambda M_\lambda
			+ 8 M \left(
					\mathcal{R}_{n-1} \left( \mathcal{F}' \right)
						+ \mathcal{R}_n \left( \mathcal{F}' \right)
				\right) \\
			&\qquad+ 16 M^2 M_k \left( \frac{2}{n} \log \frac{2}{\delta} \right)^\frac{1}{2}
			+ 10 \left( 2 \log 2 \right)^\frac{1}{2} M^2 M_k n^{ - \frac{1}{2} }
\end{align*}
	Further, if Assumption \ref{a:ISPD} holds, so $k$ is ISPD, then $d_k$ is a metric on $L^2_{\mathcal{AXZ}}$ and, if the right hand side of the inequality goes to zero as $n$ goes to infinity, \\
	$d_k \left(
			\mathbb{E} \left[ h^* \middle| A, X, Z \right]
				- \mathbb{E} \left[ \hat{h}_{k, \lambda, n} \middle| A, X, Z \right]
		\right)
		\xrightarrow{\mathrm{P}} 0$
	so
	$\mathbb{E} \left[ \hat{h}_{k, \lambda, n} \middle| A, X, Z \right]
		\xrightarrow{\mathrm{P}}
		\mathbb{E} \left[ h^* \middle| A, X, Z \right]$
	in $d_k$.  Also,
	$\left\|
			\mathbb{E} \left[ h^* \middle| A, X, Z \right]
				- \mathbb{E} \left[ \hat{h}_{k, \lambda, n} \middle| A, X, Z \right]
		\right\|_{ \mathrm{P}_{A, X, Z} }
		\xrightarrow{\mathrm{P}} 0$
	so
	$\mathbb{E} \left[ \hat{h}_{k, \lambda, n} \middle| A, X, Z \right]
		\xrightarrow{\mathrm{P}}
		\mathbb{E} \left[ h^* \middle| A, X, Z \right]$
	in $L^2 \left( \mathrm{P}_\mathcal{A, X, Z} \right)$-norm.
{\footnotesize
\begin{align*}
	&\mathcal{F}'_{a, x, z}
		= \left\{ f_{a, x, z} \ \middle| \
			\exists_{h \in \mathcal{H}}
			\forall_{ a' \in \mathcal{A}, x' \in \mathcal{X}, z' \in \mathcal{Z} }
				f_{a, x, z} \left( a', w', x', z' \right)
					= h \left( a', w', x' \right) k \left( \left( a', x', z' \right), \left( a, x, z \right) \right)
		\right\} \\
	&\mathcal{F}'
		= \left\{ f \ \middle| \
			\exists_{ h \in \mathcal{H}, a \in \mathcal{A}, x \in \mathcal{X}, z \in \mathcal{Z} }
			\forall_{ a' \in \mathcal{A}, x' \in \mathcal{X}, z' \in \mathcal{Z} }
				f \left( a', w', x', z' \right)
					= h \left( a', w', x' \right) k \left( \left( a', x', z' \right), \left( a, x, z \right) \right)
		\right\}
\end{align*}
}
\end{theorem}

\addtocounter{theorem}{-1}
}

\begin{proof}

	Let $\Xi = \left\{ A, W, X, Y, Z \right\}$.  Since $\tilde{h}_k$ minimizes $R_k(h)$ and $\hat{h}_{k, U, \lambda, n}$ minimizes $\hat{R}_{k, U, \lambda, n}(h)$ for $h \in \mathcal{H}$,
	$\hat{R}_{k, U, \lambda, n} \left( \hat{h}_{k, U, \lambda, n} \right) \leq \hat{R}_{k, U, \lambda, n} \left( \tilde{h}_k \right)$.
	
	Taking
	$f \left( \left( a, w, x, y, z \right), \left( a', w', x', y', z' \right) \right)
		= \left( y - h \left( a, w, x \right) \right)
			\left( y' - h \left( a', w', x' \right) \right) \\
			k \left( \left( a, x, z \right), \left( a', x', z' \right) \right)$,
	noting that $| f | \leq ( M + M )^2 \cdot M_k = ( 2 M )^2 M_k = 4 M^2 M_k$,
	and applying lemma \ref{Exp to Emp} and corollary \ref{Imp Ineq} to $R_k(h) = \mathbb{E} f$ and $\hat{R}_{k, U, \lambda, n}(h) = \hat{E}_{S, \lambda} f$ tells us that, with probability at least $1 - \frac{\delta}{2}$,
	
\begin{align*}
	R_k(h)
		&\leq \hat{R}_{k, U, \lambda, n}(h)
			+ 2 \mathbb{E}_\Xi
				\left(
					\mathcal{R}_{n-1} \left( \mathcal{F}_{1, \Xi} \right)
						+ \mathcal{R}_n \left( \mathcal{F}_{2, \Xi} \right)
				\right)
			+ 8 M^2 M_k \left( \frac{2}{n} \log \frac{2}{\delta} \right)^\frac{1}{2} \\
	\hat{R}_{k, U, \lambda, n}(h)
		&\leq R_k(h)
			+ \lambda M_\lambda
			+ 2 \mathbb{E}_\Xi
				\left(
					\mathcal{R}_{n-1} \left( \mathcal{F}_{1, \Xi} \right)
						+ \mathcal{R}_n \left( \mathcal{F}_{2, \Xi} \right)
				\right)
			+ 8 M^2 M_k \left( \frac{2}{n} \log \frac{2}{\delta} \right)^\frac{1}{2}
\end{align*}
so, with probability at least $1 - \delta$,

\begin{align*}
	R_k \left( \hat{h}_{k, U, \lambda, n} \right)
		&\leq \hat{R}_{k, U, \lambda, n} \left( \hat{h}_{k, U, \lambda, n} \right)
			+ 2 \mathbb{E}_\Xi
				\left(
					\mathcal{R}_{n-1} \left( \mathcal{F}_{1, \Xi} \right)
						+ \mathcal{R}_n \left( \mathcal{F}_{2, \Xi} \right)
				\right)
			+ 8 M^2 M_k \left( \frac{2}{n} \log \frac{2}{\delta} \right)^\frac{1}{2} \\
		&\leq \hat{R}_{k, U, \lambda, n} \left( \tilde{h}_k \right)
			+ 2 \mathbb{E}_\Xi
				\left(
					\mathcal{R}_{n-1} \left( \mathcal{F}_{1, \Xi} \right)
						+ \mathcal{R}_n \left( \mathcal{F}_{2, \Xi} \right)
				\right)
			+ 8 M^2 M_k \left( \frac{2}{n} \log \frac{2}{\delta} \right)^\frac{1}{2} \\
		&\leq R_k \left( \tilde{h}_k \right)
			+ \lambda M_\lambda
			+ 4 \mathbb{E}_\Xi
				\left(
					\mathcal{R}_{n-1} \left( \mathcal{F}_{1, \Xi} \right)
						+ \mathcal{R}_n \left( \mathcal{F}_{2, \Xi} \right)
				\right)
			+ 16 M^2 M_k \left( \frac{2}{n} \log \frac{2}{\delta} \right)^\frac{1}{2} \\
\end{align*}

	Applying lemma \ref{R-Comp} yields,

\begin{align*}
	R_k \left( \hat{h}_{k, U, \lambda, n} \right)
		&\leq R_k \left( \tilde{h}_k \right)
			+ \lambda M_\lambda
			+ 4 \left(
					2 M \mathbb{E}_{A, X, Z} \mathcal{R}_{n-1} \left( \mathcal{F}'_{A, X, Z} \right)
						+ 2 M \mathbb{E}_{A, X, Z} \mathcal{R}_n \left( \mathcal{F}'_{A, X, Z} \right)
				\right) \\
			&\qquad+ 4 \left(
					\left( 2 \log 2 \right)^\frac{1}{2} M^2 M_k ( n - 1 )^{ - \frac{1}{2} }
						+ \left( 2 \log 2 \right)^\frac{1}{2} M^2 M_k n^{ - \frac{1}{2} }
				\right) \\
			&\qquad+ 16 M^2 M_k \left( \frac{2}{n} \log \frac{2}{\delta} \right)^\frac{1}{2} \\
		&\leq R_k \left( \tilde{h}_k \right)
			+ \lambda M_\lambda
			+ 8 M \mathbb{E}_{A, X, Z}
				\left(
					\mathcal{R}_{n-1} \left( \mathcal{F}'_{A, X, Z} \right)
						+ \mathcal{R}_n \left( \mathcal{F}'_{A, X, Z} \right)
				\right) \\
			&\qquad+ 4 \left( 2 \log 2 \right)^\frac{1}{2} M^2 M_k n^{ - \frac{1}{2} }
				\left(
					 \left( \frac{n}{ n - 1 } \right)^\frac{1}{2} + 1
				\right)
			+ 16 M^2 M_k \left( \frac{2}{n} \log \frac{2}{\delta} \right)^\frac{1}{2} \\
		&\leq R_k \left( \tilde{h}_k \right)
			+ \lambda M_\lambda
			+ 8 M \mathbb{E}_{A, X, Z}
				\left(
					\mathcal{R}_{n-1} \left( \mathcal{F}'_{A, X, Z} \right)
						+ \mathcal{R}_n \left( \mathcal{F}'_{A, X, Z} \right)
				\right) \\
			&\qquad+ 16 M^2 M_k \left( \frac{2}{n} \log \frac{2}{\delta} \right)^\frac{1}{2}
			+ 4 \left( 2 \log 2 \right)^\frac{1}{2} M^2 M_k n^{ - \frac{1}{2} }
						\cdot \frac{5}{2} \\
		&= R_k \left( \tilde{h}_k \right)
			+ \lambda M_\lambda
			+ 8 M \mathbb{E}_{A, X, Z}
				\left(
					\mathcal{R}_{n-1} \left( \mathcal{F}'_{A, X, Z} \right)
						+ \mathcal{R}_n \left( \mathcal{F}'_{A, X, Z} \right)
				\right) \\
			&\qquad+ 16 M^2 M_k \left( \frac{2}{n} \log \frac{2}{\delta} \right)^\frac{1}{2}
			+ 10 \left( 2 \log 2 \right)^\frac{1}{2} M^2 M_k n^{ - \frac{1}{2} } \\
		&\leq R_k \left( \tilde{h}_k \right)
			+ \lambda M_\lambda
			+ 8 M \left(
					\mathcal{R}_{n-1} \left( \mathcal{F}' \right)
						+ \mathcal{R}_n \left( \mathcal{F}' \right)
				\right)
			+ 16 M^2 M_k \left( \frac{2}{n} \log \frac{2}{\delta} \right)^\frac{1}{2} \\
			&\qquad+ 10 \left( 2 \log 2 \right)^\frac{1}{2} M^2 M_k n^{ - \frac{1}{2} } \\
\end{align*}

	By assumption, $h^*$ satisfies $\mathbb{E} \left[ Y - h^*(A, X, X) \middle| A, X, Z \right] = 0$ $\mathrm{P}_{A, X, Z}$-almost surely, so that,
	
\begin{align*}
	&\hspace{-2em} \mathbb{E} \left[ Y - h(A, W, X) \middle| A, X, Z \right]
		= \mathbb{E} \left[ Y - h^*(A, W, X) + h^*(A, W, X) - h(A, W, X) \middle| A, X, Z \right] \\
		&= \mathbb{E} \left[ Y - h^*(A, W, X) \middle| A, X, Z \right]
			+ \mathbb{E} \left[ h^*(A, W, X) - h(A, W, X) \middle| A, X, Z \right] \\
		&= 0	+ \mathbb{E} \left[ h^*(A, W, X) - h(A, W, X) \middle| A, X, Z \right] \\
		&= \mathbb{E} \left[ h^*(A, W, X) - h(A, W, X) \middle| A, X, Z \right]
\end{align*}
$\mathrm{P}_{A, X, Z}$-almost surely.  Then, 

\begin{align*}
	R_k(h)
		&= \mathbb{E} \left[
						\left( Y - h \left( A, W, X \right) \right)
						\left( Y' - h \left( A', W', X' \right) \right)
						k \left( \left( A, X, Z \right), \left( A', X', Z' \right) \right)
					\right] \\
		&= \mathbb{E} \left[
						\mathbb{E} \left[
							\left( Y - h \left( A, W, X \right) \right)
							\left( Y' - h \left( A', W', X' \right) \right)
						\right. \right. \\
						&\hspace{12em} \times \left. \left.
							k \left( \left( A, X, Z \right), \left( A', X', Z' \right) \right)
						\middle| \left( A, X, Z \right), \left( A', X', Z' \right) \right]
					\right] \\
		&= \mathbb{E} \left[
						\mathbb{E} \left[ Y - h \left( A, W, X \right) \middle| A, X, Z \right]
						\right. \\
						&\hspace{12em} \times \left.
						\mathbb{E} \left[ Y' - h \left( A', W', X' \right) \middle| A', X', Z' \right]
						\right. \\
						&\hspace{12em} \times \left.
						k \left( \left( A, X, Z \right), \left( A', X', Z' \right) \right)
					\right] \\
		&= \mathbb{E} \left[
						\mathbb{E} \left[
							h^* \left( A, W, X \right) - h \left( A, W, X \right)
						\middle| A, X, Z \right]
						\right. \\
						&\hspace{12em} \times \left.
						\mathbb{E} \left[
							h^* \left( A', W', X' \right) - h \left( A', W', X' \right)
						\middle| A', X', Z' \right]
					\right. \\
					&\hspace{12em} \times \left.
						k \left( \left( A, X, Z \right), \left( A', X', Z' \right) \right)
					\right] \\
		&= \mathbb{E} \left[
						\mathbb{E} \left[
							\left( h^* \left( A, W, X \right) - h \left( A, W, X \right) \right)
							\left( h^* \left( A', W', X' \right) - h \left( A', W', X' \right) \right)
					\right. \right. \\
					&\hspace{12em} \times \left. \left.
							k \left( \left( A, X, Z \right), \left( A', X', Z' \right) \right)
						\middle| \left( A, X, Z \right), \left( A', X', Z' \right) \right]
					\right] \\
		&= \mathbb{E} \left[
						\left( h^* \left( A, W, X \right) - h \left( A, W, X \right) \right)
						\left( h^* \left( A', W', X' \right) - h \left( A', W', X' \right) \right)
					    \right. \\
					    &\hspace{12em} \times \left.
						k \left( \left( A, X, Z \right), \left( A', X', Z' \right) \right)
					\right] \\
		&= d_k^2 \left( h^*, h \right)
\end{align*}

	Thus,

\begin{align*}
	d_k^2 \left( h^*, \hat{h}_{k, U, \lambda, n} \right)
		&= R_k \left( \hat{h}_{k, U, \lambda, n} \right) \\	
		&\leq R_k \left( \tilde{h}_k \right)
			+ \lambda M_\lambda
			+ 8 M \mathbb{E}_{A, X, Z} \left(
					\mathcal{R}_{n-1} \left( \mathcal{F}'_{A, X, Z} \right)
						+ \mathcal{R}_n \left( \mathcal{F}'_{A, X, Z} \right)
				\right) \\
			&\qquad+ 16 M^2 M_k \left( \frac{2}{n} \log \frac{2}{\delta} \right)^\frac{1}{2}
			+ 10 \left( 2 \log 2 \right)^\frac{1}{2} M^2 M_k n^{ - \frac{1}{2} } \\
		&= d_k^2 \left( h^*, \tilde{h}_k \right)
			+ \lambda M_\lambda + 8 M \mathbb{E}_{A, X, Z} \left(
					\mathcal{R}_{n-1} \left( \mathcal{F}'_{A, X, Z} \right)
						+ \mathcal{R}_n \left( \mathcal{F}'_{A, X, Z} \right)
				\right) \\
			&\qquad+ 16 M^2 M_k \left( \frac{2}{n} \log \frac{2}{\delta} \right)^\frac{1}{2}
			+ 10 \left( 2 \log 2 \right)^\frac{1}{2} M^2 M_k n^{ - \frac{1}{2} } \\
		&\leq d_k^2 \left( h^*, \tilde{h}_k \right)
			+ \lambda M_\lambda
			+ 8 M \left(
					\mathcal{R}_{n-1} \left( \mathcal{F}' \right)
						+ \mathcal{R}_n \left( \mathcal{F}' \right)
				\right)
			+ 16 M^2 M_k \left( \frac{2}{n} \log \frac{2}{\delta} \right)^\frac{1}{2} \\
			&\qquad+ 10 \left( 2 \log 2 \right)^\frac{1}{2} M^2 M_k n^{ - \frac{1}{2} }
\end{align*}

	If the right hand side of this expression goes to zero as $n$ goes to infinity, then, for any $\delta, \epsilon > 0$, we can find $n$ such that the right hand side is less than $\epsilon$.  Further, since we can do this for any value of $\delta$, we can choose a sequence of $\delta_n$s decreasing in $n$, so that $\lim_{n \to \infty} \delta_n = 0$, so that the left hand side converges in probability.  If $k$ is ISPD, by Lemma \ref{Inner Product}, $d_k$ is a metric on $L^2_\mathcal{AXZ}$.  Thus,
	$d_k \left( h^*, \hat{h}_{k, U, \lambda, n} \right) \xrightarrow{\mathrm{P}} 0$, so
	$\mathbb{E} \left[ \hat{h}_{k, U, \lambda, n} \middle| A, X, Z \right]
		\xrightarrow{\mathrm{P}}
		\mathbb{E} \left[ h^* \middle| A, X, Z \right]$,
	in $d_k$.
	Further, the fact that $k$ is Integrally Strictly Positive Definite, implies that
	$\left\|
			\mathbb{E} \left[ h^* \middle| A, X, Z \right]
				- \mathbb{E} \left[ \hat{h}_{k, U, \lambda, n} \middle| A, X, Z \right]
		\right\|_{ \mathrm{P}_{A, X, Z} }
		\xrightarrow{\mathrm{P}} 0$,
	so that
	$\mathbb{E} \left[ \hat{h}_{k, U, \lambda, n} \middle| A, X, Z \right]
		\xrightarrow{\mathrm{P}}
		\mathbb{E} \left[ h^* \middle| A, X, Z \right]$,
	in $L^2 \left( \mathrm{P}_\mathcal{AXZ} \right)$-norm,
	as well.

\end{proof}

\begin{lemma} \label{U to V}
	
	Let $f : \mathcal{X}^2 \to [-M, M]$,
	$\forall_{x \in \mathcal{X}} f \left( x, x \right) \geq 0$,
	$\hat{U}_n[f]	= \frac{1}{ n ( n - 1 ) } \sum_{i, j = 1, j \neq i}^n f \left( x_i, x_j \right)$,
	and
	$\hat{V}_n[f] = n^{-2} \sum_{i, j = 1}^n f \left( x_i, x_j \right)$.
	Then,
	$( n - 1 ) \hat{U}_n[f]
		\leq n \hat{V}_n[f]
		\leq ( n - 1 ) \hat{U}_n[f] + M$.
	
\end{lemma}

\begin{proof}

\begin{align*}
	( n - 1 ) \hat{U}_n[f]
		&= n^{-1} \sum_{i, j = 1, j \neq i}^n f \left( x_i, x_j \right)
		\leq n^{-1} \sum_{i, j = 1}^n f \left( x_i, x_j \right) \\
		&= n \hat{V}_n[f]
		= n^{-1} \left(
					\sum_{i, j = 1, j \neq i}^n f \left( x_i, x_j \right)
					+ \sum_{i, j = 1, j = i}^n f \left( x_i, x_j \right)
				\right) \\
		&\leq n^{-1} n ( n - 1 ) \hat{U}_n[f] + n^{-1} \sum_{i = 1}^n f \left( x_i, x_i \right)
		\leq ( n - 1 ) \hat{U}_n[f] + M
\end{align*}
so
\begin{align*}
	( n - 1 ) \hat{U}_n[f]
		\leq n \hat{V}_n[f]
		\leq ( n - 1 ) \hat{U}_n[f] + M
\end{align*}

\end{proof}

\begin{cor} \label{V Stat}

	Let $\tilde{h}_k$ minimize $R_k(h)$ and $\hat{h}_{k, V, \lambda, n}$ minimize $\hat{R}_{k, V, \lambda, n}(h)$ for $h \in \mathcal{H}$ and let $h^* : \mathcal{A} \times \mathcal{W} \times \mathcal{X} \to \mathbb{R}$ satisfy $\mathbb{E} \left[ Y - h^*(A, W, X) \middle| A, X, Z \right] = 0$ $\mathrm{P}_{A, X, Z}$-almost surely, where
\begin{align*}
	R_k(h)
		&= \mathbb{E} \left[
						\left( Y - h \left( A, W, X \right) \right)
						\left( Y' - h \left( A', W', X' \right) \right)
						k \left( \left( A, X, Z \right), \left( A', X', Z' \right) \right)
					\right] \\
	\hat{R}_{k, V, \lambda, n}(h)
		&= n^{-2} \sum_{i, j = 1}^n
					\left( y_i - h \left( a_i, w_i, x_i \right) \right)
					\left( y_j - h \left( a_j, w_j, x_j \right) \right)
					k \left( \left( a_i, x_i, z_i \right), \left( a_j, x_j, z_j \right) \right) \\
			&\qquad+ \lambda \Lambda[f, \theta_f]
\end{align*}
	Also let,
\begin{align*}
	d_k^2 \left( h, h' \right)
		&= \mathbb{E} \left[
						\left( h \left( A, W, X \right) - h' \left( A, W, X \right) \right)
						\left( h \left( A', W', X' \right) - h' \left( A', W', X' \right) \right)
					\right. \\
					&\hspace{4em} \left. \times
					    k \left( \left( A, X, Z \right), \left( A', X', Z' \right) \right)
					\right]
\end{align*}
	Then, $d_k^2 \left( h^*, h \right) = R_k(h)$ and, with probability at least $1 - \delta$,
\begin{align*}
	d_k^2 \left( h^*, \hat{h}_{k, V, \lambda, n} \right)
		&\leq d_k^2 \left( h^*, \tilde{h}_k \right)
			+ \lambda M_\lambda
			+ 8 M \mathbb{E}_{A, X, Z} \left(
					\mathcal{R}_{n-1} \left( \mathcal{F}'_{A, X, Z} \right)
						+ \mathcal{R}_n \left( \mathcal{F}'_{A, X, Z} \right)
				\right) \\
			&\qquad+ 16 M^2 M_k \left( \frac{2}{n} \log \frac{2}{\delta} \right)^\frac{1}{2}
			+ \left( 4 M^2 M_k + \lambda M_\lambda \right) ( n - 1 )^{-1} \\
			&\qquad + 10 \left( 2 \log 2 \right)^\frac{1}{2} M^2 M_k n^{ - \frac{1}{2} } \\
		&\leq d_k^2 \left( h^*, \tilde{h}_k \right)
			+ \lambda M_\lambda
			+ 8 M \left(
					\mathcal{R}_{n-1} \left( \mathcal{F}' \right)
						+ \mathcal{R}_n \left( \mathcal{F}' \right)
				\right)
			+ 16 M^2 M_k \left( \frac{2}{n} \log \frac{2}{\delta} \right)^\frac{1}{2} \\
			&\qquad+ \left( 4 M^2 M_k + \lambda M_\lambda \right) ( n - 1 )^{-1}
			+ 10 \left( 2 \log 2 \right)^\frac{1}{2} M^2 M_k n^{ - \frac{1}{2} } \\
\end{align*}
	Further, if Assumption \ref{a:ISPD} holds, so $k$ is ISPD, then $d_k$ is a metric on $L^2_{\mathcal{AXZ}}$ and, if the right hand side of the inequality goes to zero as $n$ goes to infinity, \\
	$d_k \left(
			\mathbb{E} \left[ h^* \middle| A, X, Z \right]
				- \mathbb{E} \left[ \hat{h}_{k, \lambda, n} \middle| A, X, Z \right]
		\right)
		\xrightarrow{\mathrm{P}} 0$
	so
	$\mathbb{E} \left[ \hat{h}_{k, \lambda, n} \middle| A, X, Z \right]
		\xrightarrow{\mathrm{P}}
		\mathbb{E} \left[ h^* \middle| A, X, Z \right]$
	in $d_k$.  Also,
	$\left\|
			\mathbb{E} \left[ h^* \middle| A, X, Z \right]
				- \mathbb{E} \left[ \hat{h}_{k, \lambda, n} \middle| A, X, Z \right]
		\right\|_{ \mathrm{P}_{A, X, Z} }
		\xrightarrow{\mathrm{P}} 0$
	so
	$\mathbb{E} \left[ \hat{h}_{k, \lambda, n} \middle| A, X, Z \right]
		\xrightarrow{\mathrm{P}}
		\mathbb{E} \left[ h^* \middle| A, X, Z \right]$
	in $L^2 \left( \mathrm{P}_\mathcal{A, X, Z} \right)$-norm.
{\footnotesize
\begin{align*}
	&\mathcal{F}'_{a, x, z}
		= \left\{ f_{a, x, z} \ \middle| \
			\exists_{h \in \mathcal{H}}
			\forall_{ a' \in \mathcal{A}, x' \in \mathcal{X}, z' \in \mathcal{Z} }
				f_{a, x, z} \left( a', w', x', z' \right)
					= h \left( a', w', x' \right) k \left( \left( a', x', z' \right), \left( a, x, z \right) \right)
		\right\} \\
	&\mathcal{F}'
		= \left\{ f \ \middle| \
			\exists_{ h \in \mathcal{H}, a \in \mathcal{A}, x \in \mathcal{X}, z \in \mathcal{Z} }
			\forall_{ a' \in \mathcal{A}, x' \in \mathcal{X}, z' \in \mathcal{Z} }
				f \left( a', w', x', z' \right)
					= h \left( a', w', x' \right) k \left( \left( a', x', z' \right), \left( a, x, z \right) \right)
		\right\}
\end{align*}
}

\end{cor}

\begin{proof}

	Defining $f$ and $\Xi$ as in Theorem \ref{AXZ Conv}, then
	$\hat{R}_{k, U, n}(h)
			= \frac{1}{ n ( n - 1 )} \sum_{i, j = 1, j \neq i}^n f_h \left( \xi_i, \xi_j \right)$,
	$\hat{R}_{k, U, \lambda, n}(h)
			= \frac{1}{ n ( n - 1 )} \sum_{i, j = 1, j \neq i}^n f_h \left( \xi_i, \xi_j \right)
				+ \lambda \Lambda[f, \theta_f]$,
	$\hat{R}_{k ,V, n}(h)
			= n^{-2} \sum_{i, j = 1}^n f_h \left( \xi_i, \xi_j \right)$,
	$\hat{R}_{k, V, \lambda, n}(h)
			= n^{-2} \sum_{i, j = 1}^n f_h \left( \xi_i, \xi_j \right) + \lambda \Lambda[f, \theta_f]$,
	Noting that $|f| \leq 4 M^2 M_k$ and applying Lemma $\ref{U to V}$ to $\hat{R}_{k, U, \lambda, n}$ and $\hat{R}_{k, V, \lambda, n}$ yields,

\begin{align*}
	\hat{R}_{k, \lambda, n}(h)
		&= \hat{R}_{k, U, n}(h) + \lambda \Lambda[f, \theta_f]
		\leq \frac{n}{n - 1} \hat{R}_{k, V, n}(h) + \lambda \Lambda[f, \theta_f] \\
		&\leq \frac{n}{n - 1} \hat{R}_{k, V, n}(h) + \frac{n}{n - 1} \lambda \Lambda[f, \theta_f]
		= \frac{n}{n - 1} \hat{R}_{k, V, \lambda, n}(h) \\
		&\leq \hat{R}_{k, U, n}(h) + ( n - 1 )^{-1} 4 M^2 M_k + \frac{n}{n - 1} \lambda \Lambda[f, \theta_f] \\
\end{align*}
	
	From the proof of Theorem \ref{AXZ Conv}, with probability at least $1 - \frac{\delta}{2}$, we have,
	
	\begin{align*}
	R_k(h)
		&\leq \hat{R}_{k, \lambda, n}(h)
			+ 2 \mathbb{E}_\Xi
				\left(
					\mathcal{R}_{n-1} \left( \mathcal{F}_{1, \Xi} \right)
						+ \mathcal{R}_n \left( \mathcal{F}_{2, \Xi} \right)
				\right)
			+ 8 M^2 M_k \left( \frac{2}{n} \log \frac{2}{\delta} \right)^\frac{1}{2} \\
	\hat{R}_{k, \lambda, n}(h)
		&\leq R_k(h)
			+ \lambda M^2
			+ 2 \mathbb{E}_\Xi
				\left(
					\mathcal{R}_{n-1} \left( \mathcal{F}_{1, \Xi} \right)
						+ \mathcal{R}_n \left( \mathcal{F}_{2, \Xi} \right)
				\right)
			+ 8 M^2 M_k \left( \frac{2}{n} \log \frac{2}{\delta} \right)^\frac{1}{2}
\end{align*}
	
	Recalling that
	$\hat{h}_{k, V, \lambda, n}$
	minimizes
	$\hat{R}_{k, V, \lambda, n}(h)$
	and
	$\tilde{h}_k$ minimizes $R_k(h)$ over $\mathcal{H}$,
	so that
	$\hat{R}_{k, V, \lambda, n} \left( \hat{h}_{k, V, \lambda, n} \right)
		\leq \hat{R}_{k, V, \lambda, n} \left( \tilde{h}_k \right)$
	and combining the above expressions gives,

\begin{align*}
	R_k & \left( \hat{h}_{k, V, \lambda, n} \right)
		\leq \hat{R}_{k, U, \lambda, n} \left( \hat{h}_{k, V, \lambda, n} \right)
			+ 2 \mathbb{E}_\Xi
				\left(
					\mathcal{R}_{n-1} \left( \mathcal{F}_{1, \Xi} \right)
						+ \mathcal{R}_n \left( \mathcal{F}_{2, \Xi} \right)
				\right)
			+ 8 M^2 M_k \left( \frac{2}{n} \log \frac{2}{\delta} \right)^\frac{1}{2} \\
		&\leq \frac{n}{n - 1} \hat{R}_{k, V, \lambda, n} \left( \hat{h}_{k, V, \lambda, n} \right)
			+ 2 \mathbb{E}_\Xi
				\left(
					\mathcal{R}_{n-1} \left( \mathcal{F}_{1, \Xi} \right)
						+ \mathcal{R}_n \left( \mathcal{F}_{2, \Xi} \right)
				\right)
			+ 8 M^2 M_k \left( \frac{2}{n} \log \frac{2}{\delta} \right)^\frac{1}{2} \\
		&\leq \frac{n}{n - 1} \hat{R}_{k, V, \lambda, n} \left( \tilde{h}_k \right)
			+ 2 \mathbb{E}_\Xi
				\left(
					\mathcal{R}_{n-1} \left( \mathcal{F}_{1, \Xi} \right)
						+ \mathcal{R}_n \left( \mathcal{F}_{2, \Xi} \right)
				\right)
			+ 8 M^2 M_k \left( \frac{2}{n} \log \frac{2}{\delta} \right)^\frac{1}{2} \\
		&\leq \hat{R}_{k, n} \left( \tilde{h}_k \right)
				+ ( n - 1 )^{-1} 4 M^2 M_k
				+ \frac{n}{n - 1} \lambda \Lambda[f, \theta_f]
			+ 2 \mathbb{E}_\Xi
				\left(
					\mathcal{R}_{n-1} \left( \mathcal{F}_{1, \Xi} \right)
						+ \mathcal{R}_n \left( \mathcal{F}_{2, \Xi} \right)
				\right) \\
			&\quad+ 8 M^2 M_k \left( \frac{2}{n} \log \frac{2}{\delta} \right)^\frac{1}{2} \\
		&\leq R_k \left( \tilde{h}_k \right)
				+ ( n - 1 )^{-1} 4 M^2 M_k
				+ \frac{n}{n - 1} \lambda M_\lambda
			+ 4 \mathbb{E}_\Xi
				\left(
					\mathcal{R}_{n-1} \left( \mathcal{F}_{1, \Xi} \right)
						+ \mathcal{R}_n \left( \mathcal{F}_{2, \Xi} \right)
				\right) \\
			&\quad+ 16 M^2 M_k \left( \frac{2}{n} \log \frac{2}{\delta} \right)^\frac{1}{2} \\
		&\leq R_k \left( \tilde{h}_k \right)
				+ \left( 4 M^2 M_k + \lambda M_\lambda \right) ( n - 1 )^{-1}
				+ \lambda M_\lambda
			+ 4 \mathbb{E}_\Xi
				\left(
					\mathcal{R}_{n-1} \left( \mathcal{F}_{1, \Xi} \right)
						+ \mathcal{R}_n \left( \mathcal{F}_{2, \Xi} \right)
				\right) \\
			&\quad+ 16 M^2 M_k \left( \frac{2}{n} \log \frac{2}{\delta} \right)^\frac{1}{2}
\end{align*}	

	Using Lemma \ref{R-Comp} and results from the proof of Theorem \ref{AXZ Conv},

\begin{align*}
	R_k \left( \hat{h}_{k, V, \lambda, n} \right)
		&\leq R_k \left( \tilde{h}_k \right)
				+ \left( 4 M^2 M_k + \lambda M_\lambda \right) ( n - 1 )^{-1}
				+ \lambda M_\lambda \\
			&\qquad+ 8 M \left(
					\mathcal{R}_{n-1} \left( \mathcal{F}' \right)
						+ \mathcal{R}_n \left( \mathcal{F}' \right)
				\right)
				+ 16 M^2 M_k \left( \frac{2}{n} \log \frac{2}{\delta} \right)^\frac{1}{2} \\
			&\qquad+ 10 \left( 2 \log 2 \right)^\frac{1}{2} M^2 M_k n^{ - \frac{1}{2} } \\
		&= R_k \left( \tilde{h}_k \right)
			+ \lambda M_\lambda
			+ 8 M \left(
					\mathcal{R}_{n-1} \left( \mathcal{F}' \right)
						+ \mathcal{R}_n \left( \mathcal{F}' \right)
				\right)
			+ 16 M^2 M_k \left( \frac{2}{n} \log \frac{2}{\delta} \right)^\frac{1}{2} \\
			&\qquad+ \left( 4 M^2 M_k + \lambda M_\lambda \right) ( n - 1 )^{-1}
			+ 10 \left( 2 \log 2 \right)^\frac{1}{2} M^2 M_k n^{ - \frac{1}{2} } \\
\end{align*}

	Recalling that $d_k^2 \left( h^*, h \right) = R_k(h)$, we have,

\begin{align*}
	d_k^2 \left( h^*, \hat{h}_{k, V, \lambda, n} \right)
		&\leq d_k^2 \left( h^*, \tilde{h}_k \right)
			+ \lambda M^2
			+ 8 M \left(
					\mathcal{R}_{n-1} \left( \mathcal{F}' \right)
						+ \mathcal{R}_n \left( \mathcal{F}' \right)
				\right)
			+ 16 M^2 M_k \left( \frac{2}{n} \log \frac{2}{\delta} \right)^\frac{1}{2} \\
			&\qquad+ \left( 4 M^2 M_k + \lambda M_\lambda \right) ( n - 1 )^{-1}
			+ 10 \left( 2 \log 2 \right)^\frac{1}{2} M^2 M_k n^{ - \frac{1}{2} } \\
\end{align*}

\end{proof}

{
\renewcommand{\thetheorem}{\ref{AWX Conv}}

\begin{theorem}

	Under Assumption \ref{a:comp W}, $h^*$ is the unique solution $\mathrm{P}_{A, W, X}$-almost surely.
	Further, if
	$\mathbb{E} \left[ \hat{h}_{k, \lambda, n} \middle| A, X, Z \right]
		\xrightarrow{\mathrm{P}}
		\mathbb{E} \left[ h^* \middle| A, X, Z \right]$,
	$\hat{h}_{k, \lambda, n} \xrightarrow{\mathrm{P}} h^*$.
	
\end{theorem}

\addtocounter{theorem}{-1}
}

\begin{proof}
	
	Let $h^*, h^{* \prime}$ both be zeros of $\mathbb{E} \left[ Y - h \left( A, W, X \right) \middle| A, X, Z \right]$, $\mathrm{P}_{A, X, Z}$-almost surely, then

\begin{align*}
	\mathbb{E} & \left[ \left( h^* - h^{* \prime} \right) \left( A, W, X \right) \middle| A, X, Z \right]
		= \mathbb{E} \left[ \left( h^* - Y + Y - h^{* \prime} \right) \left( A, W, X \right) \middle| A, X, Z \right] \\
		&= \mathbb{E} \left[ Y - h^{* \prime} \left( A, W, X \right) \middle| A, X, Z \right]
			- \mathbb{E} \left[ Y - h^* \left( A, W, X \right) \middle| A, X, Z \right]
		= 0 - 0 \\
		&= 0
\end{align*}
	$\mathrm{P}_{A, X, Z}$-almost surely, so $h^* = h^{* \prime}$ $\mathrm{P}_{A, W, X}$-almost surely.
	
	If $\mathbb{E} \left[ \hat{h}_{k, \lambda, n} \middle| A, X, Z \right]
		\xrightarrow{\mathrm{P}}
		\mathbb{E} \left[ h^* \middle| A, X, Z \right]$,
	$\left\| \mathbb{E} \left[ \left( \hat{h}_n - h^* \right) \left( A, W, X \right) \middle| A, X, Z \right] \right\|_{\mathrm{P}_{A, X, Z}}
		\xrightarrow{\mathrm{P}} 0$,
	meaning that the convergence occurs $\mathrm{P}_{A, X, Z}$-almost surely, so, by assumption,
	$\left\| \hat{h}_n - h^* \right\|_{\mathrm{P}_{A, W, X}} \xrightarrow{\mathrm{P}} 0$,
	so $\hat{h}_n \xrightarrow{\mathrm{P}} h^*$.
	
\end{proof}

\section{Hyperparameter Tuning \& Model Architecture} \label{appendix:hp_model_arch}

We tuned the architectures of the Naive Net and NMMR models on both the Demand and dSprite experiments. The Naive Net used MSE loss to estimate $Y^a$, while NMMR relied on either the U-statistic or V-statistic. 

Within each experiment, the Naive Net and NMMR models used similar architectures. In the Demand experiment, both models consisted of 2-5 (``Network depth'' in Table \ref{tab:hp_grid}) fully connected layers with a variable number (``Network width'') of hidden units. 

In the dSprite experiment, each model had two VGG-like heads \citep{Simonyan2014-kb} that took in $A$ and $W$ images and applied two blocks of \{\texttt{Conv2D}, \texttt{Conv2D}, \texttt{MaxPool2D}\} with 3 by 3 kernels. Each \texttt{Conv2D} layer had 64 filters in the first block, then 128 filters in the second block. The output of the second block was flattened, then projected to 256 dimensions. Two subsequent fully connected layers were used, with their number of units determined by the ``layer width decay'' factor in Table \ref{tab:hp_grid}. For example, if this factor was 0.5, then the two layers would have 128 and 64 units, respectively. 

We performed a grid search over the following parameters:

\begin{table}[H]
\centering
\caption{Grid of hyperparameters for our naive neural network and NMMR models.}
\resizebox{.6\textwidth}{!}{%
\begin{tabular}{@{}lll@{}}
\toprule
Hyperparameter    & Demand              & dSprite              \\ \midrule
Learning rate     & \{3e-3, 3e-4, 3e-5\} & \{3e-4, 3e-5, 3e-6\} \\
L2 penalty        & \{3e-5, 3e-6, 3e-7\}      & \{3e-6, 3e-7\}       \\
\# of epochs      & 3000                & 500                  \\
Batch size        & 1000                & 256                  \\
Layer width decay &                     & \{0.25, 0.5\}        \\
Network width     & \{10, 40, 80\}      &                      \\
Network depth     & \{2, 3, 4, 5\}      &                      \\ \bottomrule
\end{tabular}%
}\label{tab:hp_grid}
\end{table}
We selected the final hyperparameters by considering the lowest average U-statistic or V-statistic on held-out validation sets for NMMR or the MSE for the Naive Net. For the Demand experiment, we repeated this process 10 times with different random seeds and averaged the statistics. For the dSprite experiment, we had 3 repetitions. 

Full hyperparameter choices for all methods used in this work are available in our code. The hyperparameters selected for NMMR were tuned for each dataset:

\begin{table}[H]
\centering
\caption{Optimal hyperparameters for NMMR methods}
\resizebox{\textwidth}{!}{%
\begin{tabular}{@{}lllll@{}}
\toprule
Hyperparameter    & NMMR-U Demand & NMMR-U dSprite & NMMR-V Demand & NMMR-V dSprite \\ \midrule
Learning rate     & 3e-3          & 3e-5           & 3e-3          & 3e-5           \\
L2 penalty        & 3e-6          & 3e-6           & 3e-6          & 3e-7           \\
\# of epochs      & 3,000         & 500            & 3,000         & 500            \\
Batch size        & 1,000         & 256            & 1,000         & 256            \\
Layer width decay & ---           & 0.25           & ---           & 0.5            \\
Network width     & 80            & ---            & 80            & ---            \\
Network depth     & 4             & ---            & 3             & ---            \\ \bottomrule
\end{tabular}%
}
\label{tab:nmmr_hp}
\end{table}

The hyperparameters for Naive Net for each dataset were:
\begin{table}[H]
\centering
\caption{Optimal hyperparameters for Naive Net model}
\resizebox{.6\textwidth}{!}{%
\begin{tabular}{@{}lll@{}}
\toprule
Hyperparameter    & Naive Net Demand & Naive Net dSprite \\ \midrule
Learning rate     & 3e-3             & 3e-5              \\
L2 penalty        & 3e-6             & 3e-6              \\
\# of epochs      & 3,000            & 500               \\
Batch size        & 1,000            & 256               \\
Layer width decay & ---              & 0.25              \\
Network width     & 80               & ---               \\
Network depth     & 2                & ---               \\ \bottomrule
\end{tabular}%
}
\label{tab:nn_hp}
\end{table}

Another hyperparameter of note is the choice of kernel in the loss function of NMMR. Throughout, we relied on the RBF kernel:

$$
k(x_i, x_j) = e^{\frac{-||x_i-x_j||_2^2}{2\sigma^2}} $$
with $\sigma=1$. For future work, we could consider other choices of the kernel or tune the length scale parameter $\sigma$. In the dSprite experiment, the kernel function is applied to pairs of $Z$ and $A$ data. Since $A$ is an $64\times 64$ image, we chose to concatenate $(z_i, 0.05a_i)$ as input to the kernel for the $i$-th data point. One could also consider tuning this multiplicative factor, but in practice we found that it allowed for both $Z$ and $A$ to impact the result of the kernel function.

\section{Experiment Details}\label{appendix:experiments}

\subsection{Demand Data Generating Process} \label{appendix:demand_data_gen}

The Demand experiment has the following structural equations:

\begin{compactitem}
    \item Demand: $U \sim \mathcal{U}(0, 10)$
    \item Fuel cost: $[Z_1, Z_2] = [2 \sin(2\pi U/10) + \epsilon_1, 2 \cos(2 \pi U / 10) + \epsilon_2]$
    \item Web page views: $W = 7g(U) + 45 + \epsilon_3$
    \item Price: $A = 35 + (Z_1 + 3)g(U) + Z_2 + \epsilon_4$
    \item Sales: $Y = A \times \min(\exp(\frac{W-A}{10}), 5) - 5g(U) + \epsilon_5$
    \item where $g(u) = 2 \left(\frac{(u-5)^4}{600} + e^{-4(u-5)^2} + \frac{u}{10} - 2 \right)$, and $\epsilon_i \sim \mathcal{N}(0, 1)$
\end{compactitem}

\subsection{Demand causal DAG}

\begin{figure}[h]
\centering
\begin{tikzpicture}[> = stealth, shorten > = 1pt, auto, node distance = 2cm]
\tikzstyle{every state}=[
    draw = black,
    thick,
    fill = white,
    minimum size = 6mm
]
\node (A) {$A$: Price};
\node (Z) [above left = 0.5cm and 0.5cm of A] {$Z$: Cost};
\node (U) [above right = 1cm and -0.2cm of A] {$U$: Demand};
\node (Y) [right = 1cm of A] {$Y$: Sales};
\node (W) [above right = 0.5cm and 0.5cm of Y] {$W$: Views};

\path[->] (A) edge node {} (Y);
\path[->] (U) edge node {} (A);
\path[->] (U) edge node {} (Y);
\path[->] (U) edge node {} (Z);
\path[->] (U) edge node {} (W);
\path[->] (Z) edge node {} (A);
\path[->] (W) edge node {} (Y);
\end{tikzpicture}
\caption{Causal DAG for the Demand experiment.}
\label{fig:demand_dag}
\end{figure}
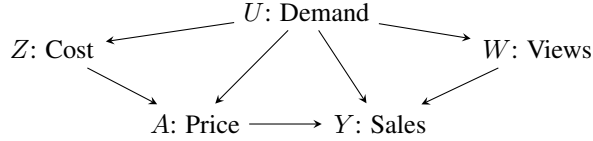

\subsection{Demand Exploratory Data Analysis} \label{appendix:demand_eda}
In Figure \ref{fig:demand_eda}, Panels A and B show that $W$ is a more informative proxy for $U$ than $Z$, although neither relationship is one-to-one. Panel C shows that the true potential outcome curve, denoted by the black curve $a \mapsto E[Y^a]$, deviates from the observed $(A, Y)$ distribution due to confounding. In particular, the largest deviation occurs at smaller values of $A$. The goal of each method is to recover this average potential outcome curve given data on $A, Z, W,$ and $Y$.

\begin{figure}[H]
    \centering
    \includegraphics[width=\textwidth]{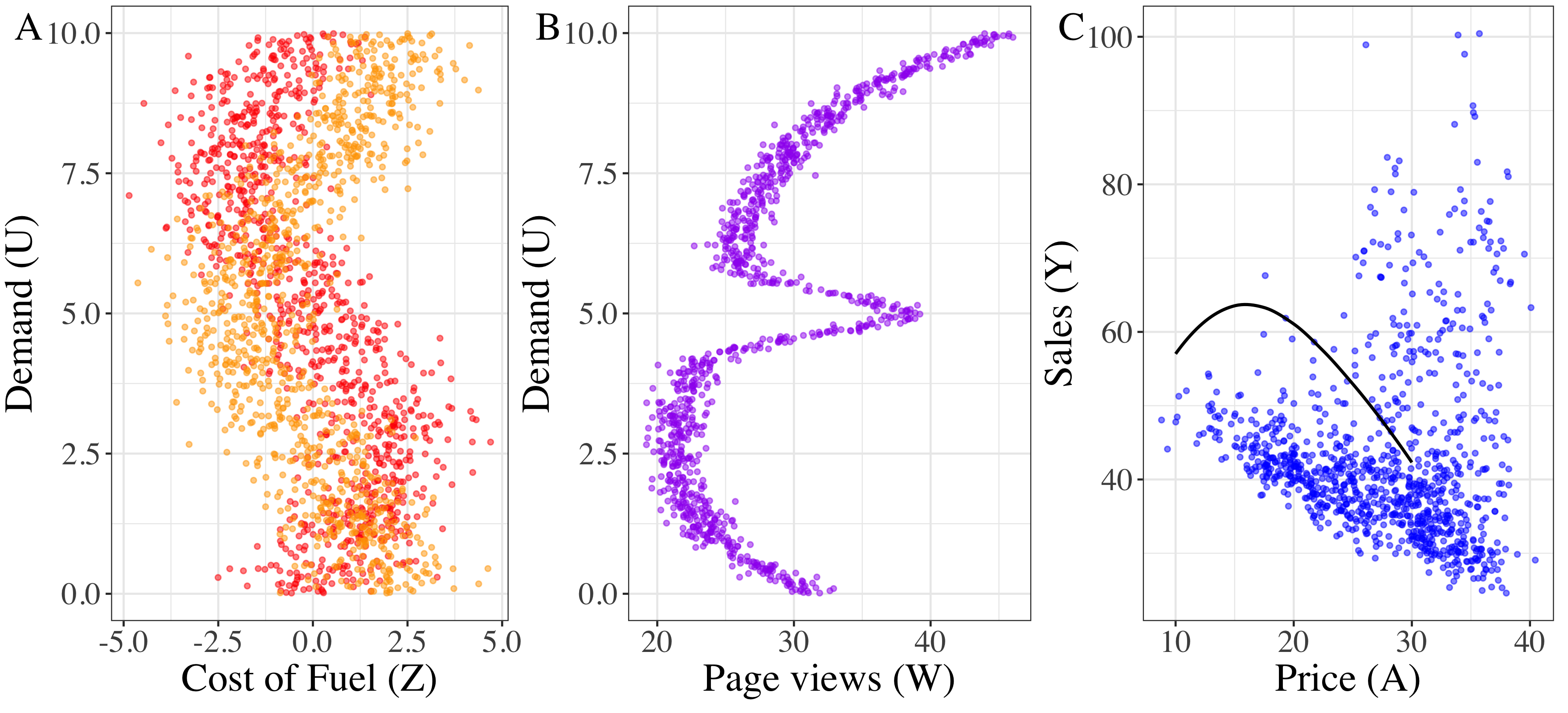}
    \caption{Views of the $A, Z, W, Y, U$ relationships. (A) $Z_1$ (red) and $Z_2$ (orange) have sinusoidal relationships with $U$, (B) $W$ has a far less noisy relationship with $U$, and (C) the observed distribution $(A, Y)$ (blue) deviates from the true average potential outcome curve (black).}
    \label{fig:demand_eda}
\end{figure}

\subsection{Demand Boxplot Statistics}
Table \ref{tab:demand_table} contains the median and interquartile ranges (in parentheses) of the c-MSE values compiled in the boxplots shown in Figure \ref{fig:demand_boxplot}. NMMR demonstrated state of the art performance on the Demand benchmark. We also extended the benchmark to include training set sizes of 10,000 and 50,000 data points, whereas \citet{Xu2021-io} originally included 1,000 and 5,000 data points. We observed that NMMR-U had strong performance across all dataset sizes. PMMR and KPV were unable to run on 50,000 training points due to computational limits on their kernel methods, while DFPV exhibited a large increase in c-MSE as training set size increased.
\begin{table}[H]
\centering
\caption{Demand Boxplot Median \& (IQR) values}
\resizebox{\textwidth}{!}{%
\begin{tabular}{@{}lcccc@{}}
\toprule
              & \multicolumn{4}{c}{Training Set Size}                                                            \\ \cmidrule(l){2-5} 
Method        & 1,000                  & 5,000                  & 10,000                & 50,000                 \\ \midrule
PMMR          & 587.51 (40.35)         & 466.5 (33.47)          & 423.1 (29.26)         & ---                    \\
KPV           & 469.94 (97.07)         & 481.32 (54.8)          & 470.62 (29.3)         & ---                    \\
Naive Net     & 160.35 (33.78)         & 186.97 (30.22)        & 204.36 (113.71))       & 224.09 (33.17)         \\
CEVAE         & 180.8 (161.26)         & 214.98 (120.88)        & 170.58 (176.1)        & 171.98 (293.27)        \\
2SLS          & 82.08 (18.82)          & 83.16 (4.51)           & 82.1 (5.55)           & 82.01 (2.22)           \\
DFPV          & 41.83 (11.78)          & 48.22 (7.73)           & 87.14 (471.59)        & 242.15 (464.38)        \\
LS            & 63.19 (5.82)           & 65.14 (2.64)           & 64.98 (2.44)          & 64.65 (0.74)            \\
\textbf{NMMR-U (ours)} & 23.68 (8.02)           & \textbf{16.21 (10.55)} & \textbf{14.25 (4.46)} & \textbf{14.27 (12.47)} \\
\textbf{NMMR-V (ours)} & \textbf{23.41 (11.26)} & 30.74 (17.73)          & 42.88 (29.45)         & 62.18 (16.97)          \\ \bottomrule
\end{tabular}%
}\label{tab:demand_table}
\end{table}

\subsection{Demand Prediction Curves} \label{section:demand_prediction_curves}

Figures \ref{fig:demand_predcurve1k}-\ref{fig:demand_predcurve50k} provide the individual predicted potential outcome curves of each method in the Demand experiment. While Figure \ref{fig:demand_boxplot} provides a summary of the c-MSE, this does not give a picture of the model's actual estimate of $\mathbb{E}[Y^a]$. These figures give an insight into ranges of $A$ for which each model provides particularly accurate or inaccurate estimates of the potential outcomes. Across all training set sizes, methods are most accurate in the region where the training observations of $Y$ are closest to the ground truth (see Figure \ref{fig:demand_eda}).

\begin{figure}[H]
     \centering
         \centering
         \includegraphics[width=\textwidth]{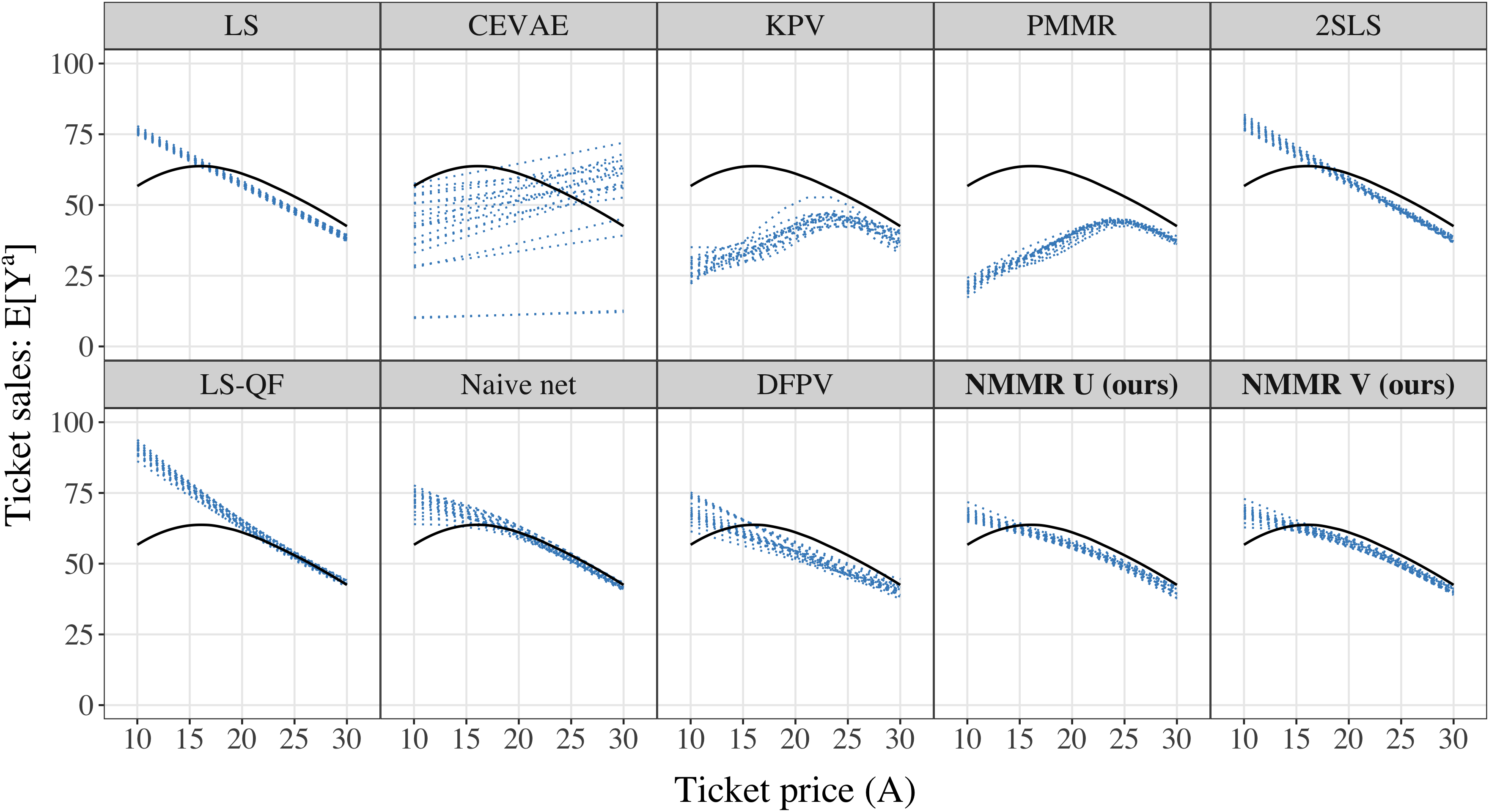}
         \caption{Demand experiment with 1,000 training data points with the true average potential outcome curves (black) and each method's predicted potential outcome curves (blue). Each method was replicated 20 times, generating one predicted curve per replicate. Note that with only a limited amount of data, most methods only recover the true curve in the later half of the range of $A$. See Figure \ref{fig:demand_eda}.}
         \label{fig:demand_predcurve1k}
\end{figure}

\begin{figure}[H]
     \centering
         \centering
         \includegraphics[width=\textwidth]{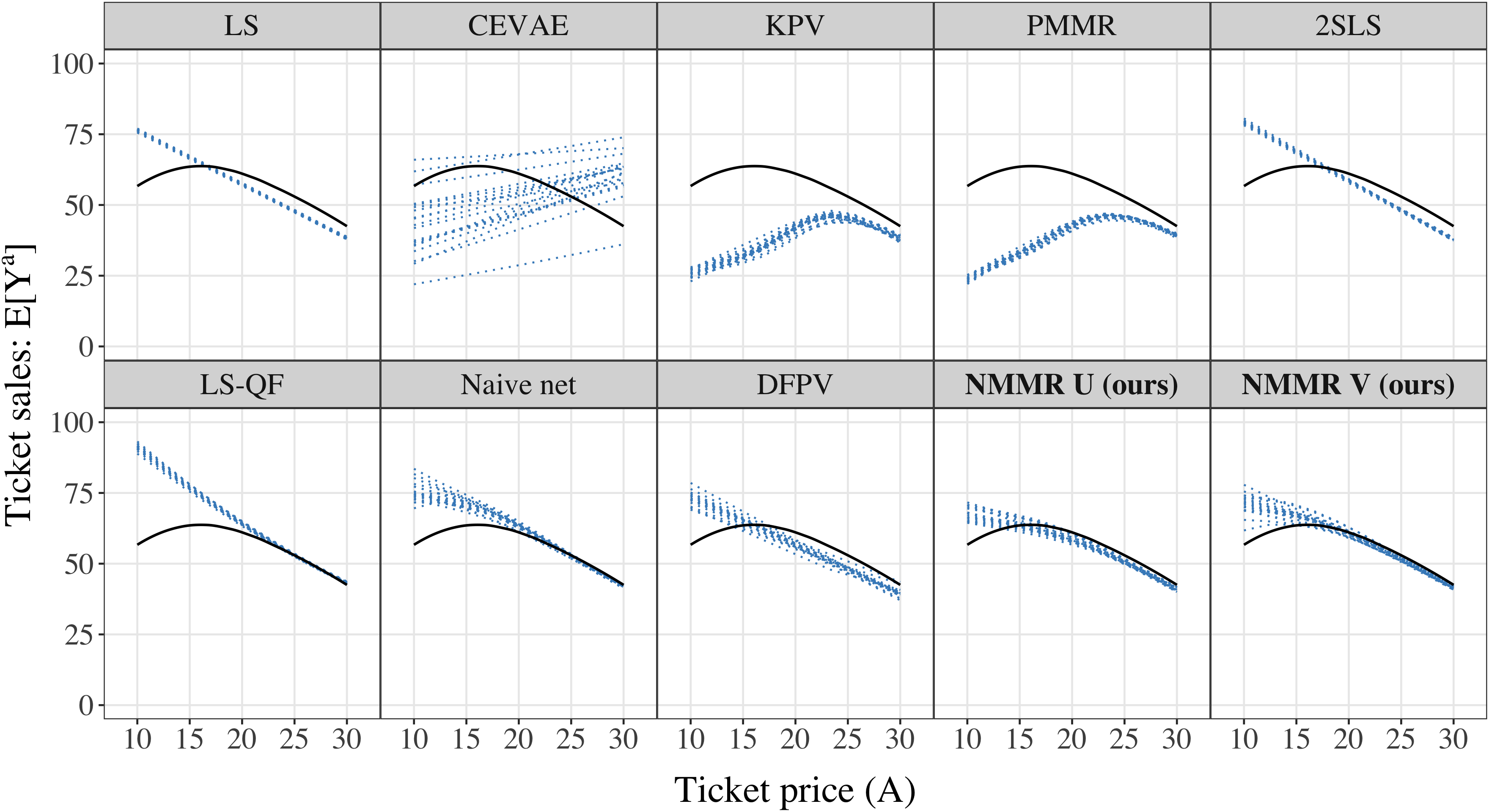}
         \caption{Demand experiment with 5,000 training data points with the true average potential outcome curves (black) and each method's predicted potential outcome curves (blue). Each method was replicated 20 times, generating one predicted curve per replicate. Note that now, NMMR begins to adjust in the range of $A\in[10,20]$ and bend down towards the true curve. NMMR is empirically accounting for the unobserved confounder $U$ through the proxy variables.}
         \label{fig:demand_predcurve5k}
\end{figure}

\begin{figure}[H]
     \centering
         \centering
         \includegraphics[width=\textwidth]{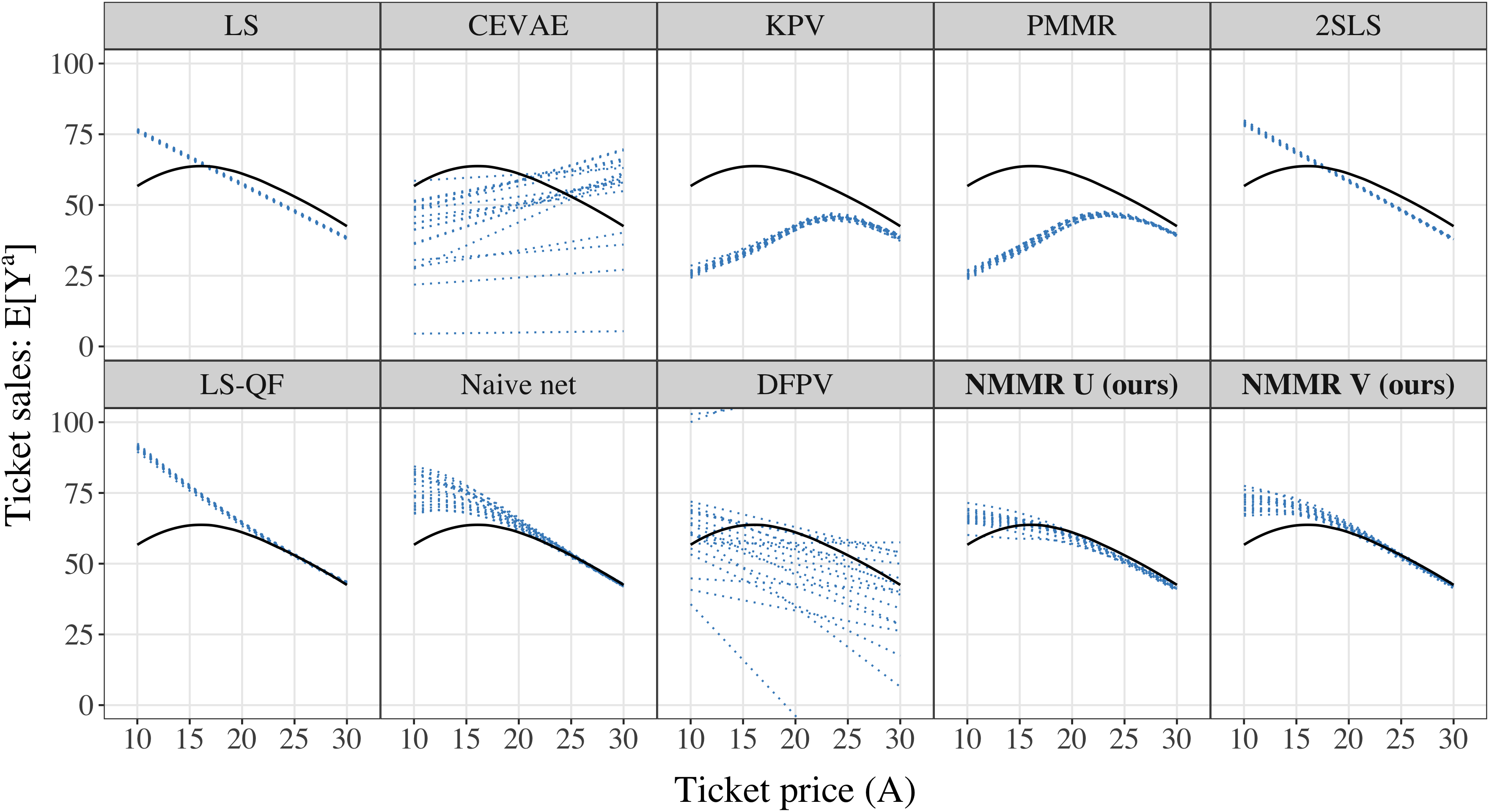}
         \caption{Demand experiment with 10,000 training data points with the true average potential outcome curves (black) and each method's predicted potential outcome curves (blue). Each method was replicated 20 times, generating one predicted curve per replicate. We observed some additional curvature to NMMR prediction curves}
         \label{fig:demand_predcurve10k}
\end{figure}

\begin{figure}[H]
     \centering
         \centering
         \includegraphics[width=\textwidth]{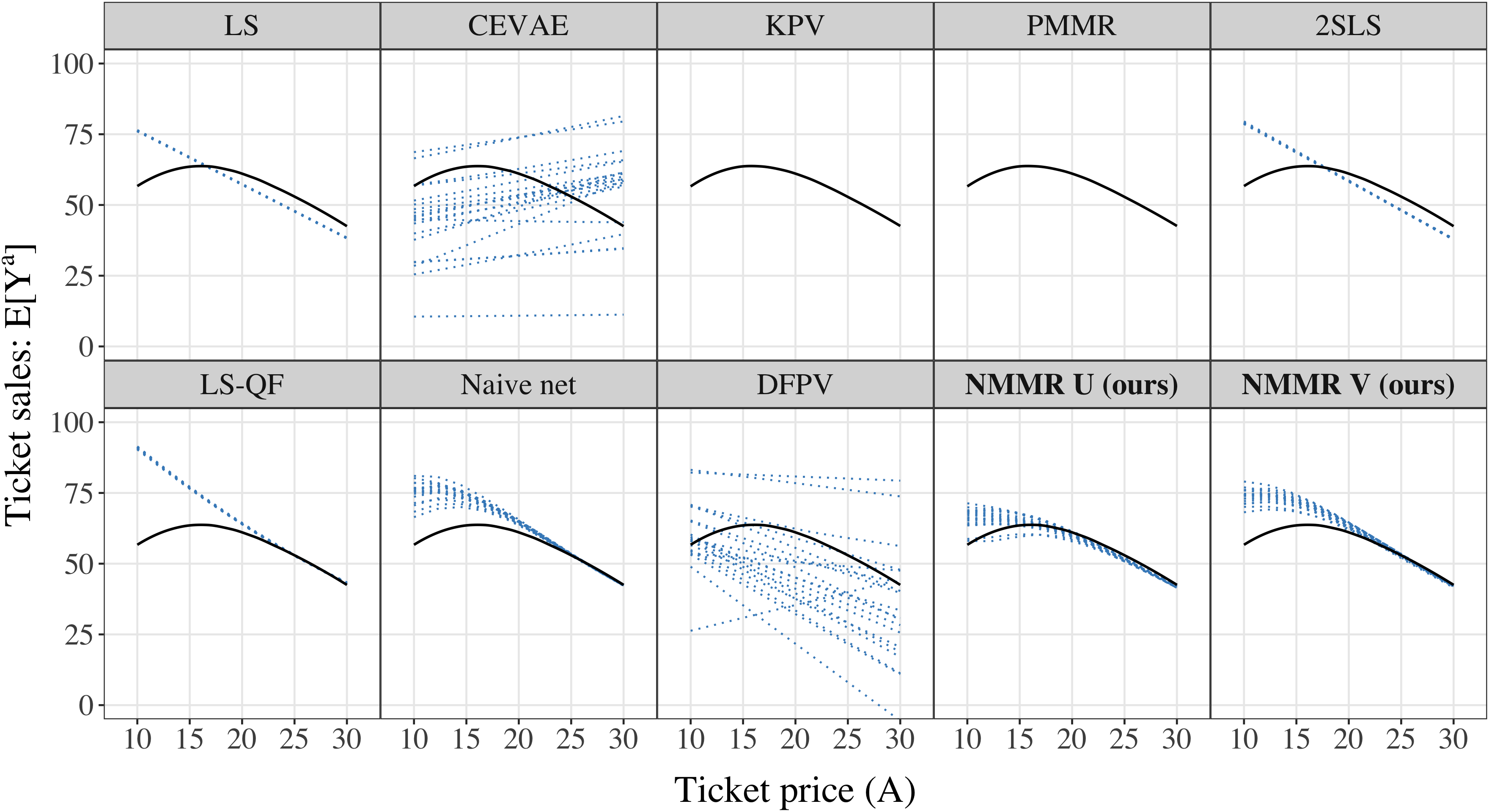}
         \caption{Demand experiment with 50,000 training data points with the true average potential outcome curves (black) and each method's predicted potential outcome curves (blue). Each method was replicated 20 times, generating one predicted curve per replicate. KPV and PMMR timed out due to computational requirements of their kernel methods.}
         \label{fig:demand_predcurve50k}
\end{figure}

\subsection{dSprite Data Generating Process}\label{section:dsprite_datagen}

The dSprite experiment has a unique data generating mechanism, given that the images $A$ and $W$ are queried from an existing dataset rather than generated on the fly. The dataset is indexed by parameters: shape, color, scale, rotation, posX, and posY. As mentioned in the paper, this experiment fixes shape = heart, color = white. Therefore, to simulate data from this system, we follow the steps:

\begin{compactenum}
    \item Simulate values for scale, rotation, posX, posY \textsuperscript{\textdagger}.
    \item Set $U = $ posY.
    \item Set $Z = $ (scale, rotation, posX).
    \item Set $A$ equal to the dSprite image with the corresponding (scale, rotation, posX and posY) as found in $Z$ and $U$, then add $\mathcal{N}(0, 0.1)$ noise to each pixel.
    \item Set $W$ equal to the dSprite image with (scale=0.8, rotation=0, posX=0.5) and posY from $U$, then add $\mathcal{N}(0, 0.1)$ noise to each pixel.
    \item Compute $Y = \frac{\frac{1}{10} ||vec(A)^TB||_2^2 - 5000}{1000} \times \frac{(31 \times \text{U} - 15.5)^2}{85.25} + \epsilon$, $\epsilon \sim \mathcal{N}(0, 0.5)$
\end{compactenum}

\textsuperscript{\textdagger} Let $\mathcal{DU}(a, b)$ denote a Discrete Uniform distribution from $a$ to $b$. Scale is a Discrete Uniform random variable taking values [0.5, 0.6, 0.7, 0.8, 0.9, 1.0] with equal probability. Rotation $\sim \mathcal{DU}(0, 2\pi)$ with 40 equally-spaced values. And both posX, posY $\sim \mathcal{DU}(0, 1)$ with 32 equally-spaced values. 

\subsection{dSprite causal DAG}

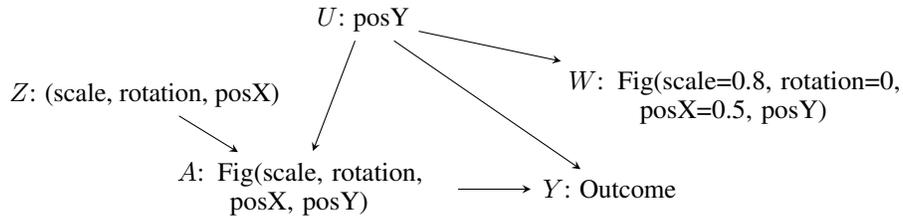
\begin{figure}[h]
\centering
\begin{tikzpicture}[> = stealth, shorten > = 1pt, auto, node distance = 2cm]
\tikzstyle{every state}=[
    draw = black,
    thick,
    fill = white,
    minimum size = 6mm
]
\node[text width=4cm,align=center] (A) {$A$: Fig(scale, rotation, \\posX, posY)};
\node (Z) [above left = 0.5cm and -2cm of A] {$Z$: (scale, rotation, posX)};
\node (U) [above right = 1.5cm and -2cm of A] {$U$: posY};
\node (Y) [right = 1cm of A] {$Y$: Outcome};
\node[text width=5cm,align=center] (W) [above right = 0.5cm and -2 cm of Y] {$W$: Fig(scale=0.8, rotation=0, \\posX=0.5, posY)};

\path[->] (A) edge node {} (Y);
\path[->] (U) edge node {} (A);
\path[->] (U) edge node {} (Y);
\path[->] (U) edge node {} (W);
\path[->] (Z) edge node {} (A);
\end{tikzpicture}
\caption{DAG for the dSprite experiment}
\label{fig:dsprite_dag}
\end{figure}

\subsection{dSprite Exemplar $A$ and $W$}

\begin{figure}[H]

    \begin{subfigure}{.45\textwidth}
    \centering
    \includegraphics[width=\textwidth]{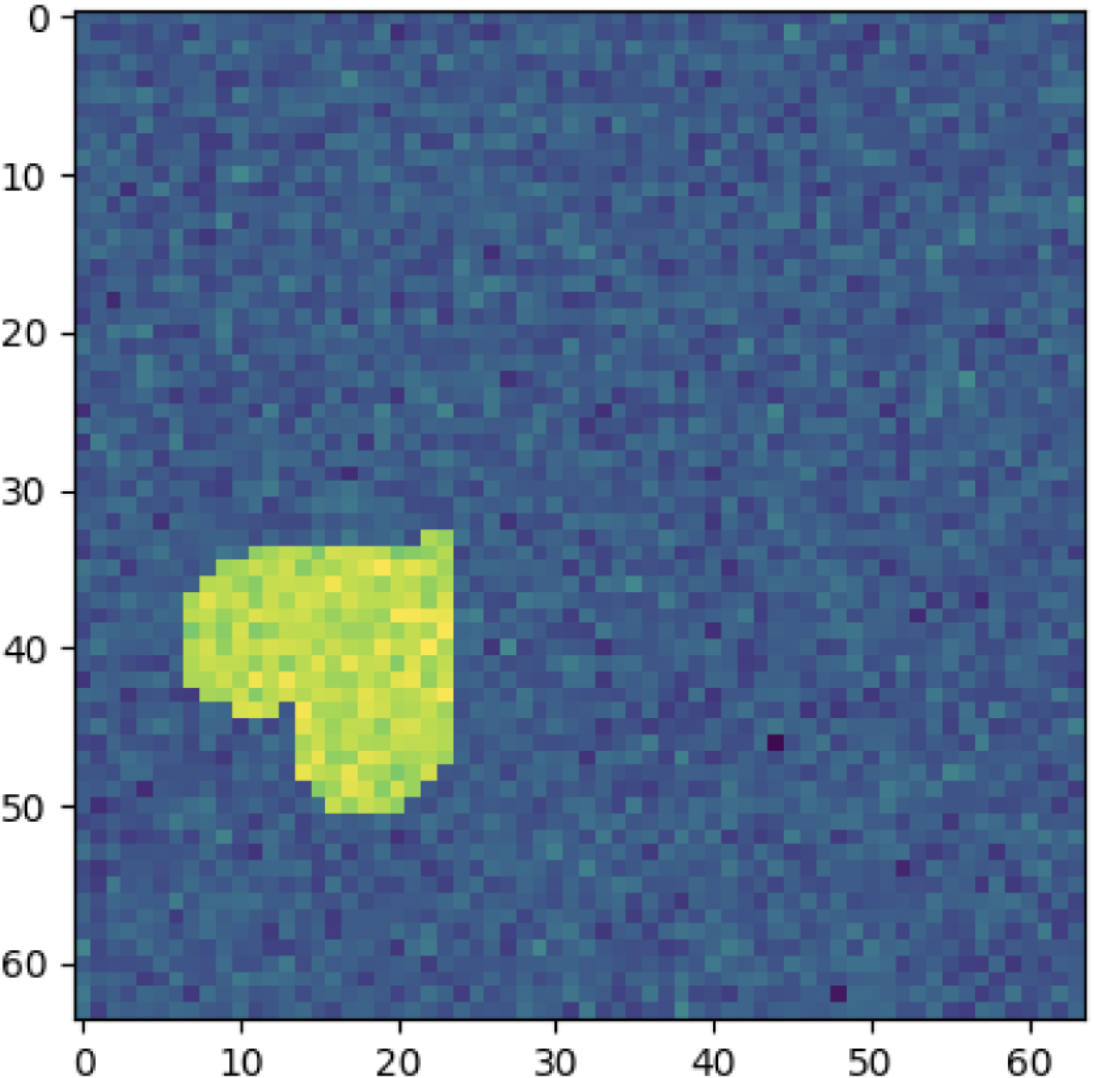}
    \caption{Example of $A$ in dSprite corresponding (scale, rotation, posX and posY) determined from $Z$ and $U$.}
    \end{subfigure}
    \hfill
    \begin{subfigure}{.45\textwidth}
    \centering
    \includegraphics[width=\textwidth]{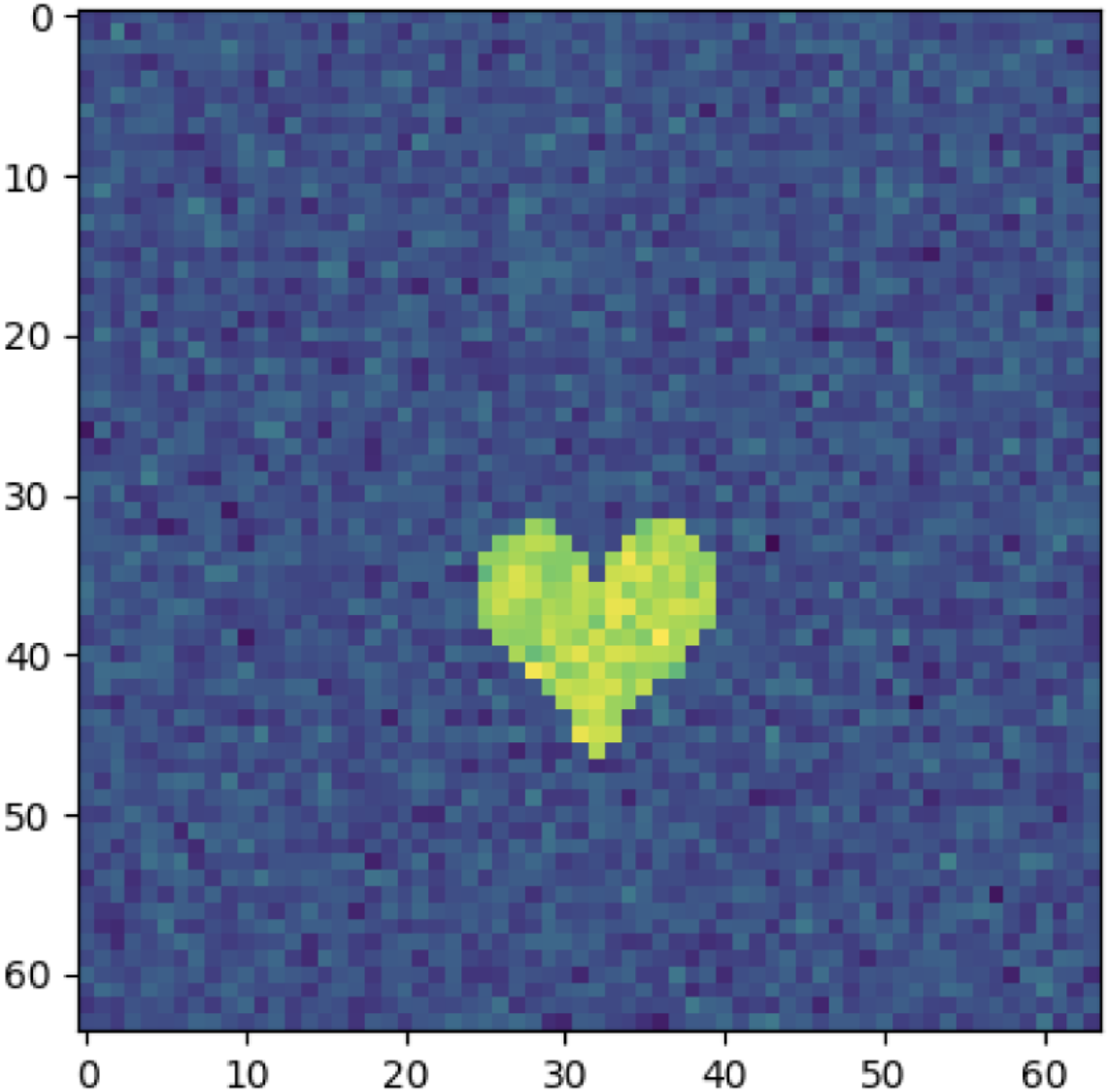}
    \caption{Example of $W$ in dSprite, which is always a centered, vertical heart with a fixed scale. The only thing that changes is posY, which is determined by $U$.}
    \end{subfigure}
    
    \caption{Examples of the image based treatment $A$ and outcome-inducing proxy $W$ in the dSprite experiment. Previous proximal inference methods did not take advantage of the inductive bias of image convolutions, which NMMR naturally incorporates into its neural network architecture for the dSprite benchmark.}
    \label{fig:dsprite_examples}
\end{figure}

\subsection{dSprite Test Set} \label{appendix:dsprite_test_set}

The dSprite test set consists of 588 images $A$ spanning the following grid of parameters:
\begin{compactitem}
    \item posX $\in [0, \frac{5}{31}, \frac{10}{31}, \frac{15}{31}, \frac{20}{31}, \frac{25}{31}, \frac{30}{31}]$
    \item posY $\in [0, \frac{5}{31}, \frac{10}{31}, \frac{15}{31}, \frac{20}{31}, \frac{25}{31}, \frac{30}{31}]$
    \item scale $\in [0.5, 0.8, 1.0]$
    \item rotation $\in [0, 0.5 \pi, \pi, 1.5\pi]$
\end{compactitem}

The labels for each test image $A$ are computed as $\mathbb{E}[Y^a] = \frac{\frac{1}{10}||vec(a)^T B||_2^2 - 5000}{1000}$.


\subsection{dSprite Boxplot Statistics}
Table \ref{tab:dsprite_table} contains the median and interquartile ranges (in parentheses) of the c-MSE values compiled in the boxplots shown in Figure \ref{fig:dsprite_boxplot}. NMMR demonstrated state of the art performance on the dSprite benchmark. We also extended the benchmark to include training set sizes of 7,500 data points, whereas \citet{Xu2021-io} originally included 1,000 and 5,000 data points. We observed that NMMR-V had strong performance across all dataset sizes and particularly excelled when training data increased. Most other methods remained relatively consistent as the amount of data increased.
\begin{table}[H]
\centering
\caption{dSprite Boxplot Median \& (IQR) values}
\resizebox{.7\textwidth}{!}{%
\begin{tabular}{@{}lccc@{}}
\toprule
              & \multicolumn{3}{c}{Training Set Size}                             \\ \cmidrule(l){2-4} 
Method        & 1,000               & 5,000                & 7,500               \\ \midrule
PMMR          & 17.7 (0.78)         & 17.74 (0.64)         & 16.2 (0.48)          \\
KPV           & 23.4 (1.37)         & 16.58 (0.93)         & 14.46 (1.01)         \\
Naive Net     & 32.25 (9.72)       & 34.24 (11.95)        & 34.76 (11.95)        \\
CEVAE         & 26.34 (0.82)        & 26.16 (1.51)         & 25.77 (1.45)         \\
DFPV          & 10.02 (2.95)        & 8.81 (2.04)          & 8.52 (1.06)          \\
\textbf{NMMR-U (ours)} & \textbf{4.72 (1.1)} & 7.1 (2.74)           & 7.52 (2.05)          \\
\textbf{NMMR-V (ours)} & 11.8 (1.88)         & \textbf{1.82 (0.67)} & \textbf{1.53 (0.68)} \\ \bottomrule
\end{tabular}%
}\label{tab:dsprite_table}
\end{table}

\subsection{dSprite DFPV vs. NMMR Evaluation}
In order to assess whether our improved performance on dSprite was due to the fact that NMMR leveraged convolutional neural networks while DFPV relied on multi-layer perceptrons with a spectral-norm regularization. We modified DFPV to include the same VGG-like heads mentioned in Appendix \ref{appendix:hp_model_arch}. We performed a grid search over the same learning rates and L2 penalties in Table \ref{tab:hp_grid}. Figure \ref{fig:dfpv_cnn} shows that DFPV with CNNs actually performed slightly worse than the original, published-version of DFPV. We report several different results for DFPV CNN since the results were so close after cross validation. Figure \ref{fig:dfpv_cnn} shows results for only the 1,000 data point evaluation -- DFPV had trouble scaling in practice as dataset size increased.

\begin{figure}[H]
    \centering
    \includegraphics[width=\textwidth]{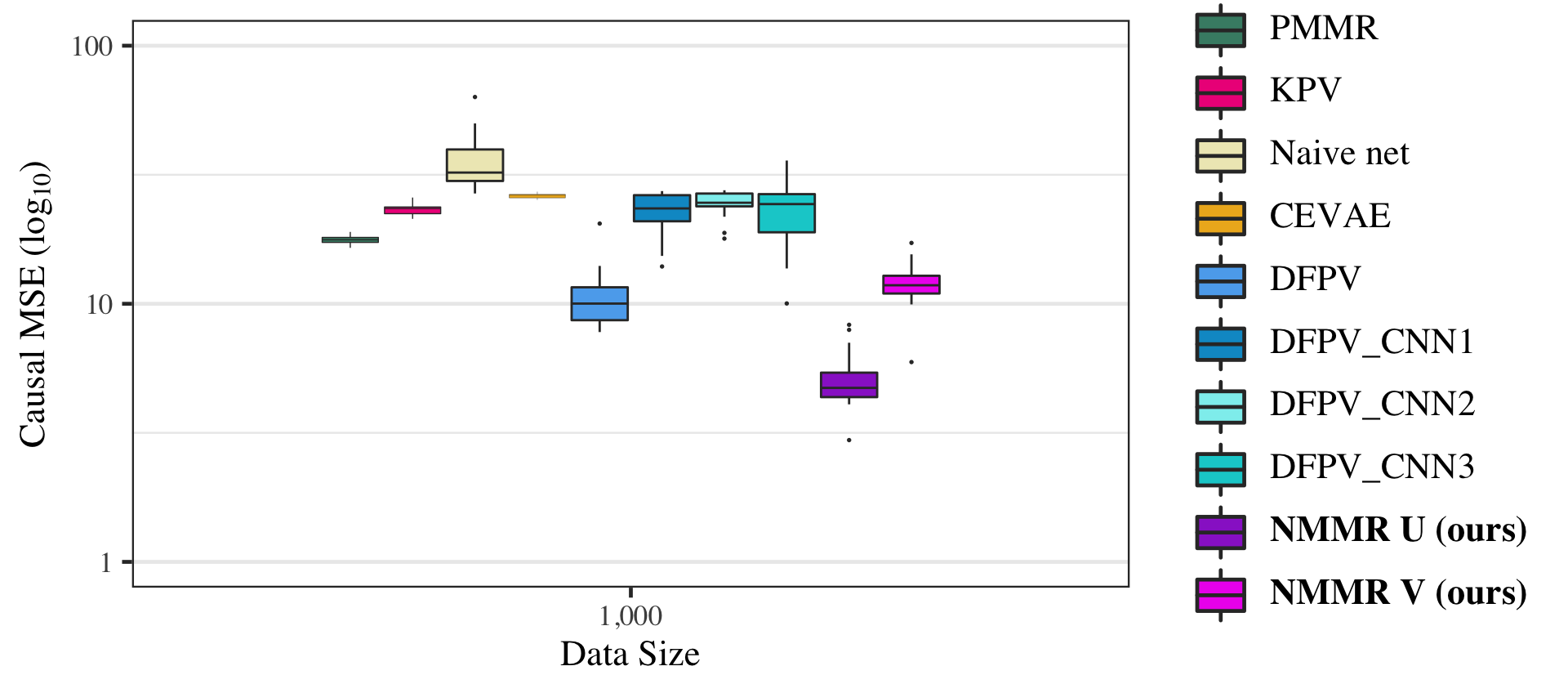}
    \caption{Performance of DFPV with CNNs compared to other evaluated methods on the dSprite benchmark. DFPV with CNNs performed worse compared to the published version of DFPV across a variety of hyperparameters}
    \label{fig:dfpv_cnn}
\end{figure}

\begin{table}[H]
\centering
\resizebox{.5\textwidth}{!}{%
\begin{tabular}{@{}lll@{}}
\toprule
Model     & Learning Rate & Weight Decay \\ \midrule
DFPV\_CNN1 & 3e-6          & 3e-6         \\
DFPV\_CNN2 & 3e-6          & 3e-7         \\
DFPV\_CNN3 & 3e-5          & 3e-7         \\ \bottomrule
\end{tabular}%
}
\caption{Hyperparameters for reported DFPV\_CNN models}
\label{tab:dfpv_cnn}
\end{table}

\section{Batched Loss Function} \label{sec:batched_loss}
When computing the unregluarized version of the loss function of NMMR:
$$\mathcal{L} = (Y-h(A,W,X))^TK(Y-h(A,W,X))$$
we either had to compute the kernel matrix $K$ for all points in the training set once, or dynamically calculate this matrix per batch. The latter approach would require many, many more calculations since we'd be repeating this process every batch and every epoch. 

Our solution relied on batching the V-statistic and U-statistic. Recall we can write the V-statistic as:
$$\hat{R}_V(h) = n^{-2} \sum_{ i, j = 1 }^n { \left( y_i - h_i \right) \left( y_j - h_j \right) k_{ij} }$$
We can vectorize this double sum as a series of vector and matrix multiplciations: 
\begin{align*}
    \begin{pmatrix}y_1-h(a, w_1, x_1),\dots,y_n-h(a, w_n, x_n)\end{pmatrix}\begin{pmatrix}k_{1,1} & \dots & k_{1,n} \\
\vdots & \ddots & \vdots \\ k_{n, 1} & \dots & k_{n,n} \end{pmatrix}\begin{pmatrix}y_1-h(a, w_1, x_1)\\\dots \\y_n-h(a, w_n, x_n)\end{pmatrix}
\end{align*}
and for the U-statistic, we can simply set the main diagonal of $K$ to be 0 to eliminate $i=j$ terms from this double sum. 

However, calculating $K$ for large datasets in a tensor-friendly manner resulted in enormous GPU allocation requests, on the order of 400GBs in the dSprite experiments. We implemented a batched version of the matrix multiplication above to circumvent this issue: 

\begin{lstlisting}[language=Python]
def NMMR_loss_batched(model_output, target, kernel_inputs, kernel_function, batch_size: int, loss_name: str):
    residual = target - model_output
    n = residual.shape[0]

    loss = 0
    for i in range(0, n, batch_size):
        # return the i-th to i+batch_size rows of K
        partial_kernel_matrix = calculate_kernel_matrix_batched(
            kernel_inputs, 
            (i, i+batch_size), 
            kernel_function)
        if loss_name == "V_statistic":
            factor = n ** 2
        if loss_name == "U_statistic":
            factor = n * (n-1)
            # zero out the main diagonal of the full matrix
            for row_idx in range(partial_kernel_matrix.shape[0]):
                partial_kernel_matrix[row_idx, row_idx+i] = 0
        # partial matrix multiplication
        temp_loss = residual[i:(i+batch_size)].T @  
            partial_kernel_matrix @ residual / factor
        loss += temp_loss[0, 0]
    return loss
\end{lstlisting}

\section{Noise Figures} \label{appendix:noise}

In the Demand noise experiment, we tested each method's ability to estimate $\mathbb{E}[Y^a]$ given varying levels of noise in the proxies $Z$ and $W$. Specifically, we varied the variance on the Gaussian noise terms $\epsilon_1$, $\epsilon_2$, and $\epsilon_3$ from the Demand structural equations described in Appendix \ref{appendix:demand_data_gen}. We will refer to these variances as $\sigma^2_{Z_1}$, $\sigma^2_{Z_2}$, and $\sigma^2_{W}$, respectively, and we set $\sigma^2_{Z_1} = \sigma^2_{Z_2}$ throughout. We refer to the pair $(\sigma^2_{Z_1}, \sigma^2_{Z_2})$ as "Z noise" and $\sigma^2_W$ as "W noise". In \citet{Xu2021-io}, these variances were all equal to 1. We evaluated each method\textdagger on 5000 samples from the Demand data generating process with the following $Z$ and $W$ noise levels:

$$\sigma^2_{Z_1}, \sigma^2_{Z_2} \in \{0, 0.01, 0.1, 0.5, 1, 2, 4, 8, 16\}$$
$$\sigma^2_W \in \{0, 0.01, 0.1, 0.5, 1, 16, 64, 150\}$$

In total there are 9 x 8 = 72 noise levels. From Appendix Figure \ref{fig:demand_eda}, Panels A and B, we can see that $Z_1$ and $Z_2$ lie approximately within the interval $[-4, 4]$, whereas $W$ lies approximately in the interval $[20, 45]$. Accordingly, we chose the maximum value of $Z$ and $W$ noise to be the square of half of the variable's range. So for $W$, half of this range is approximately 12.5 units, therefore the maximum value for $\sigma^2_W$ is $12.5^2 \approx 150$. Similarly, half of the range of $Z$ is 4, and so the maximum value of both $\sigma^2_{Z_1}$ and $\sigma^2_{Z_2}$ is $4^2 = 16$. This maximum level of noise is capable of completely removing any information on $U$ contained in $Z$ and $W$.

Intuitively, as the noise on $Z$ and $W$ is increased, they become less informative proxies for $U$. We would expect that greater noise levels will degrade each method's performance in terms of c-MSE. This experiment provides a way of evaluating how efficient each method is at recovering information about $U$, given increasingly corrupted proxies. It also provides some initial insights into the relative importance of each proxy, $Z$ and $W$.

Figure \ref{fig:demand_noise_boxplot} contains a 72-window grid plot with 1 window for each combination of $Z$ and $W$ noise and Figures \ref{fig:demand_noise_predcurve_cevae}-\ref{fig:demand_noise_predcurve_nmmr_v} show each method's individual potential outcome prediction curves at each of the 72 noise levels. We can see that NMMR-V is notably more robust to noise than NMMR-U, and also appears to be the most efficient method at higher noise levels. NMMR also consistently outperforms Naive Net, supporting the utility of the U- and V-statistic loss functions. However, we also note that kernel-based methods, such as KPV and PMMR, rank increasingly well with increased noise level, likely due to their lack of data adaptivity. We also observe that less data adaptive methods are less prone to large errors. Finally, we see a surprising trend that as the $Z$ noise is increased, several methods achieve lower c-MSE. We believe this stems from the fact that $W$ is a more informative proxy, so it is possible that noising $Z$ aids methods in relying more strongly on the better proxy for $U$. 

\textdagger We used the optimal hyperparameters for NMMR-U, NMMR-V and Naive Net found through tuning, as described in Appendix \ref{appendix:hp_model_arch}.

\begin{sidewaysfigure}[!htbp]
\centering
    \includegraphics[width=\textwidth]{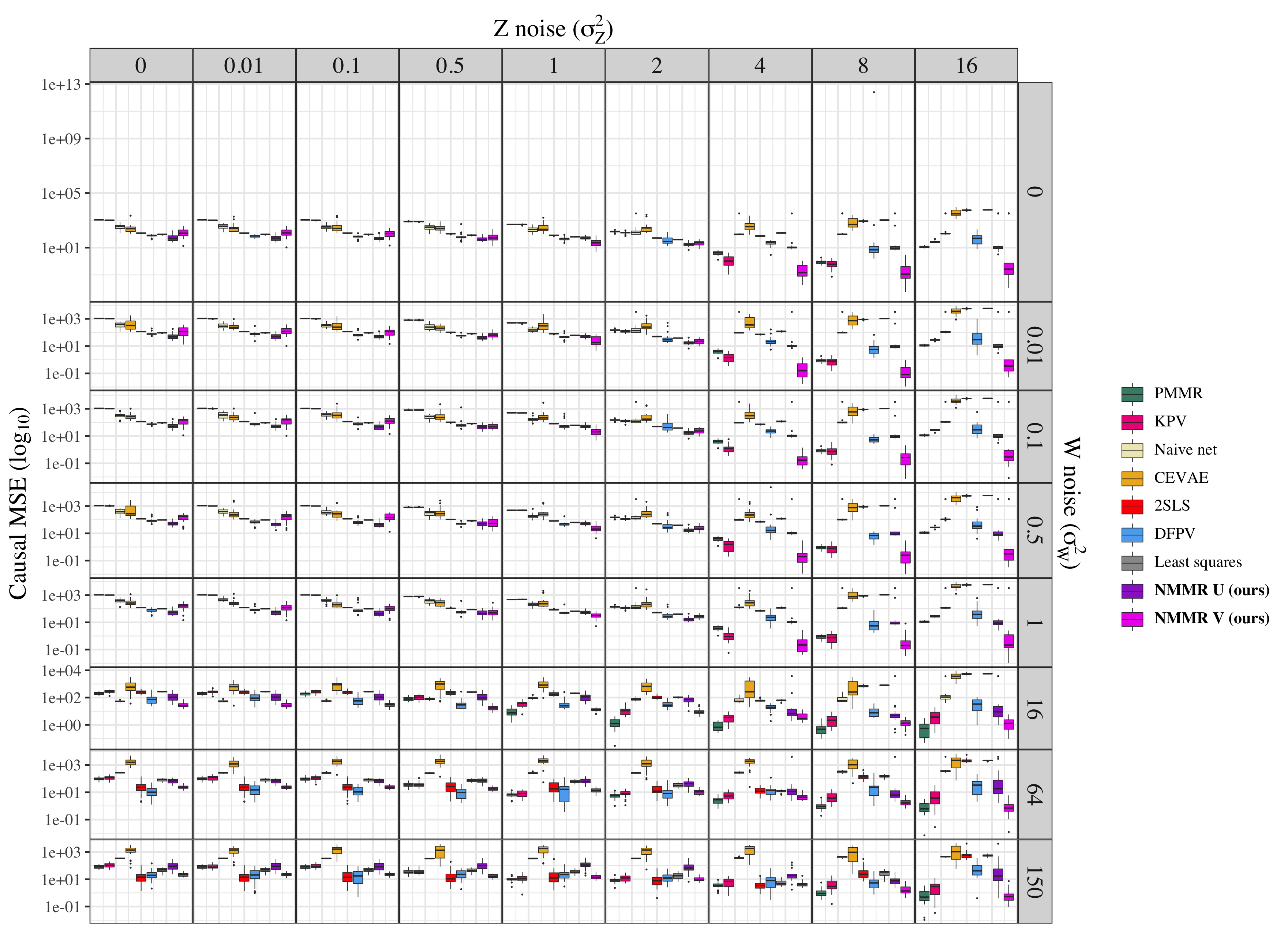}
    \caption{Causal MSE (c-MSE) of NMMR and baseline methods across 72 noise levels in the Demand experiment. \\ Each method was replicated 20 times and evaluated on the same 10 test values of $\mathbb{E}[Y^a]$ each replicate. \\ Each individual box plot represents 20 values of c-MSE.}
    \label{fig:demand_noise_boxplot}
\end{sidewaysfigure}

\begin{sidewaysfigure}[!htbp]
\centering
    \includegraphics[width=\textwidth]{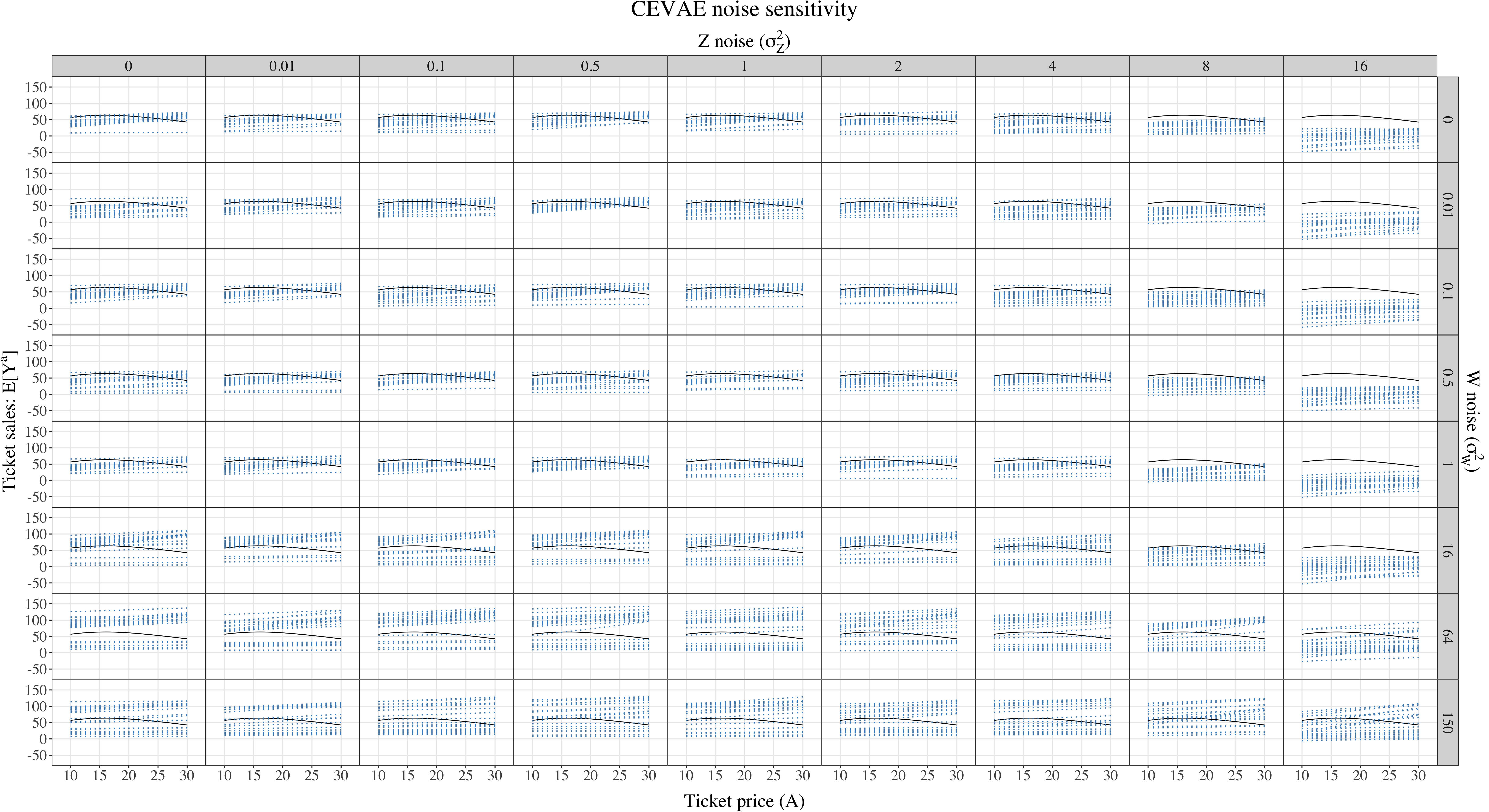}
    \caption{Predicted potential outcome curves for 20 replicates of CEVAE. Black curve is the ground truth $\mathbb{E}[Y^a]$.}
    \label{fig:demand_noise_predcurve_cevae}
\end{sidewaysfigure}

\begin{sidewaysfigure}[!htbp]
\centering
    \includegraphics[width=\textwidth]{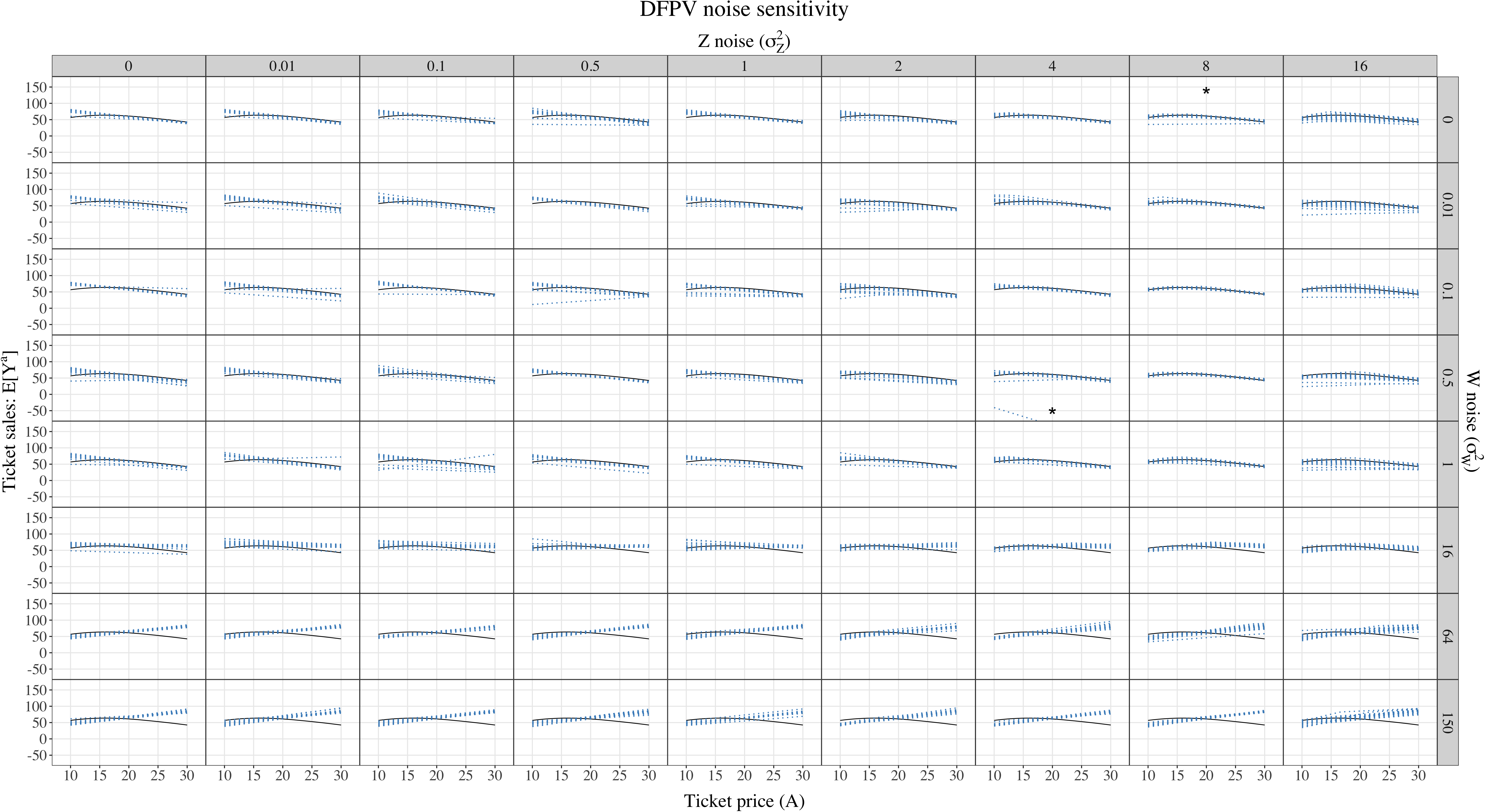}
    \caption{Predicted potential outcome curves for 20 replicates of DFPV. Black curve is the ground truth $\mathbb{E}[Y^a]$. Asterisks each indicate a replicate that lies outside of the plot's range. }
    \label{fig:demand_noise_predcurve_dfpv}
\end{sidewaysfigure}

\begin{sidewaysfigure}[!htbp]
\centering
    \includegraphics[width=\textwidth]{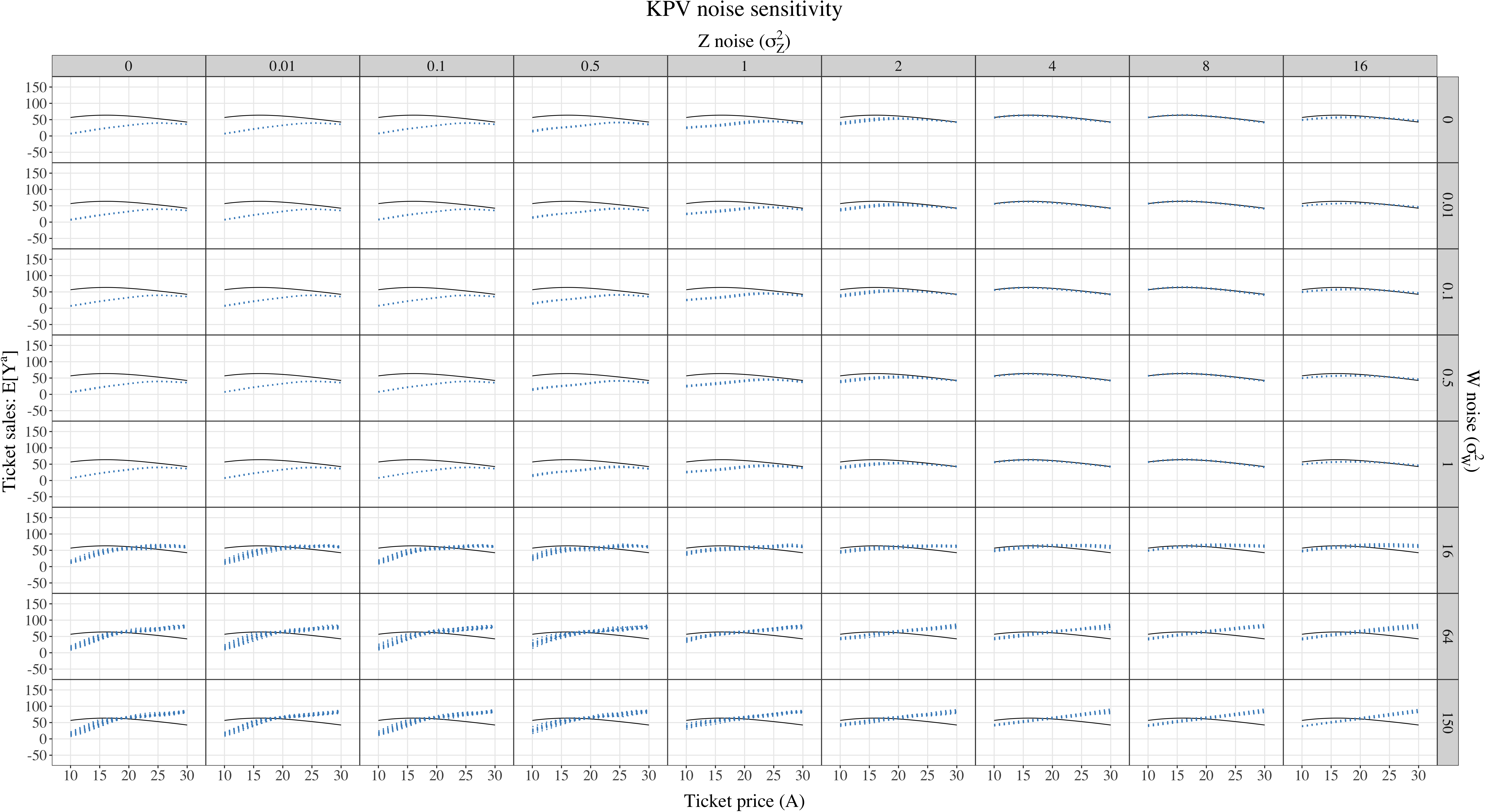}
    \caption{Predicted potential outcome curves for 20 replicates of KPV. Black curve is the ground truth $\mathbb{E}[Y^a]$.}
    \label{fig:demand_noise_predcurve_kpv}
\end{sidewaysfigure}

\begin{sidewaysfigure}[!htbp]
\centering
    \includegraphics[width=\textwidth]{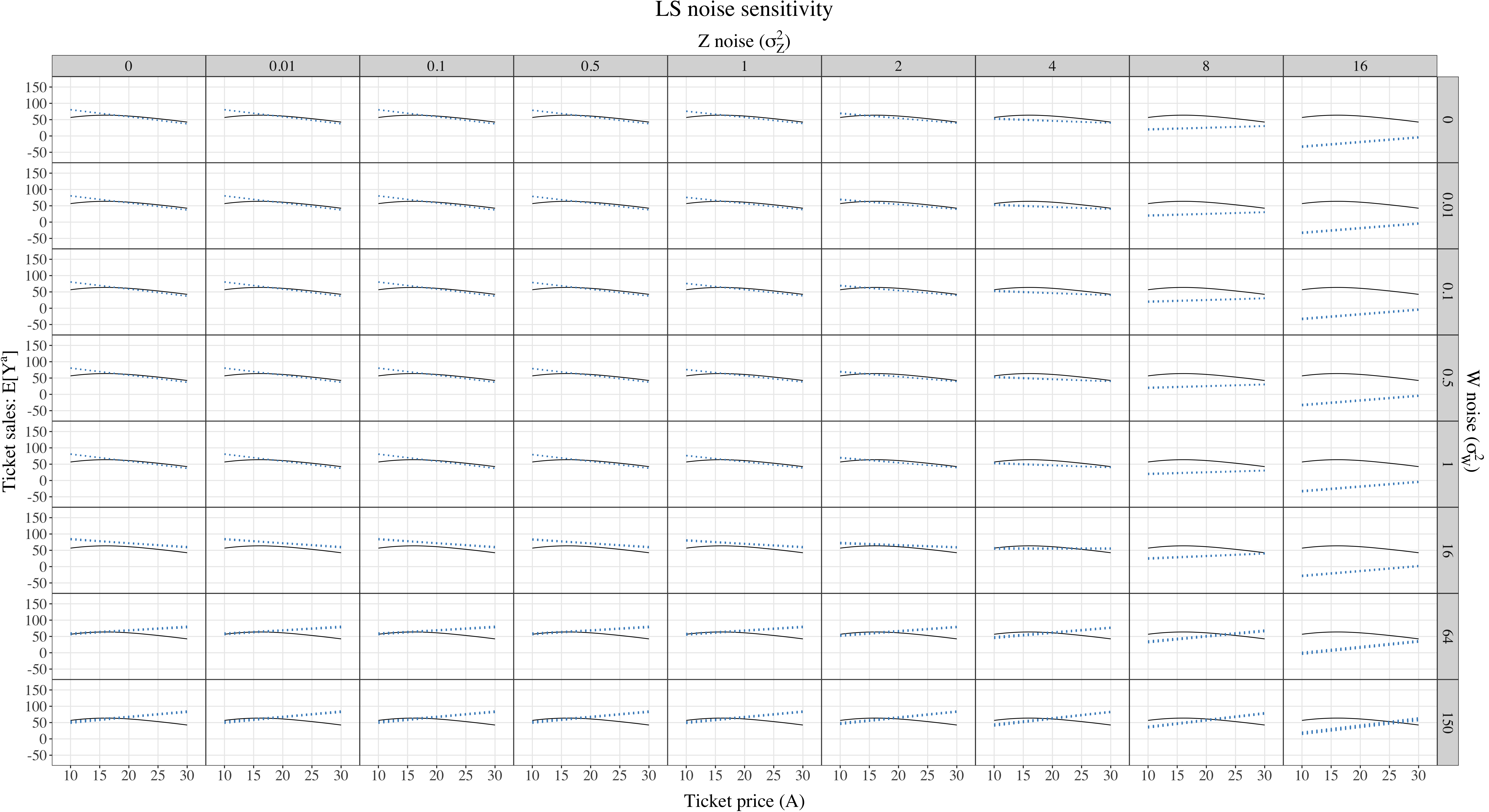}
    \caption{Predicted potential outcome curves for 20 replicates of Least Squares. Black curve is the ground truth $\mathbb{E}[Y^a]$.}
    \label{fig:demand_noise_predcurve_lr_awzy}
\end{sidewaysfigure}

\begin{sidewaysfigure}[!htbp]
\centering
    \includegraphics[width=\textwidth]{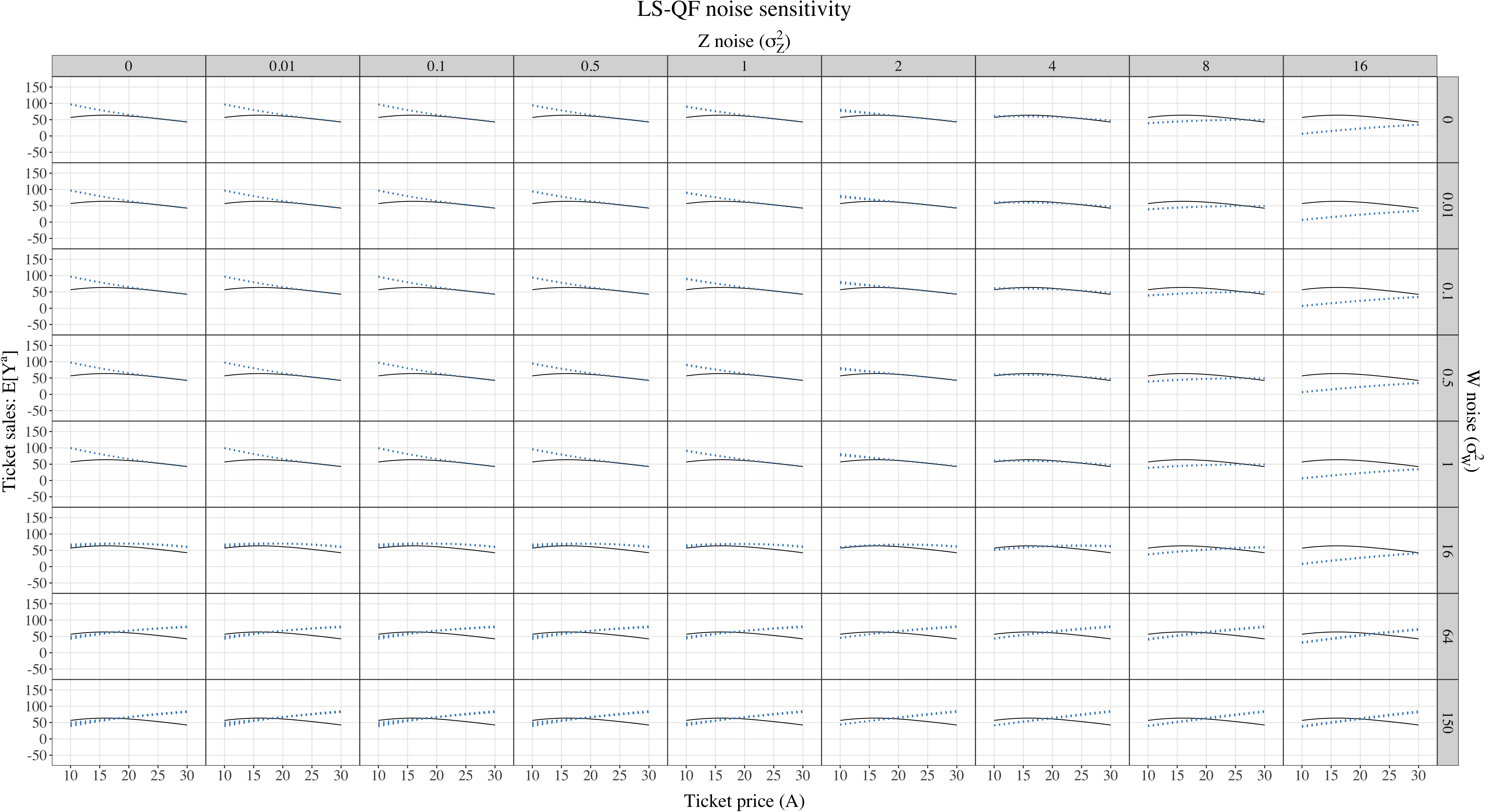}
    \caption{Predicted potential outcome curves for 20 replicates of LS-QF. Black curve is the ground truth $\mathbb{E}[Y^a]$.}
    \label{fig:demand_noise_predcurve_lr_awzy2}
\end{sidewaysfigure}

\begin{sidewaysfigure}[!htbp]
\centering
    \includegraphics[width=\textwidth]{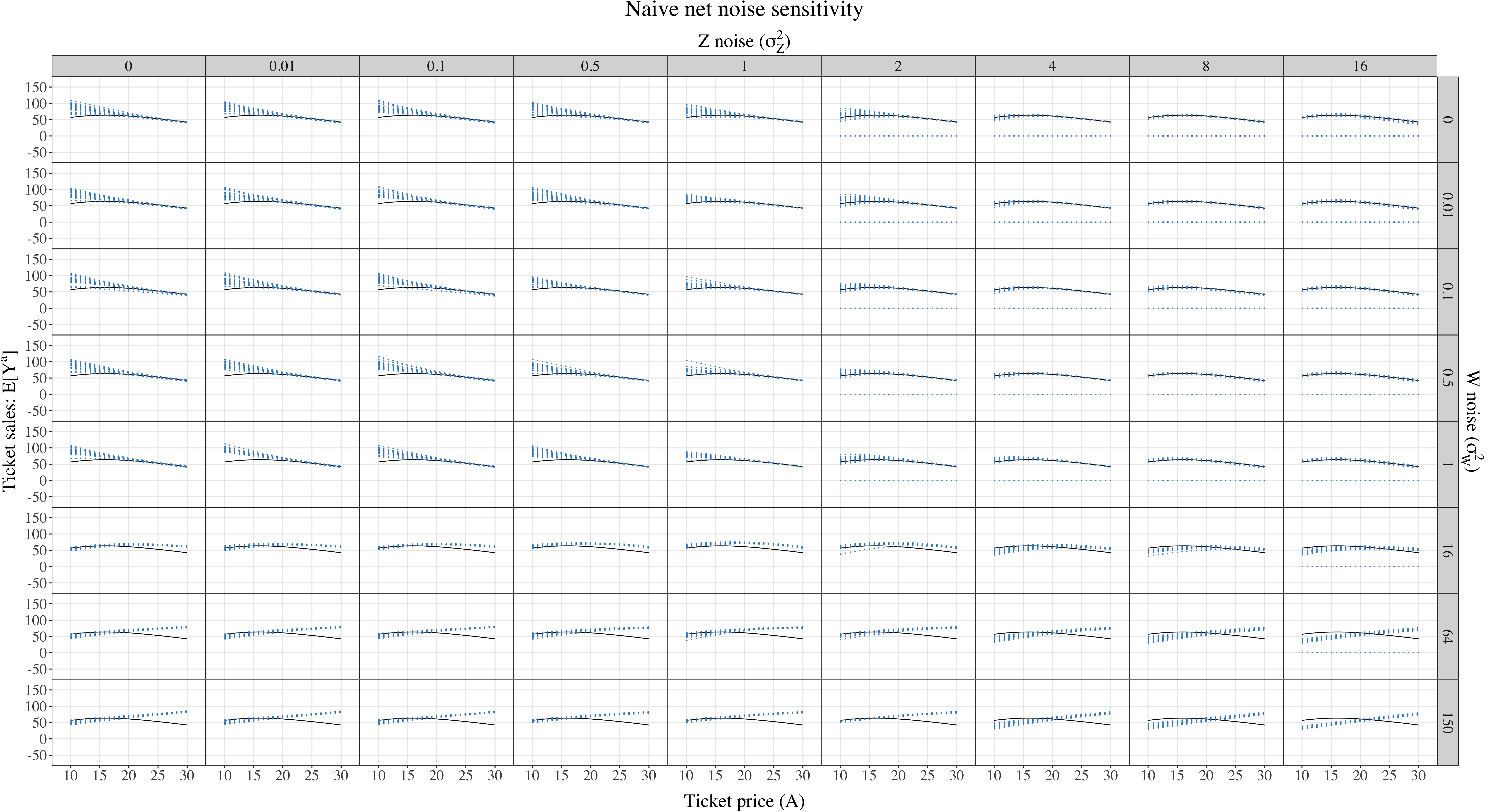}
    \caption{Predicted potential outcome curves for 20 replicates of Naive Net. Black curve is the ground truth $\mathbb{E}[Y^a]$.}
    \label{fig:demand_noise_predcurve_nn_awzy}
\end{sidewaysfigure}

\begin{sidewaysfigure}[!htbp]
\centering
    \includegraphics[width=\textwidth]{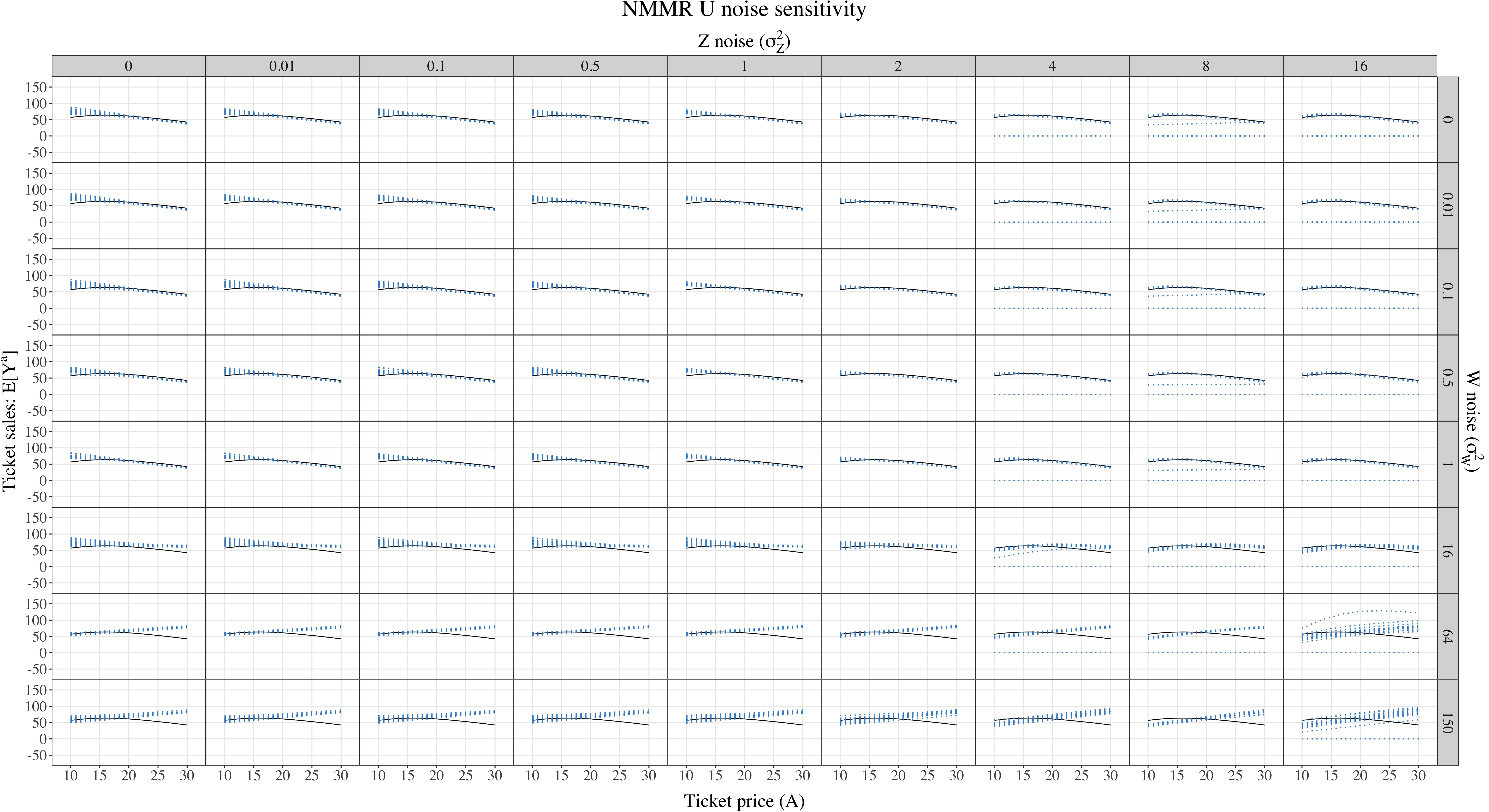}
    \caption{Predicted potential outcome curves for 20 replicates of NMMR-U. Black curve is the ground truth $\mathbb{E}[Y^a]$.}
    \label{fig:demand_noise_predcurve_nmmr_u}
\end{sidewaysfigure}

\begin{sidewaysfigure}[!htbp]
\centering
    \includegraphics[width=\textwidth]{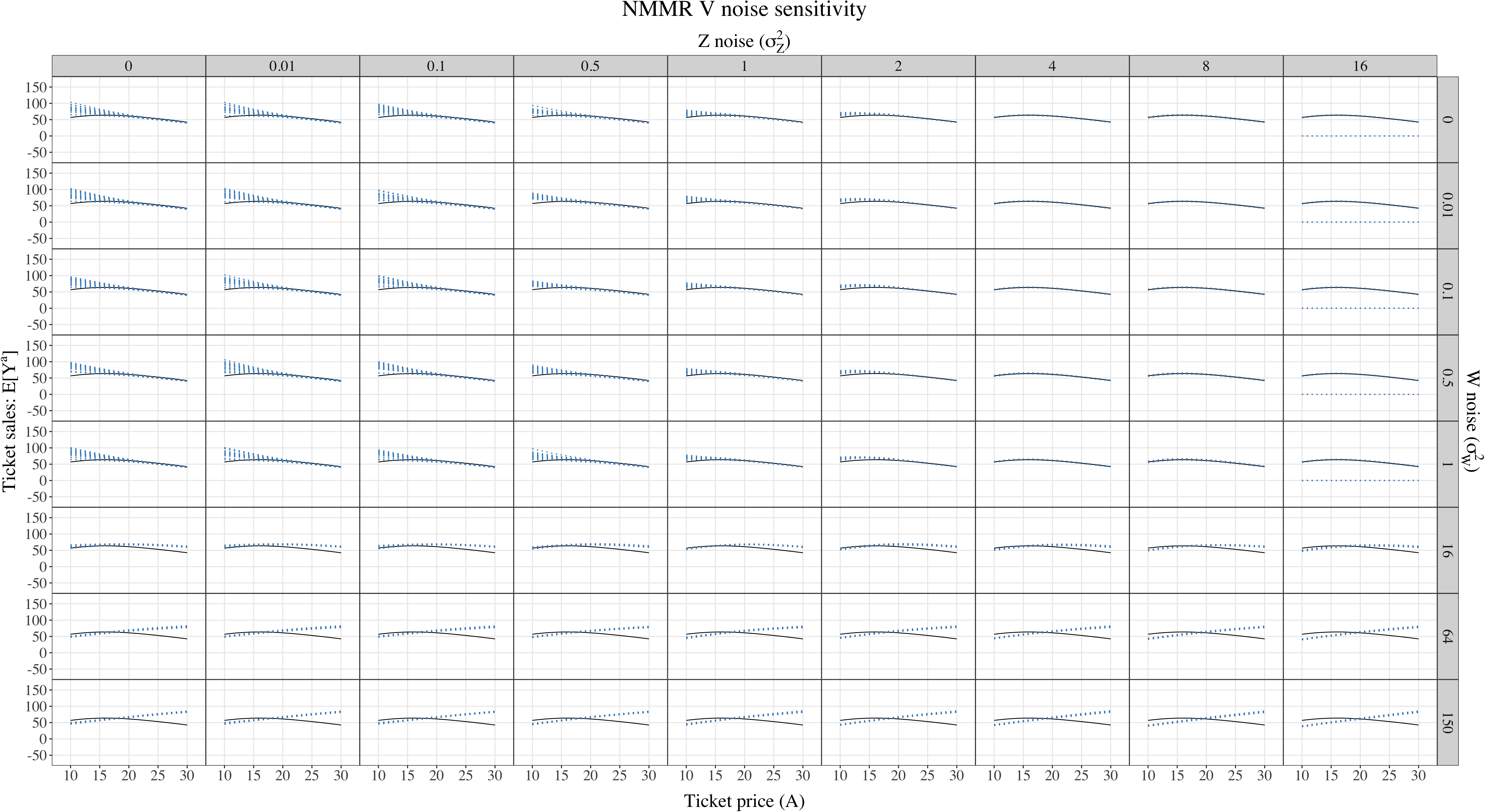}
    \caption{Predicted potential outcome curves for 20 replicates of NMMR-V. Black curve is the ground truth $\mathbb{E}[Y^a]$.}
    \label{fig:demand_noise_predcurve_nmmr_v}
\end{sidewaysfigure}

\end{document}